\documentclass[11pt]{article}
\usepackage[a4paper, margin=1in]{geometry} % Adjust margin size if needed
\usepackage{times} % Use Times font (or equivalent)
\usepackage{setspace}
\usepackage{microtype}
\usepackage{graphicx}
\usepackage{subfigure}
\usepackage{caption}
\usepackage{subcaption}
\usepackage{booktabs} % for professional tables
\usepackage{authblk} % For code highlighting
\usepackage{algorithm}
\usepackage{algorithmic}
\usepackage{hyperref}
\usepackage{tcolorbox} % Required for creating colored boxes

\usepackage{lmodern}  % Improved version of Computer Modern
\usepackage[T1]{fontenc}  % Improves special character support

\usepackage{amsmath}
\usepackage{amssymb}
\usepackage{mathtools}
\usepackage{amsthm}

\usepackage[capitalize,noabbrev]{cleveref}

\theoremstyle{plain}
\newtheorem{theorem}{Theorem}[section]
\newtheorem{proposition}[theorem]{Proposition}
\newtheorem{lemma}[theorem]{Lemma}

\theoremstyle{definition}

\theoremstyle{remark}

\newtheorem{property}{Property}

\usepackage[textsize=tiny]{todonotes}

\newcommand{\ie}{\text{i.e.}, }

\DeclarePairedDelimiter\norm{\lVert}{\rVert}% || ||
\DeclarePairedDelimiter\innorm{\langle}{\rangle}% < >

\DeclareMathOperator{\R}{\mathbb{R}} % real numbers
\DeclareMathOperator*{\argmax}{arg\,max}
\DeclareMathOperator*{\argmin}{arg\,min}
\DeclareMathOperator*{\arginf}{arg\,inf}

\DeclareMathOperator{\St}{\mathcal{S}}
\DeclareMathOperator{\A}{\mathcal{A}}

\DeclareMathOperator{\Z}{\mathcal{Z}}

\DeclareMathOperator{\B}{\mathcal{B}}
\DeclareMathOperator{\K}{\mathcal{K}}
\DeclareMathOperator{\D}{\mathcal{D}}
\newcommand{\Uc}{\mathcal{U}}

\newcommand{\rob}{_{\mathcal{\Uc}}}

\title{ Dual Formulation for Non-Rectangular $L_p$ Robust Markov Decision Processes}

\author[1]{Navdeep Kumar}
\author[4]{Adarsh Gupta}
\author[2]{Maxence Mohamed Elfatihi}
\author[3]{Giorgia Ramponi}
\author[1]{Kfir Yehuda Levy}
\author[1,5]{Shie Mannor}

\affil[1]{Technion}
\affil[2]{École Polytechnique}
\affil[3]{University of Zurich}
\affil[4]{Finsyth AI}
\affil[5]{NVIDIA Research}

\date{}

\begin{document}

\maketitle





\vskip 0.3in



 % otherwise use the standard text.

\begin{abstract}
We study robust Markov decision processes (RMDPs) with non-rectangular uncertainty sets, which capture interdependencies across states unlike traditional rectangular models. While non-rectangular robust policy evaluation is generally NP-hard, even in approximation, we identify a powerful class of $L_p$-bounded uncertainty sets that avoid these complexity barriers due to their structural simplicity. We further show that this class can be decomposed into infinitely many \texttt{sa}-rectangular $L_p$-bounded sets and leverage its structural properties to derive a novel dual formulation for $L_p$ RMDPs. This formulation provides key insights into the adversary’s strategy and enables the development of the first robust policy evaluation algorithms for non-rectangular RMDPs. Empirical results demonstrate that our approach significantly outperforms brute-force methods, establishing a promising foundation for future investigation into non-rectangular robust MDPs.
\end{abstract}

\section{Introduction}
Robust Markov Decision Processes (MDPs) provide a framework for developing solutions that are more resilient to uncertain environmental parameters compared to standard MDPs \cite{BiasVarianceShie, HKuhn2013, tamar14, Nilim2005RobustCO, Iyenger2005}. This approach is particularly critical in high-stakes domains, such as robotics, finance, healthcare, and autonomous driving, where catastrophic failures can have severe consequences. The study of robust MDPs is further motivated by their potential to offer superior generalization capabilities over non-robust methods \cite{robustnessAndGeneralization, genralization1, generalization2}.

Robust solutions are highly desirable, but obtaining them is a challenging task. In particular, robust policy evaluation has been shown to be strongly NP-hard \cite{RVI} for general convex uncertainty sets. Consequently, much of the existing work makes rectangularity assumptions, with the most common being \(s\)-rectangular uncertainty sets and its special case sa-rectangular uncertainty sets \cite{k-rectangularRMDP, r-rectRMDP, RVI, Kaufman2013RobustMP, Bagnell01solvinguncertain, Nilim2005RobustCO, Iyenger2005, ppi, Rcontamination, PG_RContamination, RPG_conv, derman2021twice, LpRMDP, LpPgRMDP, Ewok_dileep, Ewok}.

The \texttt{s}-rectangularity assumption simplifies the modeling of uncertainty by treating it as independent across states \cite{RVI}. This assumption is analytically appealing due to the existence of contractive robust Bellman operators, which facilitate computational tractability \cite{RVI}. However, real world uncertainties are often coupled across the states, and modelling it with \texttt{s}-rectangular uncertainty sets, can overly introduce conservatism in solutions (as illustrated in Figure 1 of \cite{LpRewardRobust}). In other words, the ratio between the volume of a real world coupled uncertainty set and the smallest \texttt{s}-rectangular uncertainty containing it, can be exponential in the state space (as the ratio of volumes of a n-dimension sphere and a n-dimension cube containing it is $O(2^n)$ \cite{n-sphere}). In contrast, non-rectangular RMDPs capture much better these interdependencies but lack the existence of any contractive robust Bellman operators, which makes the problem unwieldy \cite{LpRewardRobust}.

 % The s-rectangularity assumption is a mathematical construct where  uncertainty can be factored across the states \cite{RVI}. This framework fails to model the problem where the uncertainty in states are interdependent, which is more practical occurring. However, this additional coupling between states, comes with great challenge as  contractive robust Bellman operator that plays central role in rectangular robust MDPs,  cease to exist for non-rectangular robust MDPs \cite{LpRewardRobust}. 
To the best of our knowledge, research on non-rectangular robust MDPs remains limited \cite{LpRewardRobust, RPG_conv}. While \cite{LpRewardRobust} explored robust policy evaluation for non-rectangular uncertainty sets, their work was confined to reward uncertainties. In contrast, this paper addresses kernel uncertainty, which presents a significantly greater level of complexity compared to reward uncertainty. Additionally, \cite{RPG_conv} demonstrated the convergence of robust policy gradient methods with an iteration complexity of \( O(\epsilon^{-4}) \) for all types of uncertainty sets, including non-rectangular ones. However, their method depends on oracle access to the robust gradient (worst-case kernel). As a result, robust policy evaluation under non-rectangular kernel uncertainties remains an unresolved challenge.

Further, dual formulation for non-robust MDPs \cite{Puterman1994MarkovDP} has played a significant role in advancing the field. Unfortunately, no such formulation exists for the robust MDPs.

The key insight of this work is decomposing minimization over a nonrectangular $L_p$ norm-bounded uncertainty sets into minimization over a union of \texttt{sa}-rectangular $L_p$-norm bounded uncertainty sets. For each minimization over  rectangular uncertainty set, we have the robust return in close form. Now, we minimize this expression over the set of possible all \texttt{sa}-rectangular $L_p$-norm bounded uncertainty sets that makes up the original nonrectangular set. Using the fact that worst kernel is rank-one perturbation of the nominal kernel in \texttt{sa}-rectangular robust MDPs, we derive dual formulation for the $L_p$ robust MDPs. This reveals a very interesting insishts over how the adversary chooses the worst kernel. Further, this dual formulation inspires a method for evaluation of robust policy. 

\textbf{Contributions.}
\begin{itemize}
    \item We show that the general NP-hardness result for policy evaluation in non-rectangular RMDPs does not apply to \(L_p\)-bounded robust MDPs.
    \item We derive a novel dual formulation for \(L_p\)-RMDPs, providing key insights into the adversary's strategy and enabling the development of first robust policy evaluation algorithms.
    \item We experimentally validate our proposed algorithms, demonstrating significant improvements over existing brute-force methods.
\end{itemize}
This work opens up the avenue for the further investigation of non-rectangular RMDPs otherwise believed too hard. 

\subsection{Related Work}
To the best of our knowledge, there exists no work on non-rectangular robust MDPs with kernel uncertainty. This work is the first to propose an efficient method for robust policy evaluation for a very useful class of uncertainty sets, otherwise thought to be NP-Hard \cite{RVI}.

\textbf{Rectangular Robust MDPs.} In literature, \texttt{sa}-rectangular uncertainty is a very old assumption \cite{Iyenger2005,Nilim2005RobustCO}. \cite{RVI} introduced \texttt{s}-rectangular uncertainty sets and proved its tractability, in addition to the intractability of the general non-rectangular uncertainty sets. 

The most advantageous aspect of the \texttt{s}-rectangularity, is the existence of contractive robust Bellman operators. This gave rise to many robust value based methods \cite{ppi,RPG_conv}. Further, for many specific uncertainty sets, robust Bellman operators are equivalent to regularized non-robust operators, making the robust value iteration as efficient as non-robust MDPs \cite{derman2021twice, Rcontamination, LpRMDP}.  

There exists many policy gradient based methods for robust MDPs, relying upon contractive robust Bellman operators for the robust policy evaluation \cite{PG_RContamination, LpPgRMDP}.

Further, \cite{Ewok_dileep,Ewok} trie to tweak the process, and directly get samples from the adversarial model via pessimistic sampling.

There exist other notions of rectangularity such as k-rectangularity \cite{k-rectangularRMDP} and r-rectangularity \cite{r-rectRMDP} which are sparsely studied. However, \cite{TractablerRMDP} shows, the theses uncertainty sets are either equivalent to s-rectangularity or non-tractable.

\textbf{Non-Rectangular Reward Robust MDPs. } 
Policy evaluation for robust MDPs with non-rectangular uncertainty set is proven to be a Strongly-NP-Hard problem \cite{RVI}, in general. For a very specific case, where uncertainty is limited only to reward uncertainty bounded with $L_p$ norm, \cite{LpRewardRobust} proposed robust policy evaluation via frequency (occupation measure) regularization, and derived the policy gradient for policy improvement. 

\textbf{Convergence Rate of Robust Policy Gradient .} The robust policy gradient method has been shown to converge with iteration complexity of \(O(\epsilon^{-4})\) for general robust MDPs \cite{RPG_conv}. However, it requires oracle access to robust policy evaluation (i.e., the computation of the worst kernel), which can be computationally expensive \cite{RPG_conv}.

\section{Preliminary}
% \textbf{Notations.} For a set $\mathcal{S}$, $\lvert\mathcal{S}\rvert$ denotes its cardinality. $\langle u, v\rangle := \sum_{s\in\mathcal{S}}u(s)v(s)$ denotes the dot product between functions $u,v:\mathcal{S}\to\mathbb{R}$. $\lVert v\rVert_p^q :=(\sum_{s}\lvert v(s)\rvert^p)^{\frac{q}{p}}$ denotes the $q$-th power of $L_p$ norm of function $v$, and we use  $\lVert v\rVert_p := \lVert v\rVert^1_p$ and $\lVert v\rVert := \lVert v\rVert_2$ as shorthand. For a set $\mathcal{C}$, $\Delta_{\mathcal{C}}:=\{a:\mathcal{C} \to \mathbb{R}|a(c)\geq 0, \forall c, \sum_{c\in\mathcal{C}}a_c=1\}$ is the probability simplex over $\mathcal{C}$.
% $\mathbf{0} and\mathbf{1} $ denotes all-zeros vector and all-ones vector/function respectively of appropriate dimension/domain. $\mathbf{1}(a=b):=1$ if $a=b$, 0 otherwise, is the indicator function. For vectors $u,v$, $\mathbf{1}(u\geq v)$ is component wise indicator vector, i.e. $\mathbf{1}(u\geq v)(x) = \mathbf{1}(u(x)\geq v(x))$. $A\times B =\{(a,b)\mid a\in A, b\in B\}$ is Cartesian product between set $A$ and $B$.

A Markov Decision Process (MDP) can be described as a tuple $(\mathcal{S},\mathcal{A},P, R,\gamma,\mu)$, where $\mathcal{S}$ is the state space, $\mathcal{A}$ is the action space, $P$ is a transition kernel mapping $\mathcal{S}\times\mathcal{A}$ to $\Delta_{\mathcal{S}}$, $R$ is a reward function mapping $\mathcal{S}\times\mathcal{A}$ to $\mathbb{R}$, $\mu$ is an initial distribution over states in $\mathcal{S}$, and $\gamma$ is a discount factor in $[0,1)$ \cite{Puterman1994MarkovDP,Sutton1998}. A policy $\pi:\mathcal{S}\to\Delta_\mathcal{A}$ is a decision rule that maps state space to a probability distribution over action space. Let $\Pi = (\Delta_{\A})^{\St}$ denote set  of all possible policies. Further, $\pi(a|s), P(s'|s,a)$ denotes the probability of taking action $a$ in state $s$ by policy $\pi$, and the probability of transition to state $s'$ from state $s$  under action $a$ respectively. In addition, we denote $P^\pi(s'|s) = \sum_{a}\pi(a|s)P(s'|s,a)$ 
and $
R^\pi(s) = \sum_{a}\pi(a|s)R(s,a)$ as  short-hands. 

The return of a policy $\pi$, is defined as 
$ J^\pi_{P} = \langle\mu,v^\pi_{P}\rangle = \langle R^\pi,d^\pi_{P}\rangle $ where $v^{\pi}_{P} :=D^\pi R^\pi$ is value function, $d^\pi_P = \mu^\top D^\pi$ is occupation measure  and $D^\pi = 
(I-\gamma P^\pi)^{-1}$ is occupancy matrix  \cite{Puterman1994MarkovDP}. As a shorthand, we denote $d^\pi_P(s,a) = d^\pi_P(s)\pi(a|s)$ and the usage shall be clear from the context. 

 % MDPs admits a dual form, that is, primal formulation $\max_{\pi\in\Pi}\innorm{\mu,v^{\pi}}$ is equivalent to its dual formulation of $\max_{d^\in \mathcal{K}}\innorm{R^\pi,d}$ where $\mathcal{K} = \{d^\pi|\pi\in\Pi\}$ is the set occupancy measure \cite{Puterman1994MarkovDP}. This primal-dual relationship has tremendous value in non-robust MDPs literature. 

A robust Markov Decision Process (MDP) is a tuple $(\mathcal{S},\mathcal{A},R,\Uc,\gamma,\mu)$ which generalizes the standard MDP,  by containing a set of environments $\mathcal{\Uc}$ \cite{Iyenger2005,Nilim2005RobustCO}. The reward robust MDPs is well-studied in the precious work of rectangular 
\cite{derman2021twice,LpRMDP} and non-rectangular \cite{LpRewardRobust} uncertainty sets. Hence, in this work , we consider only uncertainty in the kernel which is much more challenging. 

For  an uncertainty set \(\mathcal{U}\), the robust return \( J^\pi_\Uc \) for a policy \(\pi\), and the optimal robust return \( J^*_\Uc \), are defined as: \[J^\pi_{\mathcal{U}} = \min_{P \in \mathcal{U}} J^\pi_{P}, \quad \text{and} \quad J^*_{\mathcal{U}} = \max_{\pi} J^\pi_{\mathcal{U}}, \]
respectively. The objective is to determine an optimal robust policy \(\pi^*_\mathcal{U}\) that achieves the optimal robust performance \(J^*_{\mathcal{U}}\). Unfortunately, even robust policy evaluation (i.e., finding the worst-case transition kernel \(P^\pi_\Uc \in \argmin_{P \in \Uc} J^\pi_P\)) is strongly NP-hard for general (non-rectangular) convex uncertainty sets \cite{RVI}. This makes solving non-rectangular robust MDPs a highly challenging problem.

To make the problem tractable, a common approach is to use \(\texttt{s}\)-rectangular uncertainty sets, \(\mathcal{U}^{\texttt{s}}  = \times_{s \in \mathcal{S}} \mathcal{P}_s\), where the uncertainty is modeled independently across states \cite{RVI}. These sets decompose state-wise, capturing correlated uncertainties within each state while ignoring inter-dependencies across states. This allows the robust value function to be defined as (vector minimum) $v^\pi_{\Uc^{\texttt{s}}} = \min_{P\in\Uc^{\texttt{s}}}v^\pi_P $\cite{RVI}.

A further simplification is the \(\texttt{sa}\)-rectangular uncertainty set, \(\mathcal{U}^{\texttt{sa}}\), where uncertainties are assumed to be independent across both states and actions. Formally,  
$\mathcal{U}^{\texttt{sa}} = \times_{(s,a) \in \mathcal{S} \times \mathcal{A}} \mathcal{P}_{s,a},  $
where \(\mathcal{P}_{s,a}\) are independent component sets for each state-action pair \cite{Iyenger2005,Nilim2005RobustCO,Rcontamination,PG_RContamination}. Notably, \(\texttt{sa}\)-rectangular sets are a special case of \(\texttt{s}\)-rectangular sets.

Various types of rectangular uncertainty sets have been explored in the literature \cite{Rcontamination,WR2L, Phi_RMDP}. In this work, we focus specifically on \(L_p\)-bounded uncertainty sets 
$\mathcal{U}^{\texttt{sa}}_p/ \mathcal{U}^{\texttt{s}}_p$ , which are centered around a nominal transition kernel \(\hat{P}\) \cite{ppi, derman2021twice, LpRMDP, LpPgRMDP}, defined as 
\begin{align*}
 \mathcal{U}^{\texttt{sa}}_p  &= \{ P\Bigm| \sum_{s'}P_{sa}(s')=1, \lVert P_{sa}-\hat{P}_{sa}\rVert_p\leq \beta_{sa}\},\\
    \mathcal{U}^{\texttt{s}}_p &= \{P \Bigm| \sum_{s'}P_{sa}(s')=1,\lVert P_s-\hat{P}_{s}\rVert_p\leq \beta_{s}
     \},
\end{align*}
with small enough radius vector $\beta$ is small enough, to ensure all the kernels in the uncertainty sets are valid. Further, symbol \(q\) is the Hölder conjugate of $p$ and $\sigma_p$ is the generalized standard deviation (GSTD)\cite{LpRMDP} defined as:
\[\frac{1}{p} + \frac{1}{q} = 1, \quad \text{and}\quad \sigma_p(v):=\min_{\omega\in\R}\norm{v-\omega\mathbf{1}}_p. \]
One of the most surprising, and useful facts about $L_p$ bounded uncertainty sets, is that the adversarial kernel is a rank one perturbation of the nominal kernel, as stated below. 
\begin{proposition}\label{main:bg:worstPm} (\textbf{Nature of the Adversary},\cite{LpPgRMDP}) For uncertainty set $\Uc =\Uc^{sa}_p/\Uc^s_p$, the worst kernel is given as
\begin{align*}
     P^\pi_{\Uc} =  \hat{P} -bk^\top,
    \end{align*}
where $k$ depends on the robust value function $v^\pi_\Uc$ and 
$b$ is (policy weighted of $\Uc^{\texttt{s}}_p$) radius vector.
\end{proposition}
This result is insightful however it doesn't characterize the direction of perturbation $k$ in nominal terms.
% The direction \(k\) in the result above is positively correlated with the robust value function \(v^\pi_\Uc\), effectively discouraging transitions to high-value states in favor of low-value ones. In the \(\texttt{sa}\)-rectangular case, the adversary independently applies the maximum perturbation strength to each state-action pair, i.e., \(b = \beta\). In contrast, for the \(\texttt{s}\)-rectangular case, the perturbation strength is distributed as \(b_{sa} = \beta_s \frac{\pi(a|s)^{q-1}}{\lVert \pi_s \rVert_{q}^{q-1}}\), since the worst-case values cannot be independently selected for each action due to state-level correlations.

% While both \texttt{sa}-rectangularity and \texttt{s} rectangularity assumptions, improve tractability, they fail to capture interdependencies across states which are crucial in many practical applications.
\paragraph{Robust Policy Gradient Methods.}  
The absence of contractive robust Bellman operators renders the development of value-based methods for robust MDPs particularly challenging. Consequently, policy gradient methods naturally emerge as a viable alternative. The update rule is given by:  
\begin{align}\label{def:RPG:updateRule}
    \pi_{k+1} = \text{Proj}_{\pi \in \Pi} \Big[\pi_k - \eta_k \nabla_{\pi} J^{\pi_k}_{P_k}\Big],
\end{align} 
where \( J^{\pi_k}_{P_k} - J^{\pi_k}_\Uc \leq \epsilon \gamma^k \) and learning rate $\eta_k = O(\frac{1}{\sqrt{k}})$. This approach guarantees convergence to a global solution within \( O(\epsilon^{-4}) \) iterations \cite{RPG_conv}.  

However, this update rule depends on oracle access to the robust gradient, which is highly challenging to obtain because robust policy evaluation is an NP-hard problem. Moreover, no prior work has addressed robust gradient evaluation in the context of non-rectangular robust MDPs. This work constitutes the first attempt to compute the robust gradient for such MDPs by leveraging the dual structure of robust MDPs, paving the way for practical robust policy gradient methods.

\paragraph{Dual Formulation of MDPs.}  
The primal formulation of an MDP is defined as:  
\[
\max_{v \in \mathcal{V}} \langle \mu, v \rangle, \quad \text{with its dual:} \quad \max_{d \in \mathcal{D}} \langle d, R \rangle,  
\]
where \( \mathcal{V} = \{v \mid v = R^\pi + \gamma P^\pi v, \pi \in \Pi\} \) represents the set of value functions. The dual formulation relies on the state-action occupancy measure \(d\), where \(d \in \mathcal{D} \subset \mathbb{R}^{|S| \times |A|}\) satisfies the non-negativity constraint (\(d \succeq 0\)) and the flow conservation constraint:  
$\sum_{a} d(s, a) - \gamma \sum_{s', a'} d(s', a') P(s \mid s', a') = \mu(s), \quad \forall s \in \mathcal{S}.$
The feasible set \(\mathcal{D}\) forms a convex polytope \cite{altman-constrainedMDP}, whereas the set of value functions, \(\mathcal{V}\), is a polytope that is generally non-convex \cite{valuePolytope}. This dual formulation offers several advantages, including efficient handling of constraints and the ability to solve the problem using linear programming techniques.

For robust MDPs, the geometry of robust value functions is significantly more intricate compared to standard MDPs \cite{wang}. While the dual formulation for standard MDPs is well-established, this work is the first to derive a dual formulation for robust MDPs. This novel formulation provides critical insights and lays the foundation for the development of robust policy evaluation methods.

\section{Method}
In this section, we derive the dual formulation for non-rectangular robust Markov decision processes (RMDPs) with uncertainty sets bounded by \(L_p\) balls. This dual perspective not only introduces several new research questions but also provides critical insights into the underlying problem. Further, we develop a method to compute the worst kernel, thereby enabling robust policy evaluation.

 We begin with defining the non-rectangular \(L_p\)-constrained uncertainty set around the nominal kernel $\hat{P}$ as:  
\[
\mathcal{U}_p = \Big\{P \;\Big|\; \lVert P-\hat{P}  \rVert_p \leq \beta, \; \sum_{s'} P(s' \mid s, a) = 1  \Big\}.
\]  
Throughout the paper, we use \(d^\pi, v^\pi, J^\pi, D^\pi\) as shorthand for \(d^\pi_{\hat{P}}, v^\pi_{\hat{P}}, J^\pi_{\hat{P}}\), and \(D^\pi_{\hat{P}}\), respectively w.r.t. nominal kernel $\hat{P}$.  The simplex constraint ensures that the transition kernel \(P\) satisfies the unity-sum-rows property, as discussed in \cite{LpRMDP}. The kernel radius \(\beta\) is assumed to be small enough to guarantee that all kernels within \(\mathcal{U}_p\) are well-defined, consistent with assumptions made in prior works \cite{derman2021twice, LpPgRMDP, LpRMDP}.

Note that this setting allows noise in one state to be coupled with noise in other states. Before delving into solving it, we first discuss why its  important? Why are uncertainty sets modeled with non-rectangular sets \(\Uc_p\) (e.g., \(L_2\)-balls) better than rectangular ones?

In Figure \ref{fig:lpBall}, we illustrate this by capturing the uncertainty set using non-rectangular \(\Uc_2\) (circle/sphere) balls and rectangular (square/cube) balls. The blue dots represent possible environments, with the origin being the nominal environment. Points farther away from the origin indicate larger perturbations. Specifically, points near the corners of the square/cube represent environments with large perturbations in all dimensions or coordinates simultaneously. The likelihood of such simultaneous perturbations is very low, and this issue becomes even more pronounced in higher dimensions. This phenomenon is well discussed in the paper \textit{Lightning Doesn't Strike Twice: Coupled RMDPs}\cite{Lightning_Deos_Not_Strike_Twice_RMDP}.

\begin{figure}[ht]
    \centering
    \includegraphics[width=\linewidth]{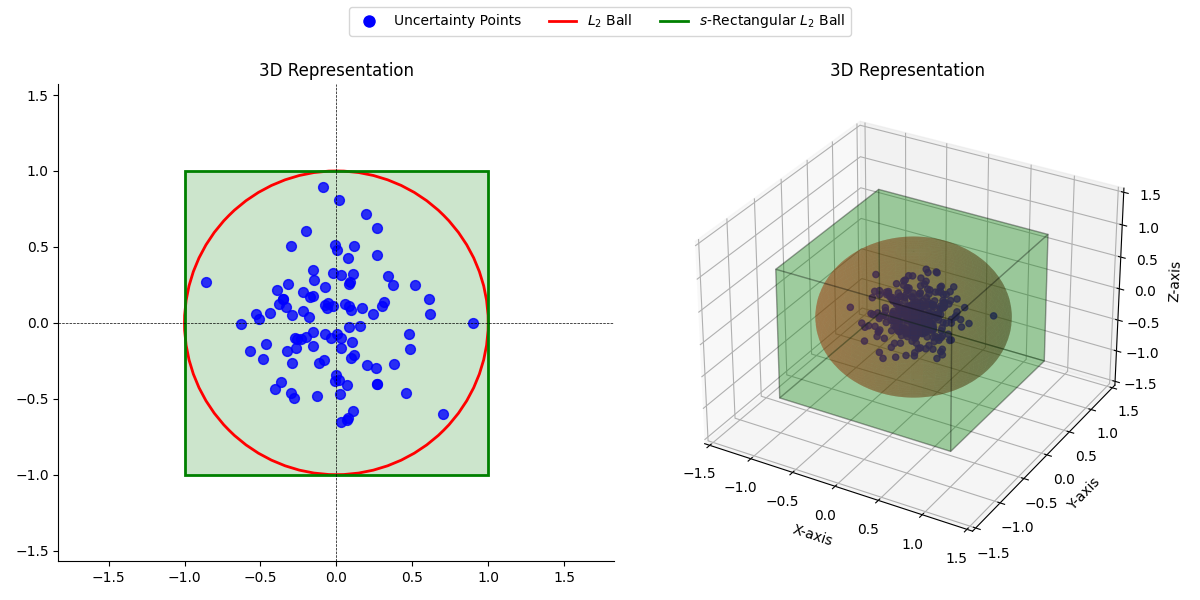}
    \caption{Modeling Uncertainty with Non-Rectangular and Rectangular \(L_2\)-Balls.}
    \label{fig:lpBall}
\end{figure}

To make matters worse, as shown in the result below, most of the volume of a high-dimensional cube lies near its corners outside the embedded sphere. This implies that rectangular robust MDPs are overly conservative, as their uncertainty sets focus on environments near the corners—corresponding to highly unlikely extreme perturbations.

\begin{proposition}  
Let \(\Uc^{\texttt{sa}}_2\) and \(\Uc^{\texttt{s}}_2\) denote the smallest \texttt{sa}-rectangular and \texttt{s}-rectangular sets, respectively, that contain \(\Uc_2\). Then:
\[
\frac{\textbf{vol}(\Uc_2)}{\textbf{vol}(\Uc^{\texttt{sa}}_2)} = O(c_{sa}^{-SA}),\quad \text{and}\quad \frac{\textbf{vol}(\Uc_2)}{\textbf{vol}(\Uc^{\texttt{s}}_2)} = O(c_s^{-S}),
\]
where \(\textbf{vol}(X)\) denotes the volume of the set \(X\), and \(c_s, c_{sa} > 1\) are constants.
\end{proposition}

The result follows from the \(n\)-dimensional sphere's volume \(c_n r^n\) (\(c_n \to 0\)) \cite{n-sphere}, compared to the enclosing cube's volume \(2^n r^n\) (side \(2r\)), resulting in a ratio of \(O(2^n)\).

From the above discussion, we conclude that non-rectangular robust MDPs are less conservative. However, robust policy evaluation (even approximation) has been proven NP-hard  for general uncertainty sets defined as intersections of finite hyperplanes \cite{RVI}. Specifically, \cite{RVI} reduces an Integer Program (IP) with $m$ constraints to robust MDPs where the uncertainty set consists of intersections of $m$  half-spaces (\( m \)-linear constraints). This polyhedral structure is fundamental to the hardness proof, hence, it does not extend to our uncertainty sets \( \mathcal{U}_p \) for \( p > 1 \).  For the case of \( \mathcal{U}_1 \), the IP reduction does apply, but since \( \mathcal{U}_1 \) is defined by a single global constraint (\(\|P - \hat{P}_1\|_1 \leq \beta\)), it forces the IP to have only one simple constraint which is efficiently solvable. A detailed discussion can be found in Appendix \ref{app:sec:IP2RMDP}.

% From the above discussion, we conclude that non-rectangular robust MDPs are less conservative. However, robust policy evaluation is proven to be NP-hard for general uncertainty sets defined as intersections of finite hyperplanes \cite{RVI}. Specifically, the paper reduces Integer Program (IP) to robust MDPs whose uncertainty sets are intersection of $m>1$ half-spaces ($m$-linear constraints). This polyhedral structural uncertainty is crucial for the proof, hence the proof doesn't apply to our uncertainty  sets $\Uc_p$ for $p>1$. Further, for uncertainty set $\Uc_1$, the IP reduction does apply, but since $\Uc_1$ has only one global constraint ($\norm{P-\hat{P}_1\leq \beta}$), it forces IP to have only constraint (specifically $F = \mathbf{1}^\top$). And this class of IPs are easily solvable, with more details in Appendix \ref{app:sec:IP2RMDP}.  
We conclude that \(L_p\)-robust MDPs are potentially tractable. A key insight is that a non-rectangular uncertainty set \( \mathcal{U}_p \) can be expressed as a union of \( \texttt{sa} \)-rectangular sets \( \mathcal{U}^{\texttt{sa}}_p(b) \) with varying radius vectors \( b \), each of which can be solved more easily on its own. 
\begin{figure}
    \centering
\includegraphics[width=0.5\linewidth]{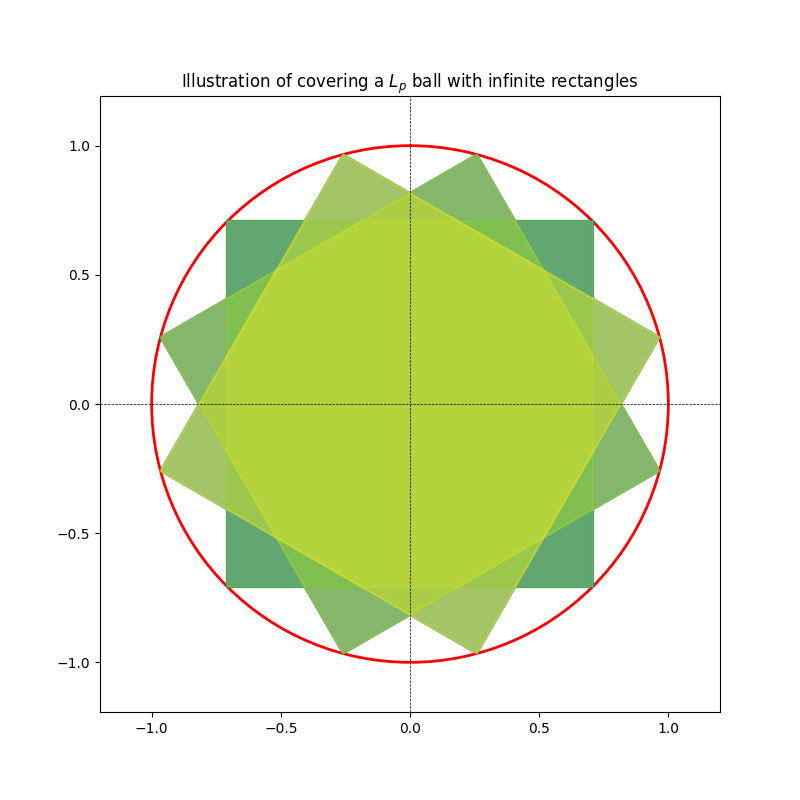}
    \caption{Illustration of Proposition \ref{main:rs:Set:sa2nr}: N-dimensional sphere can be written as infinite union of n-dimenssional inscribing cubes.}
    \label{fig:main:rs:Set:sa2nr}
\end{figure}

\begin{proposition}\label{main:rs:Set:sa2nr} Non-rectangular uncertainty $\Uc_p$ can be written as infinite union of \texttt{sa}-rectangular sets $\Uc^{sa}_p$, as
\[ \Uc_{p} = \bigcup_{b\in\mathcal{B}} \Uc^{\texttt{sa}}_p(b),\]
where $\mathcal{B} =\{b\in\R_+^{\St \times \A}\mid  \norm{b}_p\leq \beta\} $. Note that  all of them share the same nominal kernel $\hat{P}$.
\end{proposition}
The proof of the above result intuitively generalizes the idea that a circle (or \(n\)-dimensional sphere) can be covered by an inscribed square (or \(n\)-dimensional rectangles) touching its boundaries and a continuum of its rotated versions, as shown in Figure \ref{fig:main:rs:Set:sa2nr}. This offers a significant simplification to the problem at hand, as it implies that non-rectangular policy evaluation (difficult) can be decomposed into \(\texttt{sa}\)-rectangular uncertainty sets (easier) as:
\begin{align}\label{main:eq:rr:BU}
J^\pi_{\Uc_p} = \min_{b \in \mathcal{B}} \min_{P \in \Uc^{\texttt{sa}}_p(b)} J^\pi_P.
\end{align}
In essence, we have reduced a challenging problem into an infinite number of simpler ones. However, our task is not complete yet. While there exists a closed-form expression for \(J^\pi_{\Uc^{\texttt{sa}}_p} = J^\pi - \sum_{s,a} d^\pi(s,a) b_{sa} \sigma_q(v^\pi_{\Uc^{\texttt{sa}}_p})\), where \(\sigma_q(v^\pi_{\Uc^{\texttt{sa}}_p})\) represents the \(q\)-generalized standard deviation (GSTD) of the robust value function as defined in \cite{LpRMDP}, this formulation is still impractical, as  \(\max_{b\in\B}\sum_{s,a} d^\pi(s,a) b_{sa}\sigma_q(v^\pi_{\Uc^{\texttt{sa}}_p(b)})\) remains computationally challenging.

To address this issue, we turn to the dual formalism developed in the next section.

\subsection{Dual Formulation of Robust MDPs}
In this section, we derive, for the first time, a dual formulation for robust MDPs. While it is more complex than the dual formulation for non-robust MDPs and applies specifically to \(L_p\)-bounded uncertainty sets, it lays the groundwork for all the subsequent results in the paper.

From \cite{LpPgRMDP}, we know \texttt{sa}-rectangular worst-case kernel \(P^\pi_{\Uc^{\texttt{sa}}_p(b)}= \hat{P} - b k^T\) is a rank-one perturbation of the nominal kernel, where $k \in \K :=\{k|\mid \norm{k}_p\leq 1, \mathbf{1}^\top k = 0 \}$. Hence, it is enough for the adversary to focus on the rank-one perturbations, allowing us to  rewrite \eqref{main:eq:rr:BU} as 
\[J^\pi_{\Uc_p} = \min_{b \in \mathcal{B}} \min_{k\in\K} J^\pi_{\hat{P}-bk^\top} = \min_{b \in \mathcal{B}} \min_{k\in\K} \mu^\top D^\pi_{\hat{P}-bk^\top}R^\pi,
\]
where the last equality comes from  $J^\pi_P = \mu^TD^\pi_PR^\pi$. Further, as shown in Lemma 4.4 of \cite{LpPgRMDP}, applying the Sherman–Morrison formula \cite{ShermanMorrisonFormula} (see Proposition \ref{app:bg:ShermanMorrisonFormula}), the robust return can be expressed as:  
\[
J^\pi_{\Uc_p} = \min_{b \in \mathcal{B}, k \in \mathcal{K}} \Big[\mu^\top D^\pi R^\pi - \gamma \mu^\top D^\pi b^\pi \frac{k^\top D^\pi R^\pi}{1 + \gamma k^\top D^\pi b^\pi} \Big],
\]
where \(b^\pi_s := \sum_{a} \pi(a|s) b_{sa}\). The following result presents a more compact and interpretable form of this expression.

\begin{lemma}\label{main:rs:RPE:dual}
The robust return can be expressed as:
\[
J^\pi_{\Uc_p} = J^\pi - \gamma \max_{b \in \mathcal{B}, k \in \mathcal{K}} \frac{\innorm{k, v^\pi_R} \innorm{d^\pi, b^\pi}}{1 + \gamma \innorm{k, v^\pi_b}},
\]
where \(v^\pi_b = D^\pi b^\pi\) represents the value function with uncertainty radius \(b\) as the reward vector.
\end{lemma}
For the first time, the above result expresses the robust return in terms of the nominal return \(J^\pi\) and a penalty term involving only nominal values (\(d^\pi\), \(v^\pi_R = v^\pi\), and \(v^\pi_b\)). Notably, the denominator term \(1 + \gamma \innorm{k, v^\pi_b}\) is strictly positive (see appendix for details).

In the subsequent subsections, we delve deeper into evaluating this penalty term and analyzing the nature of the optimal \((k, b)\) for a given policy \(\pi\), revealing the adversary.

Additionally, by maximizing the robust return \(J^\pi_{\Uc_p}\) over policies, we derive a novel dual formulation, as stated below.

\begin{theorem} The optimal robust return is the solution to  
\[J^*_{\Uc_p} = \max_{D\in \mathcal{D} }\min_{k\in\mathcal{K},b\in\mathcal{B}}\Bigm[\mu^TD R -\gamma \mu^TDb \frac{k^TD R}{1 +\gamma  k^TD b}\Bigm]\]
where $\mathcal{D} = \bigm\{D^\pi H^\pi \mid \pi \in \Pi \bigm\}$  and  $H^\pi R := R^\pi, D^\pi = (I-P^\pi)^{-1} $.
\end{theorem} 
Notably, the dual formulations for the \texttt{sa}-rectangular and \texttt{s}-rectangular cases differ in their definitions of \(\mathcal{B}\): for the \texttt{sa}-rectangular case, \(\mathcal{B} = \{\beta\}\), whereas for the \texttt{s}-rectangular case, \(\mathcal{B} = \{b \in \mathbb{R}^{\mathcal{S} \times \mathcal{A}} \mid \|b_s\|_p \leq \beta_s\}\), as detailed in the appendix.

The result above formulates the dual of robust MDPs as a min-max problem, which is  insightful and significant in itself. However, the set \(\mathcal{D}\) may be non-convex, as suggested by Figure \ref{fig:main:pca_projections_s3_a2} (details in the appendix), making the problem non-convex. We leave this as an open question for future work: 
\begin{quote}
    How can we effectively exploit the above dual form for more insights and better algorithmic design?
\end{quote}

In this paper, we first develop a method to approximate the worst-case kernel for robust policy evaluation, as discussed in the next section. We then derive policy gradient methods to facilitate policy improvement.

% \textbf{Open Question.} One of the key questions raised by the above result is whether the set of occupation matrices \(\mathcal{D}\) is convex. For non-robust MDPs, the dual optimization is over the set of occupation measures \(\{d^\pi H^\pi \mid \pi \in \Pi\}\), which is well-known to be convex \cite{altman-constrainedMDP}. However, in our case, the optimization involves the set of occupation matrices \(\mathcal{D}\), and its convexity remains an open question. Nevertheless, Figure \ref{fig:main:pca_projections_s3_a2} and other extensive experimental results in appendix,  suggest the convexity of  \(\mathcal{D}\). 

%  If \(\mathcal{D}\) is indeed convex, can the problem be framed as a  minimax fractional programming \cite{Minmax_Fractional_Programming} or addressed with other convex optimization tools \cite{Boyd_Convex_Optimization}.

% However, these questions remain open and require further investigation. 

\begin{figure}
    \centering
\includegraphics[width=8cm, height=4cm]{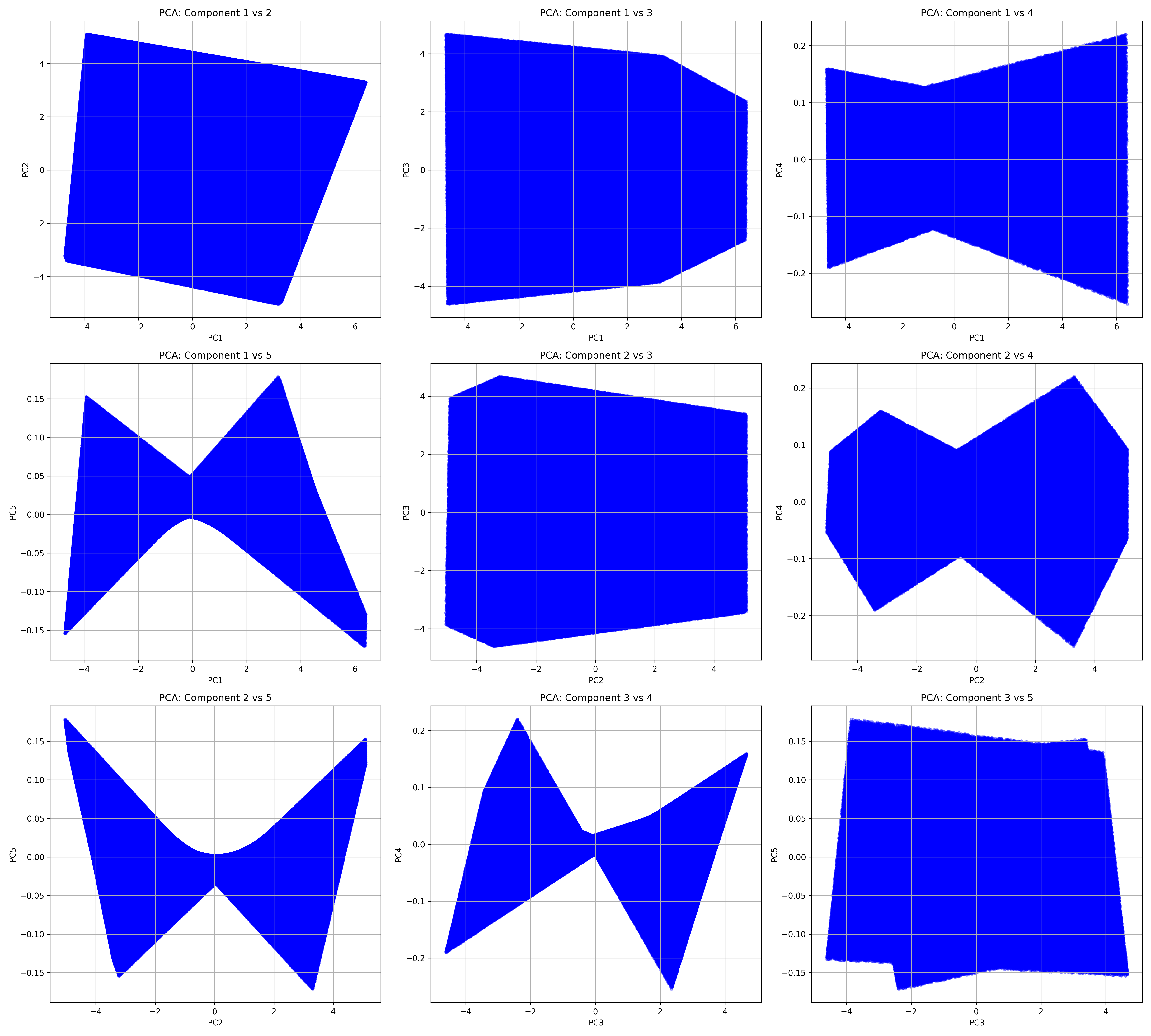}
    \caption{Projections of set $\D$ along principal components, for \(S=3, A=2\) with $10$ millions samples. This figure strongly suggests the non-convexity of the set. }
\label{fig:main:pca_projections_s3_a2}
\end{figure}
    
    % \begin{figure}[h!]
    %     \centering
    %     \includegraphics[width=8cm, height=5cm]{images/multiple_2D_projections_large_S10A5_27.png}
    %     \caption{PCA projections of the fundamental matrices \(D^\pi\) for \(S=10, A=5\). The plots show pairwise projections of the first five principal components, suggesting a convex region.}
    %     \label{fig:main:pca_projections_s10_a5}
    % \end{figure}

\subsection{Robust Policy Evaluation}
\label{subsection: robust_policy evaluation}
In this section, we propose an algorithm for robust policy evaluation and establish its performance guarantees. The following result states that the robust return can be computed using the function:
\[
F(\lambda) = \max_{b \in \mathcal{B}} \|E^\pi_\lambda b\|_q,
\]
where $
E^\pi_\lambda := \gamma \Big(I - \frac{\mathbf{1}\mathbf{1}^\top}{S}\Big) \Big[D^\pi R^\pi \mu^\top D^\pi - \lambda D^\pi\Big] H^\pi,$ and \( H^\pi R := R^\pi \) consists of easily computable quantities.
\begin{lemma}\label{main:rs:RPE:binSearch}
The robust return can be expressed as:
\[J^\pi_{\mathcal{U}_p} = J^\pi - \lambda^*,\]
where the penalty \(\lambda^*\) is a fixed point of \(F(\lambda)\). Furthermore, \(\lambda^*\) can be found via binary search since:
\[
F(\lambda) > \lambda \quad \text{if and only if} \quad \lambda > \lambda^*.
\]
\end{lemma}
The proof of this result is provided in Appendix (Lemma \ref{app:rs:RPE:binSearch}). Further, the bisection property of $F$ established in the result, directly implies the linear convergence rate of Algorithm \ref{alg:RPE}, as stated below.
\begin{algorithm}
\caption{Binary Search for Robust Policy Evaluation}\label{alg:RPE}
\begin{algorithmic} [1]
\STATE \textbf{Initialize:} Upper limit  
\(\lambda_{u} = \frac{1}{1-\gamma},\)  lower limit
\(\lambda_{l} = 0 \)  
\WHILE{ not converged: $n = n+1$ }
    \STATE \textbf{Bisection value:}  \(\lambda_{n} = (\lambda_{l} + \lambda_{u}) / 2\)
    \STATE \textbf{Bisection:} \(\lambda_{l} = \lambda_n\) if  \( F(\lambda_n) > \lambda_n \),   else \(\lambda_{u} = \lambda_n\). 
    \STATE \textbf{ Update robust return:} $J_{n} =  J^\pi - \lambda_{n} $.
\ENDWHILE
\end{algorithmic}
\end{algorithm}
\begin{theorem} Algorithm \ref{alg:RPE} converges linearly, \ie
\[
J_n - J^\pi_{\mathcal{U}_p} \leq O(2^{-n}).
\]
\end{theorem}
We conclude that robust evaluation can be performed efficiently with linear iteration complexity. However, each iteration requires solving \(\max_{x \in \mathcal{B}} \|Ax\|_q\) as a subroutine in Algorithm \ref{alg:RPE}. We focus specifically on the simplified case of \( p = 2 \), i.e., 
$\max_{\|x\|_2 \leq 1, \; x \succeq 0} \|Ax\|_2,$
 for which we proposed  a modified eigenvalue Algorithm \ref{alg:Ax:EH}.  This method has a time complexity of \( O(S^3 A^3) \), performs  very effectively (\ie very similar performance to numerical method scipy.minimize and order of magnitude of faster), with more details  in Appendix \ref{app:sec:NormEvaluation}.

\begin{algorithm}[ht]
\caption{Spectral method for $\max_{x\in\B} \|Ax\|_2$}
\label{alg:Ax:EH}
\begin{algorithmic}[1]
\STATE Compute eigenvector $v_i$ and eigenvalues $\lambda_i$ of $A^\top A$
\STATE WLOG let $\|v_i^+\|_2 \geq \|v_i^-\|_2$ where $
    v_i^+ = \max(v_i, 0), v_i^- = -\min(v_i, 0)$
\STATE Compute best score :
$
j = \argmax_{i}\lambda_i \langle v_i, \frac{v_i^+}{\|v_i^+\|_2} \rangle.$
\STATE \textbf{Output:} Approximate maximum value $\beta\|A \frac{v_j^+}{\|v_j^+\|_2}\|_2$.
\end{algorithmic}
\end{algorithm}

Performance of robust policy evaluation Algorithm \ref{alg:RPE}, is further validated experimentally in the later section.
 % Illustrate, for
% $\max_{\|x\|_2 \leq 1, \; x \succeq 0} \|Ax\|_2,$ for a random $A\in\R^{50\times 50}$, our spectral method outputs the value of $11.74$ a very similar value (11.94) to scipy.minimize with 75 times less time). Further, Table \ref{tb:Ax} illustrate the relative time and relative value achieved by taking empirical minimum by of $N$ samples, w.r.t. our spectral method. Our spectral methods performs similar to numerical optimization but much faster.

% \begin{table}[h]
%     \centering
%     \begin{tabular}{c|c|c|c|c|c|c}
%     \toprule
%     N &1& $10$ & $10^2$&$10^3$ &$10^4$&$10^5$\\
%     \midrule
%      Value&0.003 &0.64& 0.72& 0.72 & 0.73&0.76\\
%      Time& 0.04 & $0.1$ & 0.8&6.6& 300&14694\\
%     \bottomrule
%     \end{tabular}
%     \caption{Relative value achieved and relative time take to approximate $\max_{x\in\B} \norm{Ax}_2$ with $N$ number of samples by random sampling, w.r.t. our spectral Algorithm \ref{alg:Ax:EH}, for a random $A \in \R^{50\times 50}$}. 
%     \label{tb:Ax}
% \end{table}

% In this section, assuming that the optimization problem:

% \[
% \max_{b \in \mathcal{B}} \| A b \|_q
% \]

% can be efficiently solved, we have established that robust policy evaluation can be performed efficiently for the uncertainty sets \(\mathcal{U}_p\). In the next subsection, we discuss methods to evaluate this optimization, particularly in the case when \( p = 2 \).
\section{Revealing the Adversary}
Non-rectangular robust MDPs have been sparsely studied in the literature, and the nature of the adversary remains  unexplored. The following result provides the first insights into the adversary's behavior for non-rectangular uncertainty sets. It establishes that the worst-case transition kernel is a rank-one perturbation of the nominal kernel, similar to the case of rectangular uncertainty sets (see Proposition \ref{main:bg:worstPm}). However, its exact structure is significantly more intricate.

\begin{theorem}[Non-rectangular Worst Kernel] The worst kernel  for the policy $\pi$ and the uncertainty set $\Uc_p$ is 
\begin{align*}
     P^\pi_{\Uc_p}= \hat{P} - bk^\top, 
\end{align*}
where $(k,b)$ is a solution to $\max_{k\in\mathcal{K},b\in\mathcal{B}}\bigm[J^\pi_b\frac{\innorm{k,v^\pi_R}}{1 +\gamma  \innorm{k,v^\pi_b}}\bigm]$. 

\end{theorem}
The result shows that the adversary controls two variables, \( \beta \) and \( k \), with the objective of:  
\begin{itemize}
    \item \textbf{Maximizing the average uncertainty in the trajectories} \( J^\pi_\beta \), since the more frequently the agent visits high-uncertainty states, the greater the adversary's ability to steer it toward unfavorable states.  
    \item \textbf{Choosing the perturbation direction} \( k \) to maximize \( k^\top v^\pi_R \), forcing the agent into low-reward trajectories, while simultaneously minimizing \( k^\top v^\pi_b \), ensuring high exposure to high-uncertainty states.
\end{itemize}
These insights provide a deeper understanding of the adversary and can aid in designing more resilient robust algorithms.

\begin{tcolorbox}[colframe=black, colback=gray!10, title={\centering Message to Practitioners}]The adversary focuses solely on  rank-one perturbations of the nominal kernel, iteratively boosting its influence by pushing the agent into high-uncertainty states, then leveraging that influence to steer the agent toward low-reward trajectories, ultimately driving the agent to the lowest possible return.

\end{tcolorbox}

\section{Policy Improvement}
Once a worst kernel for policy is obtained using Algorithm \ref{alg:RPE}, we can compute the policy gradient to update the policy. Alternatively, we can use the policy gradient theorem derived in the result below. 

\begin{lemma}\label{main:rs:PGT}[Policy Gradient] Given the worst transition kernel $P^\pi_{\Uc_p} = \hat{P}-b k^\top$, the gradient is given as 
\begin{align*}&
\nabla _\pi  J^\pi_{\Uc_p} = d^\pi\circ Q^\pi_{R} - \gamma\frac{  k^{\top}v^\pi_R }{1 + \gamma k^{\top}v^\pi_b}d^\pi\circ Q^\pi_{b}\\&- \gamma\frac{  J^\pi_b( k^\top D^\pi)}{1 + \gamma k^{\top}v^\pi_b}\circ Q^\pi_{R}+\gamma^2\frac{  J^\pi_b (k^{\top}v^\pi)( k^\top D^\pi)}{(1 + \gamma k^{\top}v^\pi_b)^2}\circ Q^\pi_{b},
\end{align*}
where $(u\circ v)(s) := u(s)v(s)$.
\end{lemma}

Note that the above result expresses robust gradient only in nominal terms, displaying the complex interplay of different parameters.

The first term $d^\pi\circ Q^\pi_R$ is a nominal policy gradient, trying to improve the policy's emphasis (weight) for high reward actions. The first part of the second term  $\gamma\frac{  k^{\top}v^\pi_R }{1 + \gamma k^{\top}v^\pi_b} = \sigma_q(v^\pi_{\Uc_p})$ is GSTD, hence, always positive, measuring the vulnerability towards the adversary actions. This scales the other term $d^\pi\circ Q^\pi_{b}$, that is policy gradient w.r.t. the reward being the uncertainty radius. To summarize, the second term discourages the policy's emphasis (weight) in high uncertainty Q-value by the amount proportional to the vulnerability.  

The last two terms are much more complex to interpret, showcasing the complexity of robust MDPs .

\begin{theorem} Robust policy gradient algorithm \ref{alg:RPG}, convergence to an $\epsilon$-optimal policy in a total of $O(\epsilon^{-8})$ iterations.
\end{theorem}
% Policy evaluation for a non-rectangular robust MDP, a problem considered strongly NP-hard, is shown in the above result for the uncertainty set \( \mathcal{U}_2 \) to have a time complexity that is, impressively, the same—up to a factor of \( S \)—as that of non-robust policy evaluation, with a complexity of \( O(S^2 A \log \epsilon^{-1}) \).

 The policy gradient method in \cite{RPG_conv} requires \( O(\epsilon^{-4}) \) iterations to converge to the globally optimal robust policy. At the \( n \)-th policy gradient step, the guarantee in \cite{RPG_conv} necessitates an \( O(\gamma^{-n}) \)-close approximation of the worst-case kernel for the policy \( \pi_n \), which our Algorithm \ref{alg:RPE} computes in \( O(n) \) iterations. Consequently, the overall iteration complexity to achieve the global optimal robust policy becomes \( O\left(\sum_{n=1}^{\epsilon^{-4}} n\right) = O(\epsilon^{-8}) \).

Algorithm \ref{alg:RPG}, is a double loop algorithm: That is, the inner loop (Algorithm \ref{alg:RPE}) computes an approximate worst kernel for the fixed policy. On the other hand, the outer loop updates the policy using the gradient obtained using the worst kernel. Alternatively, an actor-critic style algorithm can also be obtained, where the worst kernel and the policy are updated simultaneously. We leave this direction for a future work.   

\begin{algorithm}
\caption{ Robust Policy Gradient Algorithm }\label{alg:RPG}
\begin{algorithmic} [1]
\WHILE{ not converged: $n=n+1$ }
\STATE Get worst kernel $P =\hat{P}-bk^\top$ for policy $\pi$ from Algorithm \ref{alg:RPE} with tolerance $\epsilon =\gamma^n$..
\STATE Compute gradient $G$ from Lemma \ref{main:rs:PGT} .
\STATE Policy update: $\pi \to Proj\Bigm[\pi + \alpha_n G\Bigm].$
\ENDWHILE
\end{algorithmic}
\end{algorithm}

\section{Experiments}
In this section, we evaluate the performance of Algorithm \ref{alg:RPE} for robust policy evaluation, focusing on the case \( p = 2 \). The algorithm requires computing \( F(\lambda) \) at each iteration, which involves solving the constrained matrix norm problem \( \max_{x \in \mathcal{B}} \|Ax\|_2 \). This can be efficiently handled using our spectral Algorithm \ref{alg:Ax:EH}. Figures \ref{main:fig:rve:value} and \ref{main:fig:rve:time} compare four methods for computing the robust return, specifically the penalty term \( J^\pi - J^\pi_{\Uc_2} \):  
\begin{itemize}
    \item \textbf{ SLSQP from scipy \cite{2020SciPy-NMeth} :} This is  A semi-brute force approach that uses Lemma \ref{main:rs:RPE:dual}. It computes the penalty term via scipy.minimize to directly optimize over ($b,k$). This is equivalent to optimize over only rank-one perturbation of nominal kernel as a\( (b, k) \) corresponds to selecting a rank-one perturbation of the nominal kernel \( P = \hat{P} - bk^\top \).  Note that this method is local, hence can sometime be stuck in very bad local solution. 
    
\item \textbf{Binary search Algorithm \ref{alg:RPE} :} Uses the binary search Algorithm \ref{alg:RPE}, and  spectral Algorithm \ref{alg:Ax:EH} for computing \( F(\lambda) \) in each iteration.

\item \textbf{Random Rank-One Kernel Sampling :} A semi-brute force approach that uses Lemma \ref{main:rs:RPE:dual} to sample random pairs \( (b, k) \in \mathcal{B}, \mathcal{K} \), empirically maximizing the penalty term. Since choosing \( (b, k) \) corresponds to selecting a rank-one perturbation of the nominal kernel \( P = \hat{P} - bk^\top \), the method is named accordingly.  
\item \textbf{Random Kernel Sampling :} A brute-force approach that samples random kernels directly from the uncertainty set \( \Uc_2 \), computing the empirical minimum as an estimate of the robust return.  

\end{itemize}

\begin{figure}
    \centering
    \includegraphics[width=7cm, height=4.2cm]{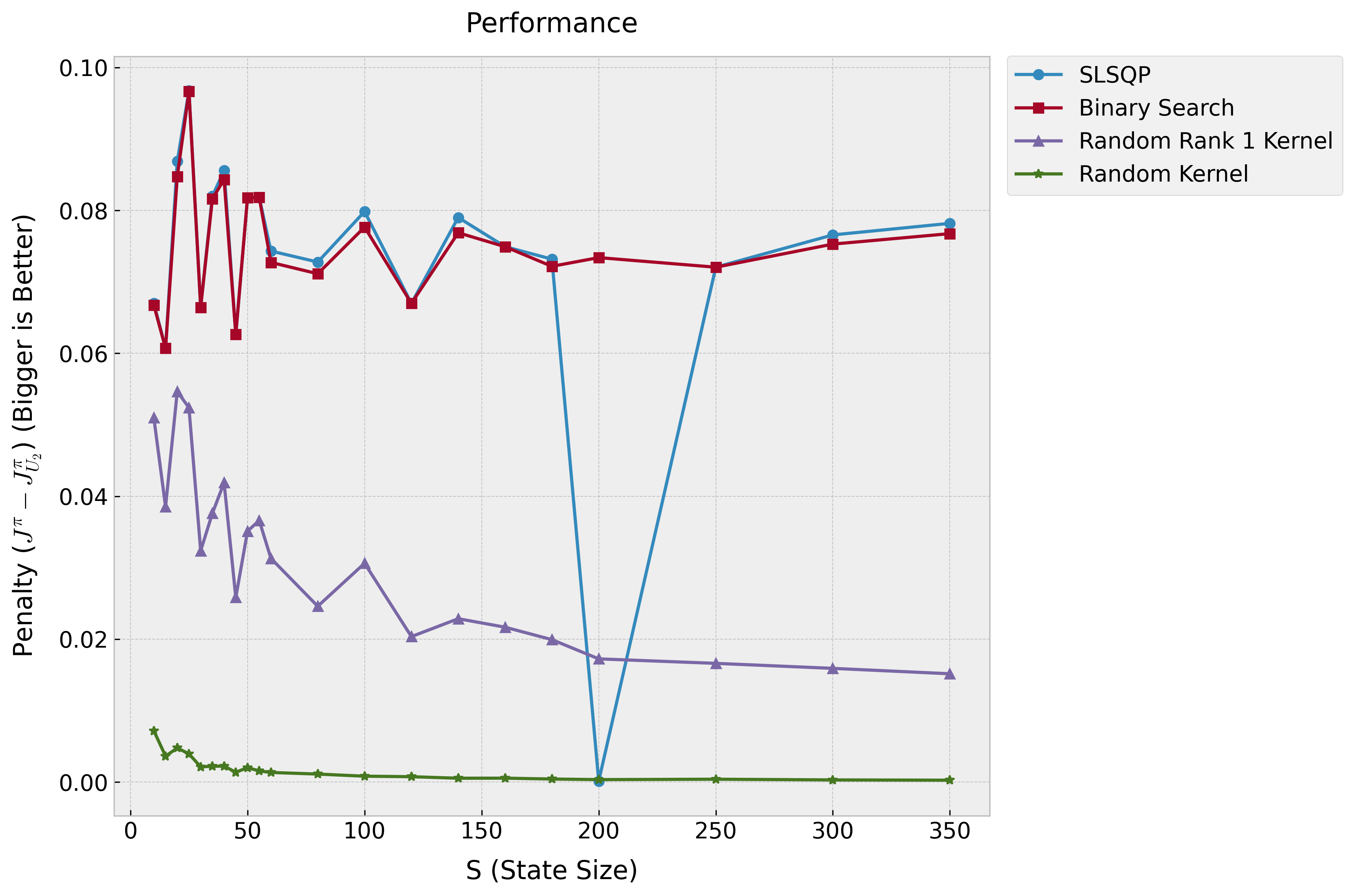}
    \caption{Performance of Robust Policy Evaluation methods with equal amount of time, with fixed action space $A=8$}
    \label{main:fig:rve:value}
\end{figure}
Figure \ref{main:fig:rve:value} presents the penalty value (\(J^\pi - J^\pi_{\Uc_2}\)) computed by different methods across various state space sizes while keeping the action space $(A=8)$ fixed, with each method given the same computational time. Our binary search Algorithm \ref{alg:RPE} combined with spectral Algorithm \ref{alg:Ax:EH} performs significantly better to  brute-force (random kernel sampling) and semi-brute (random sampling of rank-one perturbations of the nominal kernel) approaches. Notably, the scipy SLSQP variant performs slightly better on average, but our Algorithm \ref{alg:RPE} is more reliable. This is expected, as the spectral method used to compute $F(\lambda)$ in Algorithm \ref{alg:RPE}, is global, while scipy SLSQP is a local optimizer and thus more prone to getting stuck in suboptimal solutions (as evident in the figure).

\begin{figure}
    \centering
    \includegraphics[width=7cm, height=4.2cm]{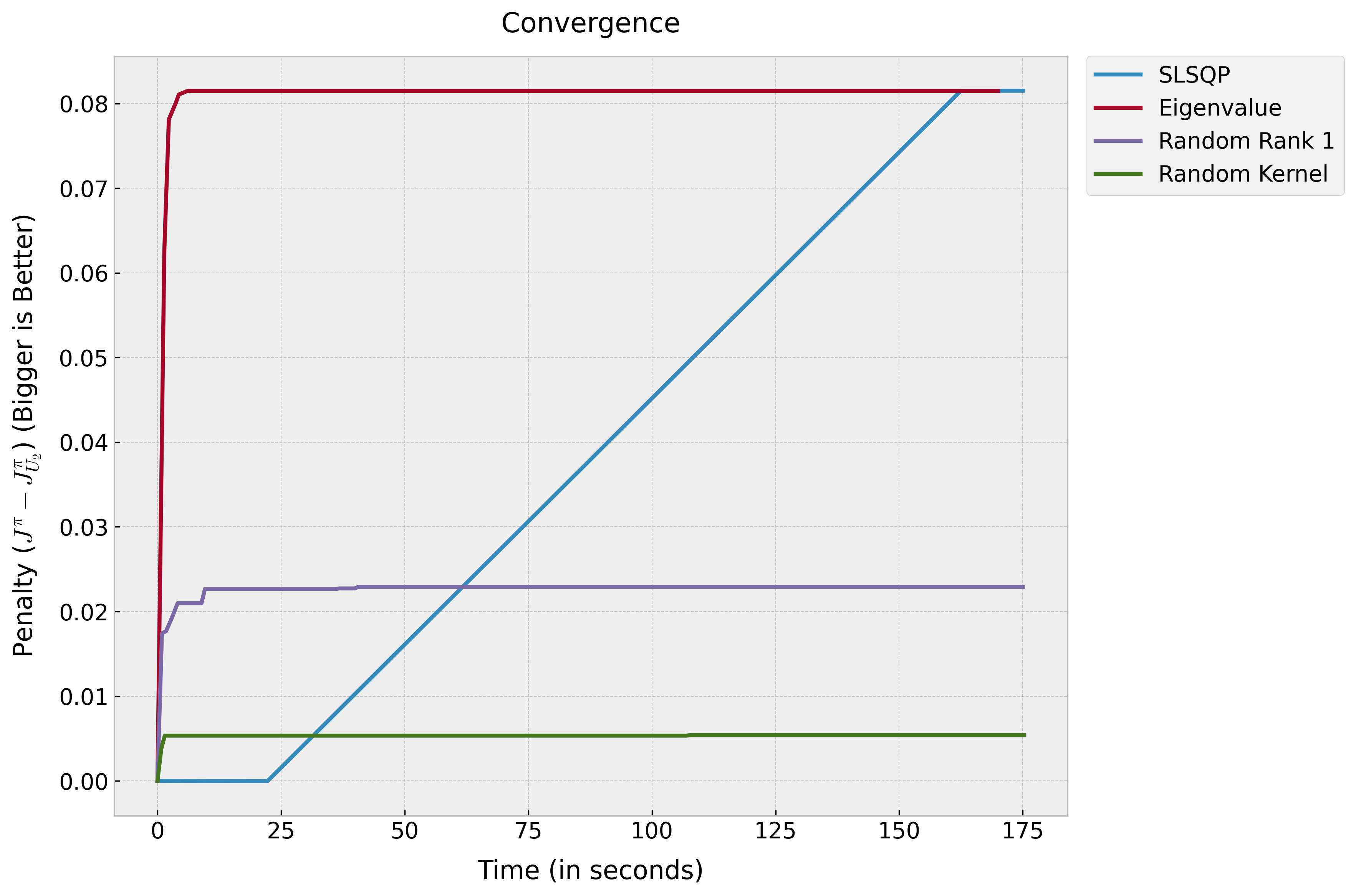}
    \caption{Convergence of Robust Policy Evaluation Methods, with fixed $S=128,A=8$.}
    \label{main:fig:rve:time}
\end{figure}
Figure \ref{main:fig:rve:time} illustrates the convergence of different approaches over time (in seconds) for a fixed state space size (\(S=8\)). Algorithm \ref{alg:RPE} converges rapidly, while the SLSQP variation gradually approaches the same performance. In contrast, the brute-force method shows slow, logarithmic improvement. This behavior arises because brute-force methods require an exponential number of samples (in the dimensionality of the problem) to adequately explore all directions. In comparison, our binary search Algorithm combined with the spectral Algorithm \ref{alg:Ax:EH} achieves significantly better efficiency, with a complexity of \(O(S^3 A^3 \log \epsilon^{-1})\).

More details of these experiments along with others can be found in the appendix, and codes are available at \url{https://anonymous.4open.science/r/non-rectangular-rmdp-77B8}. 

Our experiments confirm the efficiency of our binary search Algorithm \ref{alg:RPE} for robust policy evaluation, significantly outperforming brute-force approaches in both accuracy and convergence speed.

\section{Discussion}We studied robust Markov decision processes (RMDPs) with non-rectangular \(L_p\)-bounded uncertainty sets, balancing expressiveness and tractability. We showed that these uncertainty sets can be decomposed into infinitely many \(sa\)-rectangular sets, reducing robust policy evaluation to a min-max fractional optimization problem (dual form). This novel dual formulation provides key insights into the adversary and leads to the development of the first robust policy evaluation algorithms. Experiments demonstrate the effectiveness of our approach, significantly outperforming brute-force methods. These findings further pave the way for scalable and efficient robust reinforcement learning algorithms.

\paragraph{Future Work.} Our results naturally extend to uncertainty sets that can be expressed as a finite union of \(L_p\) balls. Furthermore, any uncertainty set can be approximated using a finite number of \(L_p\) balls, with smaller balls providing a better approximation. However, the number of balls required for an accurate approximation may grow prohibitively large. While this work is limited to \(L_p\) norms, it may be possible to generalize the approach to other types of uncertainty sets. A key challenge in such an extension would be identifying the structure of the worst-case kernel and developing the corresponding matrix inversion techniques.

Another promising direction is to design policy improvement methodologies compatible with deep neural networks. 

\bibliography{references}
\bibliographystyle{unsrt}

%%%%%%%%%%%%%%%%%%%%%%%%%%%%%%%%%%%%%%%%%%%%%%%%%%%%%%%%%%%%%%%%%%%%%%%%%%%%%%%
%%%%%%%%%%%%%%%%%%%%%%%%%%%%%%%%%%%%%%%%%%%%%%%%%%%%%%%%%%%%%%%%%%%%%%%%%%%%%%%
% APPENDIX
%%%%%%%%%%%%%%%%%%%%%%%%%%%%%%%%%%%%%%%%%%%%%%%%%%%%%%%%%%%%%%%%%%%%%%%%%%%%%%%
%%%%%%%%%%%%%%%%%%%%%%%%%%%%%%%%%%%%%%%%%%%%%%%%%%%%%%%%%%%%%%%%%%%%%%%%%%%%%%%
\newpage
\appendix
\onecolumn
\section{Summary of Notations and Definitions}
For a set $\mathcal{S}$, $\lvert\mathcal{S}\rvert$ denotes its cardinality. $\langle u, v\rangle := \sum_{s\in\mathcal{S}}u(s)v(s)$ denotes the dot product between functions $u,v:\mathcal{S}\to\mathbb{R}$. $\lVert v\rVert_p^q :=(\sum_{s}\lvert v(s)\rvert^p)^{\frac{q}{p}}$ denotes the $q$-th power of $L_p$ norm of function $v$, and we use  $\lVert v\rVert_p := \lVert v\rVert^1_p$ and $\lVert v\rVert := \lVert v\rVert_2$ as shorthand. For a set $\mathcal{C}$, $\Delta_{\mathcal{C}}:=\{a:\mathcal{C} \to \mathbb{R}|a_c\geq 0, \forall c, \sum_{c\in\mathcal{C}}a_c=1\}$ is the probability simplex over $\mathcal{C}$.
 $\mathrm{var}(\cdot)$ is variance function, defined as $\mathrm{var}(v)  = \sqrt{\sum_{s\in\mathcal{S}}(v(s) - \bar{v})^2}$ 
where $\bar{v} = \frac{\sum_{s\in\mathcal{S}}v(s)}{\lvert\mathcal{S}\rvert}$ is the mean of function $v:\mathcal{S}\to\mathbb{R}^d$.
$\mathbf{0},\mathbf{1} $ denotes all zero vector and all ones vector/function respectively of appropriate dimension/domain. $\mathbf{1}(a=b):=1$ if $a=b$, 0 otherwise, is the indicator function. For vectors $u,v$, $\mathbf{1}(u\geq v)$ is component wise indicator vector, i.e. $\mathbf{1}(u\geq v)(x) = \mathbf{1}(u(x)\geq v(x))$. $A\times B =\{(a,b)\mid a\in A, b\in B\}$ is the Cartesian product between set $A$  and $B$.

\begin{table}[ht]
  \caption{Useful Notations}
  \label{tb:relatedWork}
  \centering
  \begin{tabular}{lll}
   \toprule
    Notation&  Definition &Remark\\
   \midrule
   $p,q$&$\frac{1}{p}+\frac{1}{q}=1$& Holder's conjugates\\\\
   
   $\sigma_p$&& Standard deviation w.r.t. $L_p$ norm\\\\
   \midrule 
   $v^\pi, v^\pi_{P,R}$& $(I-\gamma P^\pi)^{-1}R^\pi$& Value function\\\\ 
   $D^\pi, D^\pi_{P,R}$& $(I-\gamma P^\pi)^{-1}$& Occupancy matrix\\\\ 
   $d^\pi, d^\pi_{P,\mu}$& $\mu^T(I-\gamma P^\pi)^{-1}$& Occupancy measure\\\\
   \midrule
$\Uc,\Uc^{\texttt{sa}}_p,\Uc^{\texttt{s}}_p,\Uc_p$&&  Uncertainty sets\\\\
 \end{tabular}
\end{table}

\section{ More on Setting, Assumptions, Results, Discussion}
\paragraph{Extension to KL Entropy Uncertainty Sets.}
For the KL uncertainty case, the worst kernel is given by $P^\pi_{\Uc^{sa}_{KL}} = (I-\gamma \hat{P}^\pi A^\pi)^{-1}$ where $A^\pi$ is a diagonal matrix \cite{Ewok}. If we can invert this matrix, then its possible to build upon it. We leave this for future work.
% \paragraph{Well conditioning of fractional programming.}
% In the results, we have $\frac{f(b,k)}{g(b,k)}$ where $f(b,k), g(b,k)$ is numerator and denominator respectively of $\frac{ \innorm{d^\pi, b^\pi}
%     \innorm{k,v^\pi_R}}{1 + \gamma\innorm{d^\pi, b^\pi}}$.  And we require $g(k)>0$ for all $k$ for $\frac{f}{g}$ to be well defined.
% To begin with the assumption that uncertainty radiuses are small enough that all the kernels in the uncertainty set are well defined, implies $\norm{b}_\infty = O(\frac{1}{S})$. This makes $\norm{v^\pi_b}_\infty = O(\frac{1}{(1-\gamma)S})$, hence $\gamma \innorm{k,v^\pi_b} = O(\frac{\gamma}{2(1-\gamma)S^{\frac{1}{p}}}) << 1$. This makes the $g(b,k)>0$ for all $k,b$.

% A another way to see  it as , is  for $P = \hat{P} -bk^T$,  we have $J^\pi_{P}-J^\pi_{\hat{P}} = \frac{f(b,k)}{g(b,k)}$ where $J^\pi_{\hat{P}}$ is finite. Morover, if we could prove $J^\pi_{P} = \sum_{n=0}^{\infty} \mu^T(P^\pi_0-b^\pi k^T)^nR^\pi$ is finite, we are done.

\subsection{Revealing the Adversary}

The worst kernel  for the policy $\pi$ and the uncertainty set $\Uc^{\texttt{sa}}_p/\Uc^{\texttt{s}}_p$, is given as 
\begin{align}
    P^\pi_{\Uc^{\texttt{sa}}_p}(\cdot|s,a) &= \hat{P}(\cdot|s,a) -\beta_{sa} k, \qquad\quad\text{and}\\
     P^\pi_{\Uc^{\texttt{s}}_p}(\cdot|s,a) &= \hat{P}(\cdot|s,a) -\beta_s \Bigm(\frac{\pi(a|s)}{\norm{\pi_s}_q}\Bigm)^{q-1} k, 
\end{align}
where $k$ is a solution of $\max_{k\in\mathcal{K}}\frac{k^Tv^\pi_R}{1 +\gamma  k^Tv^\pi_\beta}$. The relation sheds the light on the working of the adversary in \eqref{eq:kappaDual}. Basically, $k$ is the direction in which the adversary discourages the perturbation of the kernel. And the optimal direction $k$ that the adversary chooses is the one that maximizes the potential gain in the reward ($\innorm{k,v^\pi_R}$) and minimizes the average uncertainty on the trajectories $(\innorm{k,v^\pi_\beta})$.

Compared to the existing characterization of the adversary  in \cite{LpPgRMDP}, see Proposition \ref{bg:worstPm} , where the direction $k = u^\pi_\Uc$ is in robust terms, this is the first time, we have a complete overview of the working of the adversary in nominal terms.

\begin{theorem}[Non-rectangular Worst Kernel] The worst kernel  for the policy $\pi$ and the uncertainty set $\Uc_p$ is 
\begin{align*}
     P^\pi_{\Uc_p}= \hat{P} - bk^\top, 
\end{align*}
where $(k,b)$ is a solution to $\max_{k\in\mathcal{K},b\in\mathcal{B}}\bigm[J^\pi_b\frac{\innorm{k,v^\pi_R}}{1 +\gamma  \innorm{k,v^\pi_b}}\bigm]$. 

\end{theorem}
Moreover, we can see that the adversary  wants to optimize three things: A) $ J^\pi_\beta$, $ k^\top v^\pi_R$  and $k^\top v^\pi_\beta$ with two variables $\beta $ and $k$.  
\begin{enumerate}
    \item Maximizing the average uncertainty in the trajectories $ J^\pi_\beta$. This makes sense, as the more the agent visits states with high uncertainty, the higher is the ability of the adversary to undermine it. 
    \item Choosing the perturbation direction $k$ that discourages the agent to transition into good value states (w.r.t. nominal value function $v^\pi_R$). This observation was also seen in a previous study \cite{LpPgRMDP}, for sa and s rectangular uncertainty sets but it was w.r.t. robust value function $v^\pi_\Uc$. Whereas, the result has the exact finding except it is w.r.t. nominal value function, which is known contrast to the robust value function.
    \item  Choosing the uncertainty radius vector $\beta$ and the perturbation direction $k$, such that the $k^Tv^\pi_\beta$ is minimized. This can be done, by putting negative entries of $k$ at maximal entries of $v^\pi_\beta$, and positive entries of $k^T$ at states which has minimum uncertainty value function. In other words, we want to perturb the kernel such that visitation of high uncertainty states in the long term is maximized.
\end{enumerate}

\section{Background }
\subsection{Robust MDPs}

A robust Markov Decision Process (MDP) is a tuple $(\mathcal{S},\mathcal{A},R,\Uc,\gamma,\mu)$ which generalizes the standard MDP,  by containing a set of environments $\mathcal{\Uc}$ \cite{Iyenger2005,Nilim2005RobustCO}. The reward robust MDPs is well-studied in the precious work of rectangular 
\cite{derman2021twice,LpRMDP} and non-rectangular \cite{LpRewardRobust} uncertainty sets. Hence, in this work , we consider only uncertainty in kernel which is much more challenging. 

The robust return of a policy $\pi$, is its performance over the uncertainty set $\mathcal{U}$, defined as $ J^\pi_{\mathcal{U}} =\min_{P\in\mathcal{U}} J^\pi_{P} $.
The objective is to find an optimal robust policy $\pi^*_\mathcal{U}$ that achieves the optimal robust performance $ J^*_{\mathcal{U}}$, defined as 
\begin{align}\label{obj:RobMDPs} J^*_{\mathcal{U}} = \max_{\pi} J^\pi_{\mathcal{U}}. 
\end{align}

Unfortunately, a general solution to the robust objective  \eqref{obj:RobMDPs}, is proven to be strongly NP-hard for both non-convex sets and convex ones \cite{RVI}. The robust value function $v^\pi_\mathcal{U}$ and optimal robust value function $v^*_\mathcal{U}$ \cite{Iyenger2005,Nilim2005RobustCO}, for any uncertainty set $\mathcal{U}$ can be defined state wise, for all $\pi$, as 
\[ v^\pi_\mathcal{U}(s) = \min_{
P\in\mathcal{U}}v^\pi_{P}(s),\qquad v^*_\mathcal{U}(s) = \max_{
\pi\in\mathcal{\Pi}}v^\pi_{P}(s).\]

\subsection{Rectangular Robust MDPs}

Unfortunately, a general solution to the robust objective  \eqref{obj:RobMDPs}, is proven to be strongly NP-hard for both non-convex sets and convex ones \cite{RVI}. Hence, it is common practice to take the \texttt{sa}-rectangular uncertainty sets $\mathcal{U}^{\texttt{sa}}$, where ambiguity in each state and action are independent \cite{Iyenger2005,Nilim2005RobustCO,Rcontamination,PG_RContamination}. Formally defined as
$\mathcal{U}^{\texttt{sa}}=\times_{(s,a) \in \mathcal{S}\times \mathcal{A}} \mathcal{P}_{s,a}$ is decomposed over state-action-wise where $\mathcal{P}_{s,a}$  are components sets.  

However, many classes of uncertainty sets arise in practice, where ambiguities in a given state are correlated. This type of uncertainty sets are captured by \texttt{s}-rectangular uncertainty sets $\mathcal{U}^{\texttt{s}}=\mathcal{P} = \times_{s \in \mathcal{S}} \mathcal{P}_{s}$, which can be decomposed state-wise as \cite{RVI}. Note that $\texttt{sa}$-rectangular uncertainty sets are a special case of it.

Fortunately, the decoupling structure in \texttt{s}-rectangular uncertainty sets allows the existence of a kernel and a reward function that minimizes the value function over the uncertainty set for each state for any given policy. Similarly, it allows the existence of an optimal robust policy that maximizes the robust value in each state \cite{RVI}. Mathematically, the robust value function is to be rewritten as  
 \begin{align*}
 v^*_{\mathcal{U}^{\texttt{s}}} = \max_{\pi}v^\pi_\mathcal{U^{\texttt{s}}},\qquad
    v^\pi_{\mathcal{U^\texttt{s}}} = \min_{{P\in\mathcal{U}^{\texttt{s}}}}v^\pi_{P}.
\end{align*}
Hence, the robust return  can be rewritten as 
$
 J^\pi_{\mathcal{U}}  =  \langle \mu,v^\pi_{\mathcal{U}}\rangle , \quad \text{and}\quad  J^*_{\mathcal{U}} = \langle \mu,v^*_{\mathcal{U}}\rangle.$
Most importantly, this rectangularity implies the existence of contractive robust Bellman operators, which are pivotal same as non-robust MDPs \cite{RVI}. Specifically, the robust value function $v^\pi_{\mathcal{U}}$, and the optimal robust value function $v^*_{\mathcal{U}}$  is the fixed point of the robust Bellman operator $\mathcal{T}^\pi_{\mathcal{U}}$ and the optimal robust Bellman operator $\mathcal{T}^*_{\mathcal{U}}$ respectively \cite{RVI,Iyenger2005}, defined as 
\begin{align*}
    &\mathcal{T}^\pi_{\mathcal{U}} v := \min_{ P\in\mathcal{U}}T^\pi_{ P }v, \quad \text{and}\quad \mathcal{T}^*_{\mathcal{U}}v := \max_{\pi}\mathcal{T}^\pi_{\mathcal{U}}v,
\end{align*}
where $T^\pi_{ P }v :=R^\pi + \gamma P^\pi v$ is non-robust Bellman operator \cite{Puterman1994MarkovDP}.

This contractive property of Bellman operators plays a central role in the s-rectangular robust MDPs. Unfortunately, these contractive Bellman operators do not exist for non-rectangular uncertainty sets. This makes the non-rectangular robust MDPs very unwieldy. 

% \subsection{Contraction of Rectangular Robust Bellman Operators}

% \begin{proposition}
%   \[ \norm{T^\pi_{\Uc^{\texttt{s}}}v - v^\pi_{\Uc^{\texttt{s}}}}_\infty \leq \gamma\norm{v - v^\pi_{\Uc^{\texttt{s}}}}_\infty \]  
% \end{proposition}
% \begin{proof}
%     \begin{align}
%         T^\pi_{\Uc^{\texttt{s}}}v - v^\pi_{\Uc^{\texttt{s}}} =& T^\pi_{\Uc^{\texttt{s}}}v -T^\pi_{\Uc^{\texttt{s}}}v^\pi_{\Uc^{\texttt{s}}} \\
%         =& \min_{P\in \Uc^{\texttt{s}}}T^\pi_{P}v -T^\pi_{\Uc^{\texttt{s}}}v^\pi_{\Uc^{\texttt{s}}} \\
%         =& \min_{P\in \Uc^{\texttt{s}}}T^\pi_{P}v -T^\pi_{P^\pi_{\Uc^{\texttt{s}}}}v^\pi_{\Uc^{\texttt{s}}} \\
%         \preceq & T^\pi_{P^\pi_{\Uc^{\texttt{s}}}}v -T^\pi_{P^\pi_{\Uc^{\texttt{s}}}}v^\pi_{\Uc^{\texttt{s}}} .
%         % \preceq & T^\pi_{P^\pi_{\Uc^{\texttt{s}}}}v -T^\pi_{P^\pi_{\Uc^{\texttt{s}}}}v^\pi_{\Uc^{\texttt{s}}} \\
%     \end{align}
% \end{proof}

% \begin{proposition}
%   \[ \norm{T^*_{\Uc^{\texttt{s}}}v - v^*_{\Uc^{\texttt{s}}}}_\infty \leq \gamma\norm{v - v^*_{\Uc^{\texttt{s}}}}_\infty \]  
% \end{proposition}

\subsection{Rectangular \texorpdfstring{$L_p$}{Lp} robust MDPs}
\label{main:sec:rectLp}

\begin{table}
  \caption{$p$-variance}
  \label{tb:kappa}
  \centering
  \begin{tabular}{l|ll}
    \toprule    
    $p$ & $\sigma_p(v)$ &Remark\\
 \midrule   $\infty$&$\frac{\max_{s}v(s) - \min_{s}v(s)}{2}$  &  Semi-norm \\\\ 
    $2$ &  $\sqrt{\sum_{s}\bigm(v(s) -\frac{\sum_{s}v(s)}{S}\bigm)^2}$ &Variance\\\\
    $1$
      &$ \sum_{i=1}^{\lfloor (S+1)/2\rfloor}v(s_i)$  -\\&$\sum_{i =\lceil (S+1)/2\rceil}^{S}v(s_i)$   & Top half - lower half\\
    \bottomrule
  \end{tabular}
  \begin{tabular}{l}
      where $v$ is sorted, i.e. $v(s_i)\geq v(s_{i+1}) \quad \forall i.$
  \end{tabular} 
\end{table}
Some useful definitions: $q$ is reserved for Holder's conjugate of $p$ that satisfies $\frac{1}{p}+ \frac{1}{q} =1$. The generalized $p$-variance function $\sigma_p$ and $p$-mean function $\omega_p$ are defined as 
 \[\sigma_p(v) =\min_{w\in\mathbb{R}}\lVert v-w\mathbf{1}\rVert_p, \quad \quad\omega_p(v) \in  \argmin_{w\in\mathbb{R}}\lVert v-w\mathbf{1}\rVert_p.\]
 Their close form for $p=1,2,\infty$, is summarized in table \ref{tb:kappa} . The results below summarizes policy evaluation for $L_p$-robust MDPs \cite{LpRMDP}.

Let $\hat{P}$ be any nominal transition kernel. In accordance with \cite{LpRMDP}, we define $\texttt{sa}$-rectangular $L_p$ constrained uncertainty set $\mathcal{U}^{\texttt{sa}}_p$ and $\texttt{s}$-rectangular $L_p$ constrained uncertainty set $\mathcal{U}^{\texttt{s}}_p$ as  
\begin{align*}
& \mathcal{U}^{\texttt{sa}}_p  = \{ P\Bigm| \underbrace{\sum_{s'}P_{sa}(s')=1}_{\text{ simplex condition }}, \lVert P_{sa}-(\hat{P})_{sa}\rVert_p\leq \beta_{sa}\}\\
  &  \mathcal{U}^{\texttt{s}}_p = \{P \Bigm| \sum_{s'}P_{sa}(s')=1,\lVert P_s-(\hat{P})_{s}\rVert_p\leq \beta_{s}
     \}.
\end{align*}

Note that  $\Uc^{sa}_p$, $\Uc^s_p$ are  sets around nominal kernel $\hat{P}$ component wise bounded by radius vectors $\beta$. To ensure, all the kernels in $\Uc$ are valid, 
we assume the radius vector $\beta$ is small enough.

 \begin{proposition}\label{rs:Lprvi}\cite{LpRMDP} The robust return is 
 \begin{align*}
J^\pi_{\Uc^{sa}_p} = & J^\pi -\gamma\sigma_q(v^\pi_{\Uc^{sa}_p})\sum_{s,a}d^\pi(s)\pi(a|s)\beta_{sa},\\
J^\pi_{\Uc^{s}_p} = & J^\pi -\gamma\sigma_q(v^\pi_{\Uc^{s}_p})\sum_{s}d^\pi(s)\lVert\pi_s\rVert_q \beta_{s}
\end{align*}
where $\lVert \pi_s\rVert _q$ is $q$-norm of the vector $\pi(\cdot|s)\in\Delta_{\mathcal{A}}$.
\end{proposition}
\subsection{Nature of Adversary}
 First, we define the normalized and balanced robust value function for uncertainty set $\Uc = \Uc^{\texttt{sa}}_p, \Uc^{\texttt{s}}_p$ as
\begin{align*}\label{def:balVal}
    u^\pi_{\Uc}(s)  := \frac{\text{sign}(v^\pi_{\Uc}(s)-\omega_q(v^\pi_{\Uc})) \lvert v^\pi_{\Uc}(s)-\omega_q(v^\pi_{\Uc})\rvert^{q-1}}{\sigma_q(v^\pi_\Uc)^{q-1}}
\end{align*}
which has zero mean and unit norm, that is,
\[ \langle u^\pi_{\Uc}, \mathbf{1}\rangle =0, \quad \text{and}\quad\lVert u^\pi_{\Uc}\rVert_p = 1.\]
Further, it can be re-written as a gradient of $q$-variance and its correlation with robust value function is $q$-variance.
\begin{property}\label{rs:noisValCorr}\cite{LpRMDP} For uncertainty set $\Uc = \Uc^{\texttt{sa}}_p,\Uc^{\texttt{s}}_p$, we have
\[ u^\pi_{\Uc} =\nabla_v\sigma_q(v)\Bigm|_{v=v^\pi_\Uc}, \quad \text{and}\quad \langle u^\pi_{\Uc}, v^\pi_{\Uc}\rangle = \sigma_{q}(v^\pi_\Uc).\]
\end{property}

\begin{proposition}\cite{LpPgRMDP}\label{bg:worstPm} For any policy $\pi$, taking $k = u^\pi_{\Uc}$,  the worst model is related to the nominal one through:
\begin{align*}
     P^\pi_{\Uc} &=  \hat{P} -bk^\top,
    \end{align*}
where $b(s,a)=\beta_{sa}, \beta_{s}\Bigm[\frac{\pi(a|s)}{\norm{\pi_s}_{q}}\Bigm]^{q-1}$ for $\Uc^{sa}_p$ and $\Uc^s_p$ respectively.
\end{proposition}

% Further, the worst kernel satisfies the following suse
% \[\innorm{P^\pi_{\Uc}(\cdot|s,a), v^\pi_{\Uc}} = \innorm{P(\cdot|s,a),v^\pi_{\Uc}}-\beta_{s,a}\sigma_{q}(v^\pi_{\Uc}).\]
The above result states that  for \texttt{sa}-rectangular case, the worst-case reward is independent of the employed policy. However, the worst kernel is policy dependent, and a rank one perturbation of nominal kernel which is a very surprising finding. The adversarial kernel discourages the system to move to high rewarding states, and directs towards low rewarding states.

Compared to the \texttt{sa}-rectangular case, in the \texttt{s}-rectangular case the worst reward and worst kernel have an extra dependence on the policy term $\frac{\pi(a|s)^{q-1}}{\lVert \pi_s\rVert_{q}^{q-1}}$. This is because the worst values cannot be chosen independently for each action in the \texttt{s}-rectangular case, but are instead dependent on the policy.  Similarly to the \texttt{sa}-case, the adversarial kernel is also a rank-one perturbation of the nominal kernel.

\subsection{Robust Policy Gradient}

The update rule
\begin{align}
\pi_{k+1} = Proj_{\pi\in\Pi}\Bigm[\pi_k - \eta_k \nabla_{\pi} J^{\pi_k}_{P_k}\Bigm],
\end{align}
where $J^{\pi_k}_{P_k}- J^{\pi_k}_\Uc \leq \epsilon\gamma^k$, finds global solution in $O(\epsilon^{-4})$ iterations \cite{RPG_conv}.

\begin{proposition}\cite{RPG_conv} For initial tolerance $\epsilon\leq \sqrt{T}$, and learning rate $\eta_k =\frac{\delta}{\sqrt{T}}$ for any $\delta >0$,
\[J^*_\Uc-J^{\pi_k}_\Uc \leq \epsilon,\]
for $k \geq C S^2A^3(S+A)^2\epsilon^{-4}$
where some constant $C$.
\end{proposition}

 The gradient $\nabla_\pi  J^{\pi_k}\rob$ of the robust return may not exist due to non-differentiability. However, sub-differential can be defined using the standard policy gradient theorem \cite{PolicyGradient}, as
\begin{align*}\label{eq:subGrad}
    \partial_\pi J^\pi_\Uc = \nabla_\pi J^\pi_{ P } \Bigm|_{ P  = P^\pi_{\Uc} } = \sum_{s,a}d^\pi_{\Uc}(s)Q^\pi_{\Uc}(s,a)\nabla\pi(a|s).
    % =&\sum_{(s,a)\in\X}d^\pi_{\Uc}(s)Q^\pi_{\Uc}(s,a)\nabla\pi(a|s).
\end{align*}
where $Q^\pi_{\Uc}:=Q^\pi_{P^\pi_{\Uc}}$ is robust Q-value,  $ d^\pi_{\Uc}:= d^\pi_{P^\pi_{\Uc}}$ is robust occupation measure, and  $P^\pi_\Uc $ are worst values associated with policy $\pi$, defined as
\[ P^\pi_\Uc : \in \arginf_{ P \in \Uc} J^\pi_{ P }.\]

To compute the above sub-gradient, we require the access to the worst (adversarial) parameters $P^\pi_{\Uc}$ that we show can be computed efficiently using nominal parameters $ \hat{P} $ in close form, for s-rectangular robust uncertainty sets. Then these worst parameters, can be used to compute  robust occupation measure $d^\pi_{\Uc}$ and robust Q-value $Q^\pi_{\Uc}$, which in turn can be used to compute the above robust gradient.

However, it may be possible to compute the gradient directly without computing the worst kernel, as :

% \begin{proposition}\cite{LpPgRMDP}
% \label{rs:rankOnePert} 
% Let $b,k\in \mathbb{R}^{\St}$ and $\hat{P},P_1\in(\Delta_{\St})^{\St}$ two transition matrices. If $P_1 = \hat{P} - bk^{\top}$, (a rank-one perturbation $\hat{P}$) then their occupation matrices $D_i:=(I-\gamma P_i)^{-1}, i=0,1$ are related through:   
% \[ D_1 = D_{0} - \gamma\frac{ D_0bk^{\top}D_0}{(1 + \gamma k^{\top}D_0b)}.\]
% \end{proposition}

\begin{proposition}\label{rs:sap:RPG}\cite{LpPgRMDP} The policy gradient is given by:
\begin{align*}
     \frac{\partial J^\pi_{\Uc}}{\partial\pi(a|s)}  &=\Bigm[d^\pi(s) - c^\pi(s)\Bigm]Q^\pi_{\Uc}(s,a),
     % &= \innorm{\mu,G} - \gamma\innorm{\mu, B}\frac{\innorm{K,G}}{1 +\gamma \innorm{K,B}},\\
\end{align*}
where $c^\pi$ is a correction term. Taking $k=v^\pi_{\Uc}$, it is
\[c^\pi(s) =  \gamma \frac{\innorm{d^\pi,b}}{1 +\gamma \innorm{d^\pi_{k},b}}d^\pi_{k}(s),\] 
  radius $b(s) = \sum_{a} \pi(a|s)\beta_{sa},\beta_{s}\lVert\pi_s\rVert_q$ for $\Uc^{sa}_p$ and $\Uc^s_p$ respectively. 
\end{proposition}

\paragraph{Intuition.} Here, we have the robust Q-value instead of the non-robust one. Note that the correction term $c^\pi$, resulting from switching the occupation measure of the worst kernel against the nominal, is distribution of zero, that is 
\[\innorm{c^\pi,\mathbf{1}}=0,\]
as $\innorm{d^\pi_{k},\mathbf{1}} = k^\top(I-P^\pi_0)^{-1}\mathbf{1} =  \frac{k^\top\mathbf{1}}{1-\gamma}  =0$. Observe that the $c^\pi$ is positive for those states which are on average visited more from good value states compared to bad value states, w.r.t. nominal values. Given the result, we understand that the adversary
spends more time on states which are more visited from bad value states, compared to the nominal agent.

The results holds for rectangular robust MDPs. Unfortunately, nothing is known about the nature of the adversary for non-rectangular uncertainty sets.

\section{Helper Results}

\begin{proposition}[Sherman–Morrison Formula \cite{ShermanMorrisonFormula}.]\label{app:bg:ShermanMorrisonFormula} If \( A \in\R^ {n \times n} \) invertible matrix, and \( u, v \in \R^{n}\), then the matrix \( A + uv^T \) is invertible if and only if \( 1 + v^T A^{-1} u \neq 0 \):
\[
(A + uv^T)^{-1} = A^{-1} - \frac{A^{-1} u v^T A^{-1}}{1 + v^T A^{-1} u}.
\]
\end{proposition}

\begin{proposition}\label{rs:kappa}
    \begin{align*}
   \sigma_q(v) :=\min_{w\in\mathbb{R}}\lVert v-w\mathbf{1}\rVert_q, = \min_{\norm{k}_p\leq 1,1^T k=0} k^Tv 
\end{align*}

\end{proposition}
\begin{proof}
Follows directly from Lemma J.1 of \cite{LpRMDP}.
\end{proof}

\begin{proposition} Let $\Uc^{\texttt{sa}}_2,\Uc^{\texttt{s}}_2$ be smallest \texttt{sa}-rectangular set and \texttt{s}-rectangular set containing  $\Uc_2$ then
    \[ \frac{\textbf{vol}(\Uc_2)}{\textbf{vol}(\Uc^{\texttt{sa}}_2)} = O(c_{sa}^{-SA}),\quad \text{and}\quad \frac{\textbf{vol}(\Uc_2)}{\textbf{vol}(\Uc^{\texttt{s}}_2)} = O(c_s^{-S}),\]
    where $\textbf{vol}(X)$ is volume of the set $X$ and $c_s,c_{sa} > 1$ are some constants.
\end{proposition}
\begin{proof}
    Volume of n-dimension sphere of radius $r$ is $c_{n} r^n$ where $c_n\leq  \frac{8\pi^2}{15}$ \cite{n-sphere}. And to cover n-dimension sphere of radius $r$, we need cube of radius $2r$ whose volume is $(2r)^n$. Hence the first result $\frac{\textbf{vol}(\Uc_2)}{\textbf{vol}(\Uc^{\texttt{sa}}_2)} = O(2^{-SA})$ immediately follows.

    Now, volume of set of $X = \times_{s\in \St}X_s$ where $X_s$ is an $A$-dimension sphere of radius $r$ then the volume of $X$ is $(c_A r)^S$. And the volume of an $SA$ dimensional sphere is $c_{SA}r^{SA}$, where $\lim_{n\to\infty}c_n\to 0$ \cite{n-sphere}. Hence the ratio of their volume is $O((c_A)^S)$, implying the other result.
\end{proof}

\begin{proposition}\label{app:rs:Set:sa2nr} Non-rectangular uncertainty $\Uc_p$ can be written as an infinite union of \texttt{sa}-rectangular sets $\Uc^{sa}_p$, as
\[ \Uc_{p} = \bigcup_{b\in\mathcal{B}} \Uc^{\texttt{sa}}_p(b),\]
where $\mathcal{B} =\{b\in\R_+^{\St \times \A}\mid  \norm{b}_p\leq \beta\} $. Note that  all of them share the nominal kernel $\hat{P}$.
\end{proposition}
\begin{proof}
    By definition, we have
    \begin{align}
        \Uc_p &= \{P\mid \norm{P-\hat{P}}_p\leq \beta,\sum_{s'}P(s'|s,a)=1 \}\\
        &= \{P\mid \sum_{s,a}\norm{P_{sa}-\hat{P}_{sa}}^p_p\leq \beta^p,\sum_{s'}P(s'|s,a)=1 \}\\
        &= \{P\mid \sum_{s,a}b_{sa}^p\leq \beta^p, \norm{P_{sa}-\hat{P}_{sa}}^p_p = b_{sa}^p,\sum_{s'}P(s'|s,a)=1 \}\\
        &= \{P\mid \sum_{s,a}b_{sa}^p\leq \beta^p, \norm{P_{sa}-\hat{P}_{sa}}^p_p \leq b_{sa}^p,\sum_{s'}P(s'|s,a)=1 \}\\
        &= \bigcup_{\sum_{s,a}b_{sa}^p\leq \beta^p,}\{P\mid  \norm{P_{sa}-\hat{P}_{sa}}^p_p \leq b_{sa}^p,\sum_{s'}P(s'|s,a)=1 \}\\
        &= \bigcup_{b\in\mathcal{B}}\Uc^{\texttt{sa}}_p(b).
    \end{align}
\end{proof}

 \begin{proposition}\label{app:rs:projPignHole} For any vector $\norm{x}=1$, we have 
     \[\max\{\norm{Proj_{\R^n_+}(x)},\norm{Proj_{\R^n_+}(-x)}\}\geq \frac{1
     }{\sqrt{2}},\]
    where $\R^n_+$ is positive quadrant.
 \end{proposition}
 \begin{proof}
     For any vector $\norm{x}=1$, we have
     \begin{align*}
         \norm{x_+}^2 + \norm{x_-}^2 = \norm{x}^2 =1.
     \end{align*}
And $Proj_{\R^n_+}(x) = x_+$ and $Proj_{\R^n_+}(-x) = x_-$, the rest follows.
 \end{proof}

\begin{proposition}
    For $\norm{k}_p$ and $k^T\mathbf{1} = 0$, we have
    \[ 1+\gamma k^T(I-\gamma P^\pi)^{-1}b^\pi \geq 0,\]
for all $\pi$, $\norm{b}_p\leq \beta, b\succeq 0$.
\end{proposition}
\begin{proof}
This is true from the Sherman–Morrison formula as $J^\pi_{\hat{P}-bk^T}$ is finite, hence the denominator must be strictly greater than zero.
\end{proof}

\subsection{Binary Search Approach}\label{sec:BinarySearch}

\begin{proposition} \label{app:rs:fracProg:opt:cond} For $\lambda^* =\max_{x\in C} \frac{g(x)}{h(x)}$ , $F(\lambda) := \max_{x\in C}\Bigm( g(x) - \lambda h(x)\Bigm)$ , 
we have $F(\lambda^*)=0$ and   $f(\lambda)\geq 0 \iff \lambda^* \geq \lambda$. 
\end{proposition}
\begin{proof}
\begin{itemize}
    \item If $F(\lambda)\geq 0$ then 
\begin{align*}
&\exists x \quad s.t.\quad  g(x) - \lambda h(x)\geq 0 \\
   \implies &\exists x \quad s.t.\quad  \frac{g(x)}{h(x)} \geq \lambda,\qquad \text{(as $h(x)>0 $ for all $x$)}  \\
   \implies & \max_{x\in C} \frac{g(x)}{h(x)} \geq \lambda.  
\end{align*} 
    \item If $F(\lambda)\leq 0$ then 
\begin{align*}
& g(x) - \lambda h(x)\leq 0,\qquad \forall x \in C \\
   \implies &  \frac{g(x)}{h(x)} \leq \lambda,\qquad \forall x \in C,\qquad  \text{(as $h(x)>0 $ )}  \\
   \implies & \max_{x\in C} \frac{g(x)}{h(x)} \leq \lambda
\end{align*} 
\item If $F(\lambda)= 0$ then $\lambda=\max_{x\in C} \frac{g(x)}{h(x)}$ implied from the above two items.
\end{itemize}

\end{proof}

\subsection{NP-Hardness of non-rectangualr RMDP}\label{app:sec:IP2RMDP}

\begin{tcolorbox}[colframe=black, colback=gray!10, title={\centering Reduction of Integer Program to Robust MDP} ]
\textbf{0/1 Integer Program (IP):} For $ g,c\in \Z^n,\zeta\in\Z,F\in\Z^{m\times n},$
\[  \exists x \in \{0,1\}^n \quad s.t. \quad Fx \leq g\quad \text{and}\quad c^\top x\leq \zeta?\] 
 is a NP-Hard  problem \cite{Intractability}, \cite{RVI} which reduces into following robust MDP.

\paragraph{Robust MDP:}
\begin{enumerate}
    \item State Space $\St =\{b_j,b^0_j,b^1_j\mid j=1,\cdots,n \}\cup \{c_0,\tau\}$, where $\tau$ is a terminal state.
    \item Singleton Action Space: $\A$= \{a\}.
    \item Uncertainty set: $\Uc =\{P_\xi\mid \xi\in [0,1]^n, F\xi\leq g\}$ 
    \item Discount factor $\gamma\in [0,1)$; Uniform initial state distribution $\mu$.
    \item Big reward $M \geq \frac{\gamma A n\sum_{i}c_i}{2\epsilon^2}$ where $\epsilon << 1$ helps in rounding. 
     \item Transitions and rewards are illustrated in Figure \ref{fig:IP2RMP} 
     % \begin{enumerate}
%         \item Branching:  $P_\xi(b^1_j\mid b_j) = 1-P_\xi(b^0_j\mid b_j) = \xi_j,\qquad \forall j$
%         \item Jumping : $P_\xi(b_{j+1}\mid b^0_j),P_\xi(b_{j+1}\mid b^1_j) = 1,\qquad \forall j< n$
%         \item Suicide: $P(\tau\mid b_n^0),(\tau\mid b_n^1) =1$
%         \item Terminal: $P(\tau\mid\tau) =1$.
%     \end{enumerate}
%     % \item  Reward Function: $R = ?$
\end{enumerate}
\end{tcolorbox}

\begin{figure}[ht]
         \centering
\includegraphics[width=0.5\linewidth]{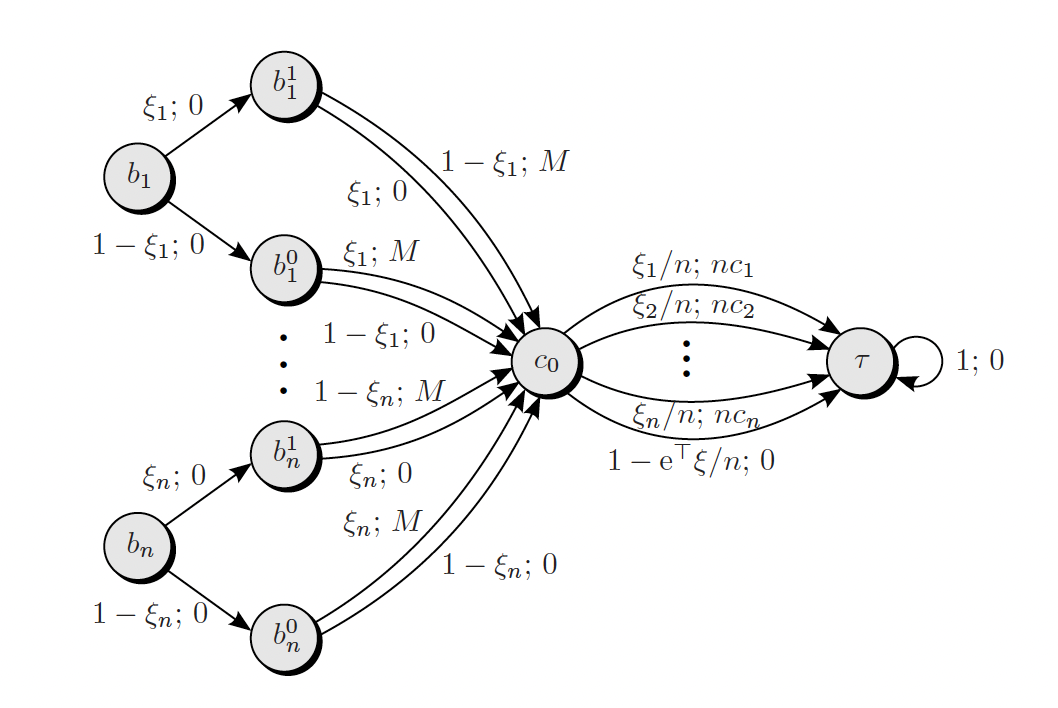}
         \caption{MDP $P_\xi$, and $R$(Figure 5 of \cite{RVI}).}
         \label{fig:IP2RMP}
     \end{figure}
Robust policy evaluation is proven to be NP-hard for general uncertainty sets defined as intersections of finite hyperplanes \cite{RVI}. Specifically, robust MDPs with uncertainty set $\Uc_{hard} := \{P_\xi| F\xi \leq g, \xi \in [0,1]^n\}$ where $P_\xi$ is a specially designed kernel with ladder structure with only action (effectively no decision) and a terminal state \cite{RVI}.

Note that $F\xi \leq g$ imposes $m$-linear constraints on $\Uc_{hard}$ while we allow only one global constraint on $\Uc_p$. Observe that $\Uc_1= \{P_\xi\mid \mathbf{1}^\top\xi\leq g, \xi \in [0,1]^n\}$ is nearest uncertainty to $\Uc_{hard}$ as both have polyhedral structure. This restrict the class of the IP programms to have a number of constraint $m=1$ and the row of $F $ to be all ones. In other words, only IP programmes that can be reduced to $\Uc_1$ are of the following form: For $,c\in \Z^n,\zeta\in\Z$, 
\[   \exists x \in \{0,1\}^n \quad s.t. \quad \mathbf{1}^Tx \leq g, \quad \text{and}\quad  c^Tx \leq \zeta?\]
\textbf{Solution:} 
\begin{itemize}
    \item Case 1) If $g <0$ then \textbf{no}.
    \item Case 2) If $g=0,\zeta\geq0$ then \textbf{yes} and  $g=0,\zeta <0$ then \textbf{yes}.
    \item If $g>0$ then compute the sum of $g$ smallest coordinate of $c$, and this sum is less/equal than $\zeta$ then answer is \textbf{yes}, otherwise \textbf{no}.
\end{itemize}

Further, for IP to be reducable to robust MDPs, the diameter of the uncertainty ($\max_{P,P'\in\Uc_{hard}}\norm{P-P'}_1=2S $) has to be large for the practical settings. Loosly speaking, robust MDPs with a $\Uc_p$ uncertainty have one global constraint and a small radius $\beta$, which corresponds to a Knapsack Problem with a small budget (IP with one constraint and a small $g$) which are much easier to solve \cite{Boyd_Convex_Optimization,Intractability}.

We can thus conclude that the hardness result of \cite{RVI} doesn't apply to our uncertainty case. 

\section{Dual Formulation}

\paragraph{s-rectangular uncertainty sets.} Now, we turn our attention to the uncertainties coupled across different actions in each state.

\begin{lemma}
     For the s-rectangular uncertainty set $\Uc^s_p$, the robust return can be written as
    \[ J^\pi_{\Uc^\texttt{s}_p} =   J^\pi- \gamma\min_{\norm{b_s}_p\leq \beta_s,\norm{k}_p \leq 1, \innorm{1,k}=0}\frac{ \innorm{d^\pi, b^\pi}
    \innorm{k,v^\pi}}{1 + \gamma k^{\top}D^\pi b^\pi},  \]
    where $b\in\R^{\St\times\A}$,  $b_s = b_{s\cdot}$, and $b^\pi(s) = \sum_{a}\pi(a|s)b_{sa}$.
\end{lemma}

\begin{proof}
The proof follows similarly to the \texttt{sa}-rectangular case, and can be found in the appendix. The key additional step is to decompose the \texttt{s}-rectangular uncertainty set $\Uc^{s}_p$ into a union of \texttt{sa}-rectangular uncertainty sets $\Uc^{sa}_p$.  
\end{proof}

The above result formulates the robust return in terms of nominal values only for the first time. This implies the robust objective can be rewritten in the dual form as :
\[J^*_{\Uc^{\texttt{s}}_p} = \max_{D\in \mathcal{D} }\min_{k\in\mathcal{K},b\in\mathcal{B}}\Bigm[\mu^TD R^\pi -\gamma \mu^TDb^\pi \frac{k^TD R^\pi}{1 +\gamma  k^TD b^\pi}\Bigm]\]
where $\mathcal{D} = \{(I-\gamma P^\pi_0)^{-1}\mid \pi \in \Pi\}$,  $\mathcal{K} = \{k\in \R^{\St} \mid \norm{k}_p=1,1^Tk = 0 \}$, and $\mathcal{B}=\{b\in R^{\St\times\A}\mid \norm{b_s}_p\leq \beta_s\}$.

Comparing the penalty term from the previous results in \cite{LpRMDP,LpPgRMDP}, the dual formulation can be written as 
\[J^*_{\Uc^{\texttt{s}}_p} = \max_{D\in \mathcal{D} }\min_{k\in\mathcal{K}}\Bigm[\mu^TD R^\pi -\gamma \mu^TD\beta^\pi \frac{k^TD R^\pi}{1 +\gamma  k^TD \beta^\pi}\Bigm]\]
where $\beta^\pi_s = \norm{\pi_s}_q\beta_s$.

Surprisingly, the optimization here looks as if it is optimized for the same value of $\beta^\pi_s =\max_{\sum_{a}\beta^p_{sa}\leq \beta^p_s}\sum_{a}\pi(a|s)\beta_{sa}=\beta_{s}\norm{\pi_s}_q\ $  for all values of feasible $k$. This suggest that  the adversary payoff is maximized by maximizing the expected uncertainty in the trajectories. 

\begin{lemma} For the sa-rectangular uncertainty set $\Uc = \Uc^{sa}_p(\beta)$ with radius vector $\beta\in \R^{\St\times\A}$, the robust return can be written as the following optimization problem,
    \[ J^\pi_\Uc =   J^\pi -\gamma \max_{\norm{k}_p=1,1^Tk = 0}\frac{\mu^TD^\pi\beta^\pi k^TD^\pi R^\pi}{1 +\gamma  k^TD^\pi \beta^\pi},\]
where $\beta^\pi_s = \sum_{a}\pi(a|s)\beta_{sa}$.
\end{lemma}
\begin{proof}
    
   From \cite{LpPgRMDP}, we know that the worst kernel $P^\pi_{\Uc^{sa}_p(\beta)}$ for the uncertainty set $\Uc^{sa}_p(\beta)$ is a rank one-perturbation of $P$. In other words,
   \[P^\pi_{\Uc^{sa}_p(\beta)} = P + \beta k^T\] for some $k\in\R^{\St}$ satisfying $\norm{k}_p=1$ and $1^Tk = 0$. This implies that it is enough to look for rank-one perturbations of the nominal kernel $\hat{P}$ in order to find the robust return. That is, 
   \begin{align*}
    J^\pi_{\Uc^{sa}_p(\beta)} &= \min_{P\in \Uc^{sa}_p(\beta)} J^\pi_P\\&= \min_{P = \hat{P} + \beta k^T, \norm{k}_p=1,1^Tk = 0} J^\pi_P,\qquad\text{(looking only at rank one perturbations)}\\
    &= \min_{P = \hat{P} + \beta k^T, \norm{k}_p=1,1^Tk = 0}\mu^TD^\pi_P R^\pi
   \\&= \min_{P = \hat{P} + \beta k^T, \norm{k}_p=1,1^Tk = 0}\mu^T(I-\gamma P^\pi)^{-1} R^\pi\\&= \min_{\norm{k}_p=1,1^Tk = 0}\mu^T\Bigm(I-\gamma (P^\pi+\beta^\pi k^T)\Bigm)^{-1} R^\pi\\&=  J^\pi -\gamma \max_{\norm{k}_p=1,1^Tk = 0}\frac{\mu^TD^\pi\beta^\pi k^TD^\pi R^\pi}{1 +\gamma  k^TD^\pi \beta^\pi}.
\end{align*}
\end{proof}

\begin{lemma}
     For $\Uc = \Uc^s_p$, the robust return can be written as the following optimization problem,
    \[ J^\pi_\Uc =   J^\pi- \gamma\min_{\norm{\beta}_p\leq \epsilon,\norm{k}_p \leq 1, \innorm{1,k}=0}\frac{ \innorm{d^\pi, \beta^\pi}
    \innorm{k,v^\pi}}{1 + \gamma k^{\top}D^\pi \beta^\pi},  \]
    where $D^{\pi} = (I-\gamma P^{\pi})^{-1}$, $d^\pi = \mu^{T}D^\pi$ and $v^\pi = D^\pi R^\pi $.
\end{lemma}

\begin{proof}

\begin{align*}
    & J^\pi_{\Uc^s_p(\beta)} = \min_{\norm{P_s-(P)_s}^p_p = \beta_s^p, 1^TP_{sa}=1} J^\pi_P \\&=\min_{\sum_{a}\beta^p_{sa}\leq \beta^p_s} \min_{\norm{P_{sa}-(P)_{sa}}_p = \beta_{sa}, 1^TP_{sa}=1} J^\pi_P \\&=\min_{\sum_{a}\beta^p_{sa}\leq \beta^p_s}  J^\pi_{\Uc^{sa}_p(\beta)} \\&=\min_{\sum_{a}\beta^p_{sa}\leq \beta^p_s}\Bigm[ J^\pi -\gamma \max_{\norm{k}_p=1,1^Tk = 0}\frac{\mu^TD^\pi\beta^\pi k^TD^\pi R^\pi}{1 +\gamma  k^TD^\pi \beta^\pi}\Bigm]\\\\&= J^\pi -\gamma \max_{\sum_{a}\beta^p_{sa}\leq \beta^p_s,\norm{k}_p=1,1^Tk = 0}\frac{\mu^TD^\pi\beta^\pi k^TD^\pi R^\pi}{1 +\gamma  k^TD^\pi \beta^\pi}.
\end{align*}
\end{proof}

\begin{lemma}
     For $\Uc = \Uc_p$, the robust return can be written as the following optimization problem
    \[ J^\pi_\Uc =   J^\pi- \gamma\min_{\norm{\beta}_p\leq \epsilon,\norm{k}_p \leq 1, \innorm{1,k}=0}\frac{ \innorm{d^\pi, \beta^\pi}
    \innorm{k,v^\pi_R}}{1 + \gamma \innorm{k,v^\pi_\beta}},  \]
    where $D^\pi = (I-\gamma P^\pi)^{-1}$, $d^\pi = \mu^TD^\pi$ and $v^\pi = D^\pi R^\pi$.
\end{lemma}

\begin{proof}
Now,
\begin{align*}
    & J^\pi_{\Uc_p(\epsilon)} = \min_{\norm{P-P}^p_p = \epsilon^p, 1^TP_{sa}=1} J^\pi_P \\&=\min_{\norm{\beta}^p_p\leq \epsilon^p} \min_{\norm{P_{sa}-(P)_{sa}}_p = \beta_{sa}, 1^TP_{sa}=1} J^\pi_P \\&=\min_{\norm{\beta}^p_p\leq \epsilon^p}  J^\pi_{\Uc^{sa}_p(\beta)} \\&=\min_{\norm{\beta}_p\leq \epsilon}\Bigm[ J^\pi -\gamma \max_{\norm{k}_p=1,1^Tk = 0}\frac{\mu^TD^\pi\beta^\pi k^TD^\pi R^\pi}{1 +\gamma  k^TD^\pi \beta^\pi}\Bigm]\\\\&= J^\pi -\gamma \max_{\norm{\beta}_p\leq \epsilon,\norm{k}_p=1,1^Tk = 0}\frac{\mu^TD^\pi\beta^\pi k^TD^\pi R^\pi}{1 +\gamma  k^TD^\pi \beta^\pi}.
\end{align*}
\end{proof}

The above result formulates the robust return in terms of nominal values only, for the first time.
Comparing with the existing result, we get a very interesting relation: 
\begin{align}\label{eq:kappaDual}
   \sigma_q(v^\pi_\Uc)=  \max_{\norm{k}_p=1,1^Tk = 0}\frac{k^Tv^\pi_R}{1 +\gamma  k^Tv^\pi_\beta},
\end{align}
where $v^\pi_x = (I-\gamma P^\pi)^{-1}x^\pi$.

The LHS is a robust quantity (variance of the robust return) which is express in the terms of purely nominal quantities. This is the simplest of all such relations. We believe that the above relation can help in theoretical derivations and experiment design but not exactly sure how yet.

\paragraph{Intuition on the adversary.} We know that the $\sigma(v^\pi_\Uc)$ is the penalty for robustness, that is 
\[ J^\pi_\Uc =  J^\pi -\gamma \innorm{d^\pi,\beta^\pi}\sigma_q(v^\pi_\Uc).\]

Knowing $\sigma(v^\pi_\Uc)$ how it arises, sheds the light on the working of the adversary in \eqref{eq:kappaDual}. Further more, recall that if $P = P -\beta k^T$ then 
\[ J^\pi_P =  J^\pi - \innorm{d^\pi,\beta^\pi}\frac{k^Tv^\pi_R}{1 +\gamma  k^Tv^\pi_\beta}.\]
Basically, $k$ is the direction the adversary discourages the perturbation of the kernel. And the optimal direction $k$ that the adversary chooses is the one that optimizes \eqref{eq:kappaDual}.

\paragraph{s-rectangular uncertainty sets.} Now, we move our attention to the coupled uncertainty case.

\begin{lemma}
     For $\Uc = \Uc^s_p$, the robust return can be written as the following optimization problem
    \[ J^\pi_\Uc =   J^\pi- \gamma\min_{\norm{\beta}_p\leq \epsilon,\norm{k}_p \leq 1, \innorm{1,k}=0}\frac{ \innorm{d^\pi, \beta^\pi}
    \innorm{k,v^\pi}}{1 + \gamma k^{\top}D^\pi \beta^\pi},  \]
    where $D^\pi = (I-\gamma P^\pi)^{-1}$, $d^\pi = \mu^TD^\pi$ and $v^\pi = D^\pi R^\pi$.
\end{lemma}

\begin{proof}
The proof follows similarly to the \texttt{sa}-rectangular case, and can be found in the appendix. The key additional step is to decompose the \texttt{s}-rectangular uncertainty set $\Uc^{s}_p$ into  as a union of \texttt{sa}-rectangular uncertainty sets $\Uc^{sa}_p$.  
\end{proof}

Comparing the penalty term from the previous results in \cite{LpRMDP,LpPgRMDP}, we get
\begin{align*}\label{eq:kappa:srect}
&\sum_{s}d^\pi(s)\norm{\pi_s}_q\sigma_q(v^\pi_\Uc)=\\&\max_{\sum_{a}\beta^p_{sa}\leq \beta^p_s,\norm{k}_p=1,1^Tk = 0}\frac{(d^\pi\beta^\pi)( k^Tv^\pi)}{1 +\gamma  k^TD^\pi \beta^\pi}.
\end{align*}
Again, the above relation looks very interesting as it relates the robust term on LHS with non-robust terms on RHS.

Surprisingly, the optimization here looks as if it is optimized for the same value of $\beta^\pi_s =\max_{\sum_{a}\beta^p_{sa}\leq \beta^p_s}\sum_{a}\pi(a|s)\beta_{sa}=\beta_{s}\norm{\pi_s}_q\ $  for all values of feasible $k$. This suggests that  the adversary's payoff is maximized by maximizing the expected uncertainty in the trajectories. 

\paragraph{Looking at GSTD from two angles  } From the binary search method at section \ref{sec:BinarySearch},  we know that $\sigma_q(v^\pi_\Uc)$ is the solution to the following: 
\begin{align}
    &\max_{\norm{k}_p\leq 1, 1^Tk =0}\Bigm[ k^Tv^\pi_R - x (1+\gamma k^Tv^\pi_\beta)\Bigm] =0\\
    &\max_{\norm{k}_p\leq 1, 1^Tk =0} k^T\bigm(v^\pi_R - \gamma x v^\pi_\beta\bigm) =x\\ &\max_{\norm{k}_p\leq 1, 1^Tk =0} k^TD^\pi\bigm(R^\pi - \gamma x \beta^\pi\bigm) =x
\end{align}

\[ v^\pi_R = \gamma\frac{k^Tv^\pi_R}{1+\gamma k^Tv^\pi_\beta}v^\pi_\beta = \gamma \sigma(v^\pi_\Uc)v^\pi_\beta.\]
Its maybe possible to make this process online, simultaneously updating $x\to v^\pi_\Uc$, $v$ and $k$.

%%%%%%%%%%%%%%%%%%%%%%%%%%%%%%%%%%%%%%%%%%%%%%%%%%%%%%%%%%%%%%%%%%%%%%%%%%%%%%%
%%%%%%%%%%%%%%%%%%%%%%%%%%%%%%%%%%%%%%%%%%%%%%%%%%%%%%%%%%%%%%%%%%%%%%%%%%%%%%%
\section{Robust Policy Evaluation}

\begin{lemma} \label{app:rs:RPE:binSearch}
The robust return can be expressed as \[J^\pi_{\Uc_p} = J^\pi - \lambda^*,\] where the penalty \(\lambda^*\) is a fixed point of \(F(\lambda)\). Furthermore, $\lambda^*$ can be found via binary search as $ 
F(\lambda) > \lambda $ if and only if $ \lambda > \lambda^*,$
where \(F(\lambda) = \max_{b \in \mathcal{B}} \|E^\pi b\|_q\), \(E^\pi = \gamma \Big(I - \frac{\mathbf{1}\mathbf{1}^\top}{S}\Big)\Big[D^\pi R^\pi \mu^\top D^\pi - \lambda D^\pi\Big]H^\pi\), and \(H^\pi R := R^\pi\).
\end{lemma}
\begin{proof}
We want to evaluate the following
\[\lambda^* := \max_{b\in\mathcal{B},k\in \mathcal{K} }\gamma  \frac{k^TD^\pi R^\pi \mu^TD^\pi b^\pi}{1 +\gamma  k^TD^\pi b^\pi}.\]
This is of the form $\max_{x}\frac{f(x)}{g(x)}$ . Then according to Proposition \ref{app:rs:fracProg:opt:cond}, we have $f(\lambda^*) =0$ and $f(\lambda)>0$ if and only if $\lambda^* > \lambda$, where

\begin{align*}
    & f(\lambda) := \max_{b\in\mathcal{B},k\in \mathcal{K} } \left[ \gamma k^T A^\pi b^\pi -\lambda(1 +\gamma  k^T D^\pi b^\pi) \right] \\
    &=  \max_{b\in\mathcal{B},k\in \mathcal{K} } k^\top C^\pi b -\lambda,\\
    &=  \max_{b\in\mathcal{B},\norm{k}_p\leq 1 } k^\top \left(I - \frac{\mathbf{1}\mathbf{1}^T}{S} \right) C^\pi b -\lambda, \quad\text{(from Proposition \ref{app:rs:projection2matrix})}\\
    &=  \max_{b\in\mathcal{B}} \norm{\left(I - \frac{\mathbf{1}\mathbf{1}^T}{S} \right) C^\pi b}_{q} -\lambda, \quad\text{(Holder's inequality)}
\end{align*}

where $A^\pi  =D^\pi R^\pi \mu^TD^\pi, 
C^\pi := \gamma \Bigm(A^\pi -\lambda D^\pi\Bigm) H^\pi $.

\end{proof}
% \begin{algorithm}
% \caption{Binary Search: Robust Policy Evaluation Algorithm for  $L_p$ Robust MDP}\label{alg:BS}
% \begin{algorithmic} [1]\STATE \textbf{Input}: $\lambda_{max} = -\lambda_{min} =\frac{1}{1-\gamma}$ \\
% \WHILE{ not converged}
%  \STATE If $c_n>0$ then set $\lambda_{min}=\lambda_n$, if $c< 0$ then $\lambda_{max}=\lambda_n$, and if $c=0$ then terminate.

% \STATE Set \[\lambda_n = \frac{\lambda_{min}+\lambda_{max}}{2}\]
% \STATE Compute the indicator:
% \begin{align*} 
% c_{n+1}=\max_{b\in B,k\in\zeta}  k^{\top}\bigm(v^\pi {d^\pi}^{\top} -\lambda_n\gamma D^\pi\bigm) b^\pi,
% \end{align*}
% where $\zeta =\{k\in\R^{\St}\mid \norm{k}_p \leq 1, \innorm{1,k}=0\}$ and $B$ depends on $\Uc = \Uc^{sa}_p,\Uc^s_p,\Uc_p$.

% \STATE Update the estimate of the robust return
% \[ J_{n+1} =  J^\pi -\gamma \lambda_{n+1}.\]
% \ENDWHILE
% \STATE \textbf{Output}:  Robust return: $ J^\pi -\gamma \lambda_{n+1} \to  J^\pi_\Uc$ \\
% \end{algorithmic}
% \end{algorithm}

\section{Robust Policy Improvement}
In the previous section, , we found that the worst kernel is a rank-one perturbation of the nominal kernel. Exploiting this, we developed a method to evaluate the robust policy efficiently. This methods also computes the perturbation ($\beta k^T$), and consequently the worst kernel. 

We can use it directly to compute the gradient w.r.t. the policy for this computed worst kernel. Then, we can apply policy improvement by gradient ascent as in \cite{RPG_conv}:
\begin{align}
    \pi_{n+1} = proj\Bigm[\pi_n +\eta_k\nabla_\pi  J^{\pi}_{P_n}\mid_{\pi=\pi_n}\Bigm],
\end{align}
where $P_n$ is the estimate of the worst kernel for $\pi_k$. This has global convergence guarantees with iteration complexity of $O(\epsilon^{-4})$ \cite{RPG_conv}.

Alternatively, we can derive the policy gradient for the approximated perturbation, as done in the result below.

\begin{lemma}[Approximate Policy Gradient Theorem] Given the transition kernel $P = \hat{P}-\beta k^\top$,
the return is given as 
\begin{align*}
     J^\pi_P &:= J^\pi_0 - \gamma\frac{  J^\pi_\beta\innorm{ k,v_R^\pi}}{1 + \gamma \innorm{k,v^\pi_\beta}}, 
\end{align*}
and the gradient is given as 
\begin{align*}&
\nabla _\pi  J^\pi_P = d^\pi\circ Q^\pi_{R} - \gamma\frac{  k^{\top}v^\pi_R }{1 + \gamma k^{\top}v^\pi_\beta}d^\pi\circ Q^\pi_{\beta}\\&- \gamma\frac{  J^\pi_\beta( k^\top D^\pi)}{1 + \gamma k^{\top}v^\pi_\beta}\circ Q^\pi_{R}+\gamma^2\frac{  J^\pi_\beta (k^{\top}v^\pi)( k^\top D^\pi)}{(1 + \gamma k^{\top}v^\pi_\beta)^2}\circ Q^\pi_{\beta}. 
\end{align*}

\end{lemma}
\begin{proof}
The expression for the return directly follows from the inverse matrix theorem as proved in \cite{LpPgRMDP}. Now, we move our attention towards evaluating the gradient, using the policy gradient theorem \cite{Sutton1998} in the format used in \cite{PG_conv} (see appendix).
\begin{align*}&
\nabla _\pi  J^\pi_P = d^\pi\circ Q^\pi_{R} - \gamma\frac{  k^{\top}D^\pi R^\pi }{1 + \gamma k^{\top}D^\pi \beta^\pi}d^\pi_\mu\circ Q^\pi_{\beta}\\&- \gamma\frac{ \mu^TD \beta^\pi }{1 + \gamma k^{\top}D^\pi \beta^\pi}d^\pi_k\circ Q^\pi_{R}+\gamma^2\frac{ \mu^TD \beta^\pi k^{\top}D^\pi R^\pi}{(1 + \gamma k^{\top}D^\pi \beta^\pi)^2}d^\pi_k\circ Q^\pi_{\beta},\\
&= d^\pi\circ Q^\pi_{R} - \gamma\frac{  k^{\top}v^\pi_R }{1 + \gamma k^{\top}D^\pi \beta^\pi}d^\pi\circ Q^\pi_{\beta}\\&- \gamma\frac{  J^\pi_\beta( k^\top D^\pi)}{1 + \gamma k^{\top}D^\pi \beta^\pi}\circ Q^\pi_{R}+\gamma^2\frac{  J^\pi_\beta (k^{\top}v^\pi)( k^\top D^\pi)}{(1 + \gamma k^{\top}D^\pi \beta^\pi)^2}\circ Q^\pi_{\beta}. 
\end{align*}
\end{proof}

The main advantage of the above policy gradient is that  constituents terms like $ J^\pi_\beta, v^\pi_\beta, Q^\pi_\beta $, in addition to  nominal terms $ J^\pi_R,v^\pi_R,Q^\pi_R $ can be computed easily with bootstrapping exploiting Bellman operators.

\section{Evaluation of \texorpdfstring{$\max_{x,y} x A y$}{max xy A}}
\label{app:sec:NormEvaluation}
Algorithm \ref{alg:RPE} requires an oracle access to 
\[\max_{\norm{b}_p\leq \beta, \norm{k}_p\leq 1, 1^Tk=0} k^TAb,\]
where $k\in\R^{\St}$, $b\in\R^{\St\A}$ and $p\geq 1$. The above is a bilinear problem, which is NP-Hard, but we have a very useful structure on domain set ($L_p$ bounded set).

\begin{proposition}\label{app:rs:ortho} [Orthogonality Equivalence]Let $\mathcal{K} = \{k\mid \norm{k}_p\leq 1, 1^\top k=0\}$, and $\mathcal{W} = \{k^T( I-\frac{11^T}{S}) \mid \norm{k}_p\leq 1\}$ . Then we have,
    \[\mathcal{K} = \mathcal{W}.\]
\end{proposition}

\begin{proof} 
 Now let  $k\in \mathcal{K}$, then $k^T( I-\frac{11^T}{S}) = k^\top \in \mathcal{W}$. Now the other direction, let $k\in\mathcal{W}$, then $\innorm{k^T( I-\frac{11^T}{S}),1} = 0$ by construction and $\norm{k^T( I-\frac{11^T}{S})}_p \leq \norm{k}_p\leq 1$, this implies $k^T( I-\frac{11^T}{S})\in\mathcal{K}$.
\end{proof}
The above result implies that
\begin{align*}
    \max_{\norm{b}_p\leq \beta, \norm{k}_p\leq 1, 1^Tk=0} k^TAb =& 
    \max_{\norm{b}_p\leq \beta, k\in\mathcal{K}} k^TAb\\ =& 
    \max_{\norm{b}_p\leq \beta, k\in\mathcal{W}} k^TAb,\qquad\text{(as $\mathcal{K} =\mathcal{W}$ from above Proposition \ref{app:rs:ortho})}\\ =&\max_{\norm{b}_p\leq \beta, \norm{k}_p=1}k^\top ( I-\frac{11^T}{S}) Ab,\qquad \text{(def. of $\mathcal{W}$)}. 
\end{align*}

Further, we have equivalence of optimizers
\[ \argmax_{ \norm{k}_p\leq 1, 1^Tk=0,\norm{b}_p\leq \beta} k^TAb = \Bigm\{ (b^*,( I-\frac{11^T}{S})k^*) \mid (b^*,k^*)\in \argmax_{ \norm{k}_p=1,\norm{b}_p\leq \beta}k^\top ( I-\frac{11^T}{S}) Ab\Bigm\}.\]

\begin{proposition}\label{app:rs:projection2matrix}
  The solving of \[\max_{ \norm{k}_p\leq 1, 1^Tk=0,\norm{b}_p\leq \beta} k^TAb,\quad\text{ is equivalent to }\qquad  \max_{ \norm{k}_p=1,\norm{b}_p\leq \beta}k^\top ( I-\frac{11^T}{S}) Ab.\]
\end{proposition}
\begin{proof}
    Directly follow from the proposition above.
\end{proof}

\subsection{Eigenvalue Approach (Spectral Methods)}
This section focus on deriving a spectral method for solving the optimization problem:
\[
\max_{\|x\|_2 \leq 1, \, x \geq 0} \|Ax\|_2 ,
\]
where \(A \in \mathbb{R}^{n \times n}\).
Compute $A^\top A$. We perform eigenvalue decomposition of $A^\top A$: 
\[
A^\top A = V \Lambda V^\top,
\]
where $\Lambda = \text{diag}(\lambda_1, \lambda_2, \dots, \lambda_n)$ (eigenvalues) and $V = [v_1, v_2, \dots, v_n]$ (eigenvectors). Further, WLOG \[\lambda_1\geq \lambda ,\cdots, \qquad , \text{and}\quad \norm{v+_i}\geq  \norm{v-_i}\qquad \forall i,\qquad u_i := \frac{v^+_i}{\norm{v^+_i}}\]
where $v_i^+ = \max(v_i, 0), \quad v_i^- = -\min(v_i, 0)$ denotes positive and negative parts respectively.

\begin{itemize}
    \item Zero Order Solution:
    \[f_0 = \norm{Au_1}. \]
    \item First order solution:
    \[f_1 = \max_{i}\norm{Au_i}.\]
    \item  Second order solution:
    \[f_2 = \max_{i,j}\max_{t\in[0,1]}\norm{A\frac{(tv_i+(1-t)v_j)^+}{\norm{(tv_i+(1-t)v_j)^+}}}.\]
    \item Third order solution:
    \[f_3 = \max_{i,j,k}\max_{r,s,t,\in[0,1],r+s+t=1}\norm{A\frac{(rv_i+sv_j+tv_k)^+}{\norm{(rv_i+sv_j+tv_k)^+}}}.\]
\end{itemize}

Upper bounds on $\max_{\norm{x}_2\leq 1, x \succeq 0}\norm{Ax}_2$:
\begin{itemize}
    \item Zero order upper bound: $\lambda_1$
    \item First order upper bound: 
    \[
    \sqrt{\sum_{i}\lambda_ic_i },
    \]
    where
    \[
    c_i = \begin{cases}
    \innorm{v_i,u_i}^2, & \quad\text{if} \quad \sum_{j=1}^{i}\innorm{v_j,u_j}^2\leq 1, \\[0.8em]
    1-\sum_{j=1}^{i-1}\innorm{v_j,u_j}^2, & \quad \text{if} \quad 
    \begin{aligned}
        &\sum_{j=1}^{i}\innorm{v_j,u_j}^2\geq 1, \\
        &\sum_{j=1}^{i-1}\innorm{v_j,u_j}^2\leq 1
    \end{aligned} \\[0.8em]
    0, & \quad\text{otherwise}.
    \end{cases}
    \]
\end{itemize}

\begin{lemma}[Zero Order Approximation] The highest projected eigenvector  $ u=\frac{v^+_1}{\norm{v^+_1}}$ is atleast a half-good solution, \ie 
    \[\norm{Au}^2_2 \geq\frac{\lambda_1}{2}\geq \frac{1}{2}\max_{\|x\|_2 \leq 1, \, x \geq 0} \|Ax\|^2_2.  \]
Further, if $A$ is rank-one then it is exact, \ie
\[ \norm{Au}_2 = \max_{\|x\|_2 \leq 1, \, x \geq 0} \|Ax\|_2.\]
\end{lemma}

\begin{proof}
     We have $\norm{v_1^+} \geq \frac{1}{\sqrt{2}}$ from Proposition \ref{app:rs:projPignHole}. Let $u =\frac{(v_1)_+}{\norm{(v_1)_+}} = \sum_{i}\sigma_i v_i $, where $\sigma_i=\innorm{u,v_i}$,  we have
    \begin{align*}
        u^TA^TAu &= (\sum_{i}\sigma_i v_i)(\sum_{i}\lambda_iv_iv_i^T)(\sum_{i}\sigma_i v_i) \\
        &= \sum_{i}\lambda_i\sigma^2_i,\qquad\text{(as $v_i$ are orthogonal)} \\
        &= \lambda_1\sigma_1^2 +\sum_{i\neq 1}\lambda_i\sigma^2_i,\qquad\text{} \\
        &\geq \lambda_1\sigma_1^2 +\sum_{i\neq 1}\lambda_n\sigma^2_i,\qquad\text{(as $\lambda_2\geq \lambda_3,\cdots$)} \\
        &= \lambda_1\sigma_1^2 +\lambda_n(1-\sigma_1^2),\qquad\text{(as $\sum_{i}\sigma_i^2=1$)} \\
        &\geq \frac{1}{2}(\lambda_1+\lambda_n),\qquad\text{(as $\sigma_1\geq\frac{1}{\sqrt{2}}$)}.
    \end{align*}
    
Rest follows.
    
\end{proof}

\begin{proposition}[First Order is Better than the First]
 \[\norm{Au_j}^2_2\geq \max_{i}\lambda_i\sigma_i^2\geq \frac{\lambda_1}{2}\]
 where $j \in \argmax_{i} \lambda_i\innorm{v_i,u_i}$ and $\sigma_i = \innorm{v_i,u_i}\geq\frac{1}{\sqrt{2}}$.
\end{proposition}
\begin{proof}
Let $u_j =\frac{(v_j)_+}{\norm{(v_j)_+}} = \sum_{i}\sigma^j_i v_i $, where $\sigma^j_i=\innorm{u_j,v_i}$,  we have
    \begin{align*}
        u^T_jA^TAu_j &= (\sum_{i}\sigma_i^j v_i)(\sum_{i}\lambda_iv_iv_i^T)(\sum_{i}\sigma_i^j v_i) \\
        &= \sum_{i}\lambda_i(\sigma^j_i)^2,\qquad\text{(as $v_i$ are orthogonal)} ,\\&\geq\lambda_j(\sigma^j_j)^2 , \\        &=\max_{i} \lambda_i(\sigma_i)^2,\qquad\text{(by definition of $j$)}. 
    \end{align*}
    
Rest follows.
    
\end{proof}

\begin{proposition}Second order solution $f_2 = \max_{i,j}\max_{t\in[0,1]}\norm{A\frac{(tv_i+(1-t)v_j)^+}{\norm{(tv_i+(1-t)v_j)^+}}}$ is exactly equal to $\max_{\norm{x}_2\leq 1, x\succeq 0}\norm{Ax}_2$ when $A$ is rank two.
\end{proposition}

This approach is computationally efficient but may not always yield the exact solution, especially when multiple eigenvectors significantly contribute to the optimal \(x\).

The intuition behind this approach is that the matrix \(A^\top A\) can be decomposed into its eigenvalues and eigenvectors, representing the principal directions of the transformation applied by \(A\). The eigenvector corresponding to the largest eigenvalue provides the direction of maximum scaling for \(A\). However, since the solution is constrained to the nonnegative orthant (\(x \geq 0\)), we adjust the eigenvectors by only considering their positive parts. The method identifies an approximate solution \(u_j\) by selecting and normalizing the positive part of the eigenvector that contributes the most to the objective function.

\begin{algorithm}[H]
\caption{Second Order Spectral Approximation for $\max_{\|x\|_2 \leq 1, x \geq 0} \|Ax\|_2$}
\label{app:alg:Ax:EH}
\begin{algorithmic}[1]    \STATE Normalize the positive part:
    \[
    u_i = \frac{v_i^+}{\|v_i^+\|_2}.
    \]
\STATE Compute scores for all eigenvectors:
\[
\text{Score}_i = \lambda_i \langle v_i, u_i \rangle.
\]
\STATE Select $j = \arg\max_i \text{Score}_i$.
\STATE \textbf{Output:} Approximate solution $u_j = v_j^+ / \|v_j^+\|_2$ and approximate maximum value $\|A u_j\|_2$.
\end{algorithmic}
\end{algorithm}

\subsection*{Notes}
\begin{itemize}
    \item This approach is effective when the largest eigenvalue \(s_1\) dominates the others. It approximates the solution by leveraging the spectral properties of \(A^\top A\).
    \item The result might not be exact if multiple eigenvalues contribute significantly, as the approach considers only the contribution of individual eigenvectors.
\end{itemize}

% The eigenvalue heuristic leverages the spectral decomposition of \(A^\top A\) to approximate the solution:
% \begin{enumerate}
%     \item Compute \(A^\top A\) and perform eigenvalue decomposition to obtain eigenvalues \((s_i)\) and eigenvectors \((v_i)\).
%     \item Adjust each eigenvector \(v_i\) to ensure \(\|v_i^+\|_2 \geq \|v_i^-\|_2\), where \(v_i^+ = \max(v_i, 0)\) and \(v_i^- = \min(v_i, 0)\).
%     \item Normalize the positive part \(v_i^+\) to obtain \(u_i = v_i^+ / \|v_i^+\|_2\).
%     \item Compute scores for each eigenvector: \(\text{Score}_i = s_i \langle v_i, u_i \rangle\).
%     \item Select the vector \(u_j\) corresponding to the largest score and compute \(\|Au_j\|_2\).
% \end{enumerate}
% This method is computationally efficient and provides a good approximation when the largest eigenvalue dominates. However, it may fail if the contributions from multiple eigenvectors are significant.
\subsubsection{Reformulation}
\begin{proposition}
     \[\max_{\norm{x}_2\leq 1, x\succeq  0 }xA^TAx =\max_{V\sigma\succeq 0, \norm{\sigma}_2\leq 1}\sum_{i}\lambda_i\sigma^2_i,\qquad \text{(where $VV^\top = I, \lambda_i\geq 0$)}\]
\end{proposition}
 \begin{proof}
     $x = \sum_{i}\sigma_i v_i $, where $\sigma_i=\innorm{x,v_i}$,  we have
    \begin{align*}
        x^TA^TAx &= (\sum_{i}\sigma_i v_i)(\sum_{i}\lambda_iv_iv_i^T)(\sum_{i}\sigma_i v_i) \\
        &= \sum_{i}\lambda_i(\sigma_i)^2,\qquad\text{(as $v_i$ are orthogonal)}. 
    \end{align*}
Further, $x = V\sigma$. This map is bijection. Rest follows.
 \end{proof}

Since, $\lambda_i\geq 0$, the objective $\innorm{\lambda,\sigma}$ is convex in $\sigma$, further the domain $\{\sigma \mid V\sigma\succeq 0, \norm{\sigma}_2\leq 1\}$ is intersection of a polytope and a sphere, hence convex. This makes the following
\[\max_{V\sigma\succeq 0, \norm{\sigma}_2\leq 1}\sum_{i}\lambda_i\sigma^2_i\]
 as convex objective with convex domain.

\begin{proposition}
     \[\max_{\norm{x}_2\leq 1, x\succeq  0 }xA^TAx =\max_{b\in B}\sum_{i}\lambda_ib_i,\]
    where $B =\{b\mid b_i =\innorm{v_i,x}^2,x\succeq 0\}$.
\end{proposition}
 \begin{proof}
     $x = \sum_{i}\sigma_i v_i $, where $\sigma_i=\innorm{x,v_i}$,  we have
    \begin{align*}
        x^TA^TAx &= (\sum_{i}\sigma_i v_i)(\sum_{i}\lambda_iv_iv_i^T)(\sum_{i}\sigma_i v_i) \\
        &= \sum_{i}\lambda_i(\sigma_i)^2,\qquad\text{(as $v_i$ are orthogonal)}. 
    \end{align*}
    
Rest follows.
 \end{proof}

\begin{proposition}
    The set 
    \[B =\{b\mid b_i =\innorm{v_i,x}^2,x\succeq 0, \norm{x}_2=1\},\]
is convex.
\end{proposition}
\begin{proof}
    Le $b, b'\in B$ be corresponding point for $x,x' \succeq 0$, and $\norm{x},\norm{x'}=1$ respectively.
    Now make circle with origin, $x$ and $x'$. The minor arc containing $x$ and $x'$ lies entirely in the $\R^n_+$.  The $x$ lying on the arc will generate a line  connecting $b$ and $b'$. Hence $B$ is convex.
\end{proof}

\subsection{Experimental Verification}
This section describes three different methods for solving the optimization problem:
\[
\max_{\|x\|_2 \leq 1, \, x \geq 0} \|Ax\|_2,
\]
where \(A \in \mathbb{R}^{n \times n}\). The methods are compared in terms of their computational efficiency and the quality of their solutions.

\subsubsection{Brute Force Random Search}
The brute force method randomly samples vectors \(x \in \mathbb{R}^n\) from the nonnegative orthant, normalizes them to satisfy \(\|x\|_2 = 1\), and evaluates \(\|Ax\|_2\) for each sampled vector. The steps are as follows:
\begin{enumerate}
    \item Generate \(N\) random vectors \(x_i \geq 0\), \(i = 1, \dots, N\).
    \item Normalize each vector to unit norm: \(x_i \gets x_i / \|x_i\|_2\).
    \item Compute \(\|Ax_i\|_2\) for each vector and select the maximum value.
\end{enumerate}
This method is simple to implement but computationally expensive, as it evaluates \(A\) for a large number of randomly generated vectors. See figure \ref{fig:random-guess}
\begin{figure}[ht!]
    \centering
    \includegraphics[width=0.70\linewidth]{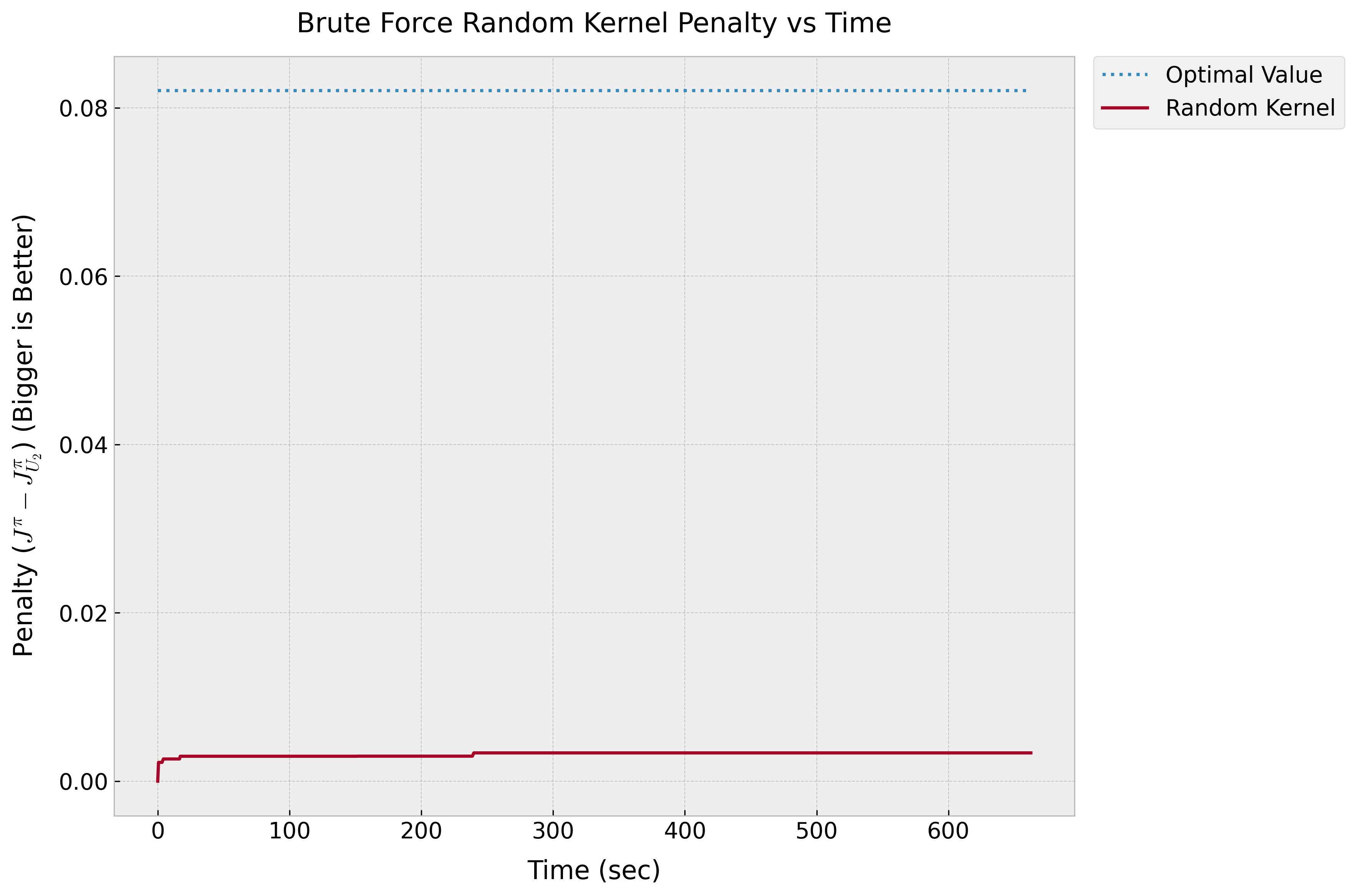}
    \caption{Random Kernel Guess takes exponentially long time to converge. While Algorithm 1 only took 0.14 sec to find the optimal value.}
    \label{fig:random-guess}
\end{figure}

\subsubsection{Numerical Optimization (Scipy Minimize)}
This approach uses numerical optimization to directly solve the problem:
\[
\max_{\|x\|_2 \leq 1, \, x \geq 0} \|Ax\|_2.
\]
The optimization problem is formulated as:
\[
\min_{x} -\|Ax\|_2, \quad \text{subject to } \|x\|_2 \leq 1 \text{ and } x \geq 0.
\]
Steps include:
\begin{enumerate}
    \item Define the objective function as \(-\|Ax\|_2\).
    \item Impose constraints: \(\|x\|_2 \leq 1\) and \(x \geq 0\).
    \item Solve the problem using \texttt{scipy.optimize.minimize}, with an initial guess \(x_0\).
\end{enumerate}
This method provides the exact solution but is computationally more expensive than the spectral method.

\subsection{Comparison Metrics}
The three methods are compared based on:
\begin{itemize}
    \item \textbf{Optimality:} The maximum value \(\|Ax\|_2\) achieved by each method.
    \item \textbf{Time Efficiency:} The computational time required by each method.
\end{itemize}

\subsection{Results and Observations}
The following plots compare the performance of the three methods:
\begin{itemize}
    \item \textbf{Optimality Plot:} Shows that the maximum value obtained with scipy.minimize is slightly better than our spectral method, while random search performs poorly.
    \item \textbf{Time Efficiency Plot:} Illustrates the that scipy.minimize scales much poorly with the dimension, while our spectral method is way faster than both methods. 
\end{itemize}

\begin{figure}[ht]
    \centering
    \includegraphics[width=0.8\textwidth]{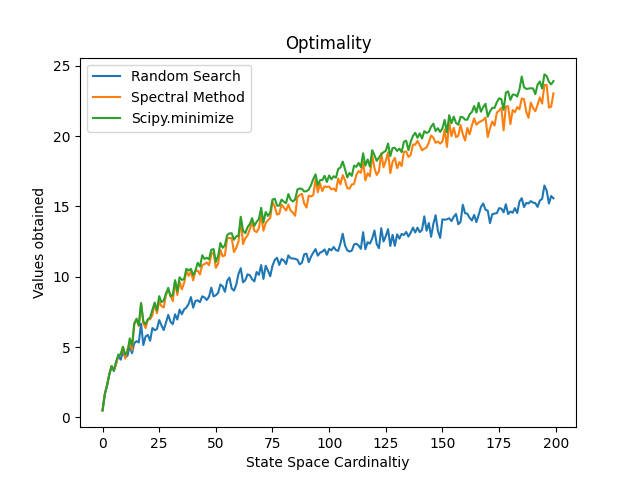}
    \caption{Comparison of optimality across methods.}
    \label{fig:optimality}
\end{figure}

\begin{figure}[ht]
    \centering
    \includegraphics[width=0.8\textwidth]{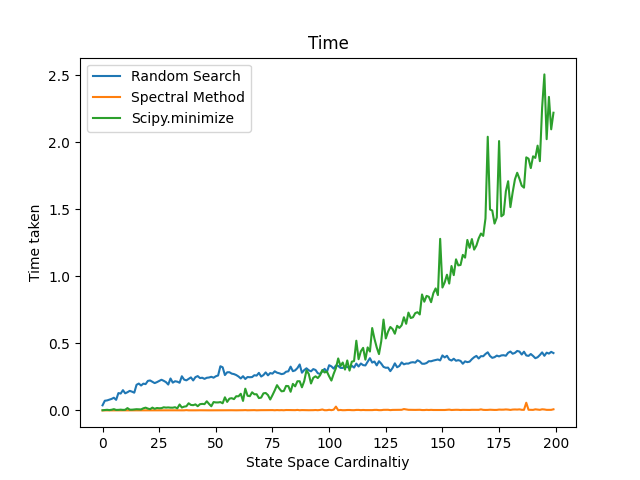}
    \caption{Comparison of computational time across methods}
    \label{fig:time}
\end{figure}
\begin{table}[ht]
    \centering
    \begin{tabular}{c|c|c|c||c|c|c}
    \toprule
& \multicolumn{3}{|c|}{Optimal values attained} & \multicolumn{3}{|c|}{Time taken} \\
 \midrule
    $n$ & Random & Spectral & minimize&Random & Spectral & minimize \\
    \midrule
    10 & 4.10& 4.45&4.46& 0.12&
       0.0007 &0.005 \\
    20&5.14& 6.71& 6.82&0.19&
       0.0003 & 0.01\\
     50& 9.23& 11.59& 11.93&  0.25&  0.0007  &  0.03  \\
     100&11.95& 16.44& 17.19&  0.31&  0.001  &  0.28\\
     200& 15.74& 22.1 & 23.68&  0.44&  0.004  &  2.1\\
     300& 19.32& 28.58& 29.73& 0.57& 0.012& 8.19\\
  500& 24.46& 36.56& 38.47& 0.83& 0.209& 43.49\\
1000& 33.91& 51.64& 54.25& 1.38& 0.171& 313.6\\
    \bottomrule
    \end{tabular}
    \caption{Attained Values and Time Taken. }
    \label{tab:Ax}
\end{table}

\subsubsection{Parameters of Experiments }

The experiments were conducted to evaluate the performance of three methods—brute force random search, eigenvalue heuristic, and numerical optimization—on solving the problem:
\[
\max_{\|x\|_2 \leq 1, \, x \geq 0} \|Ax\|_2.
\]

\paragraph{State Space Cardinality and Random matrix Generation}
\begin{itemize}
    \item \textbf{State Space Cardinality (\(n\)):} The dimension of the problem, denoted by \(n\), represents the state space cardinality. In the experiments, \(n\) varied from 1 to 300 to analyze the scalability of the methods.
    \item \textbf{Matrix Generation:} The matrix \(A \in \mathbb{R}^{n \times n}\) was generated as a random matrix with entries sampled from a standard normal distribution:
    \[
    A_{ij} \sim \mathcal{N}(0, 1), \quad i, j = 1, \dots, n.
    \]
    The same random seed (\texttt{seed = 42}) was used across all runs to ensure reproducibility.
    \item 10000 random vectors $x$ were generated for Brute Search Method.
\end{itemize}

\paragraph{Process of matrix Evaluation}
The goal of the experiments is to maximize \(\|Ax\|_2\) under the constraints \(\|x\|_2 \leq 1\) and \(x \geq 0\). The matrix \(A\) is evaluated by:
\begin{enumerate}
    \item Generating random vectors \(x \in \mathbb{R}^n\) for the brute force method.
    \item Computing the spectral decomposition of \(A^\top A\) for the eigenvalue heuristic.
    \item Defining and solving a constrained optimization problem for the numerical optimization method.
\end{enumerate}
The results, including the optimal values and computational times, are recorded for each method.

\paragraph{Evaluation Metrics}
The performance of the methods was assessed using the following metrics:
\begin{itemize}
    \item \textbf{Optimality:} The maximum value \(\|Ax\|_2\) obtained by each method.
    \item \textbf{Computational Efficiency:} The time taken by each method to compute the result.
    \item \textbf{Scalability:} The behavior of the methods as \(n\) increases.
\end{itemize}
This systematic evaluation ensures a fair comparison of the three approaches across varying problem sizes.
\paragraph{Hardware and Software Specifications}

The experiments were conducted on the following hardware and software setup:

\begin{itemize}
    \item \textbf{Model Name:} MacBook Pro (2023 model).
    \item \textbf{Model Identifier:} Mac14,7.
    \item \textbf{Chip:} Apple M2 with 8 cores (4 performance and 4 efficiency cores).
    \item \textbf{Memory:} 16 GB Unified Memory.
    \item \textbf{Operating System:} macOS Ventura.
    \item \textbf{Programming Language:} Python 3.9.
    \item \textbf{Libraries Used:} 
    \begin{itemize}
        \item \texttt{numpy} for numerical computations.
        \item \texttt{scipy} for numerical optimization.
        \item \texttt{matplotlib} for generating plots.
        \item \texttt{time} for recording computational times.
    \end{itemize}
\end{itemize}

The experiments were designed to ensure reproducibility by fixing the random seed (\texttt{seed = 42}). Computational times and results are specific to the above hardware configuration and may vary on different systems.

\section{Convexity of \texorpdfstring{$\D$}{D}}
\subsection{MDP Configuration}
We define an MDP with the following parameters:
\begin{itemize}
    \item \textbf{State space size}: \( S = 3 \)
    \item \textbf{Action space size}: \( A = 2 \)
    \item \textbf{Discount factor}: \( \gamma = 0.9 \)
    \item Random kernel $P$, random reward $R$, seed 42.
    \item Compute the set $\D =\{D^\pi H^\pi|\pi\}$ with 10 millions random policies $\pi$
\end{itemize}

\subsection{Dimensionality Reduction via PCA}
Given the high-dimensional nature of the \( D^\pi H^\pi \) representations, we apply \textbf{Principal Component Analysis (PCA)} to extract meaningful structure.
\begin{itemize}
    \item We \textbf{retain the top 10 components} to capture the dominant variations in the dataset.
    \item The \textbf{explained variance ratio} is visualized to assess how much information each component retains.
    \item \textbf{2D and 3D projections} of the first few principal components are generated for visualization.
\end{itemize}

\begin{figure}[ht]
    \centering
    \includegraphics[width=0.7\textwidth]{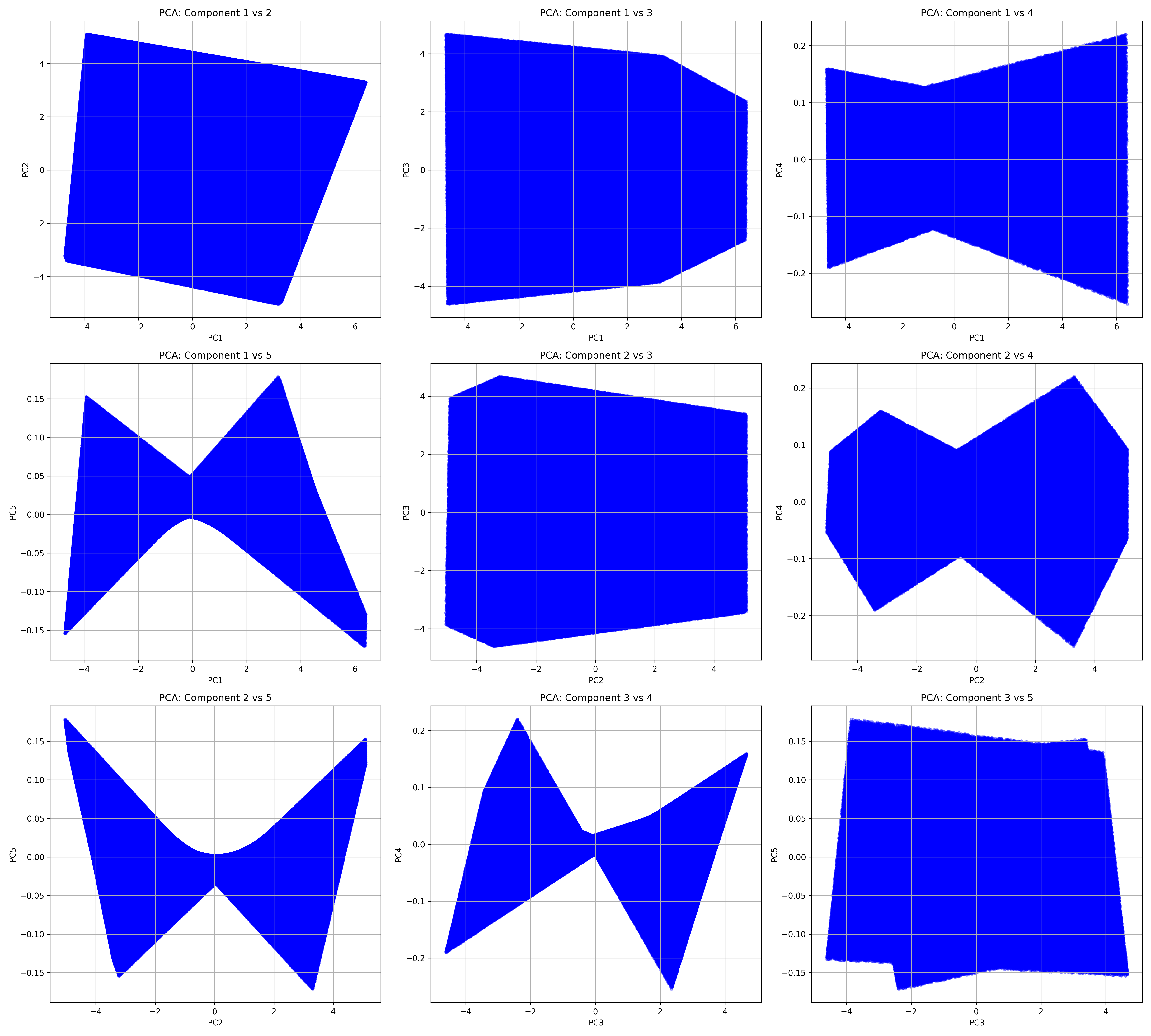}
    \caption{2D PCA projections of the first 5 components.}
    \label{fig:pca_2d}
\end{figure}

\begin{figure}[ht]
    \centering
    \includegraphics[width=0.6\textwidth]{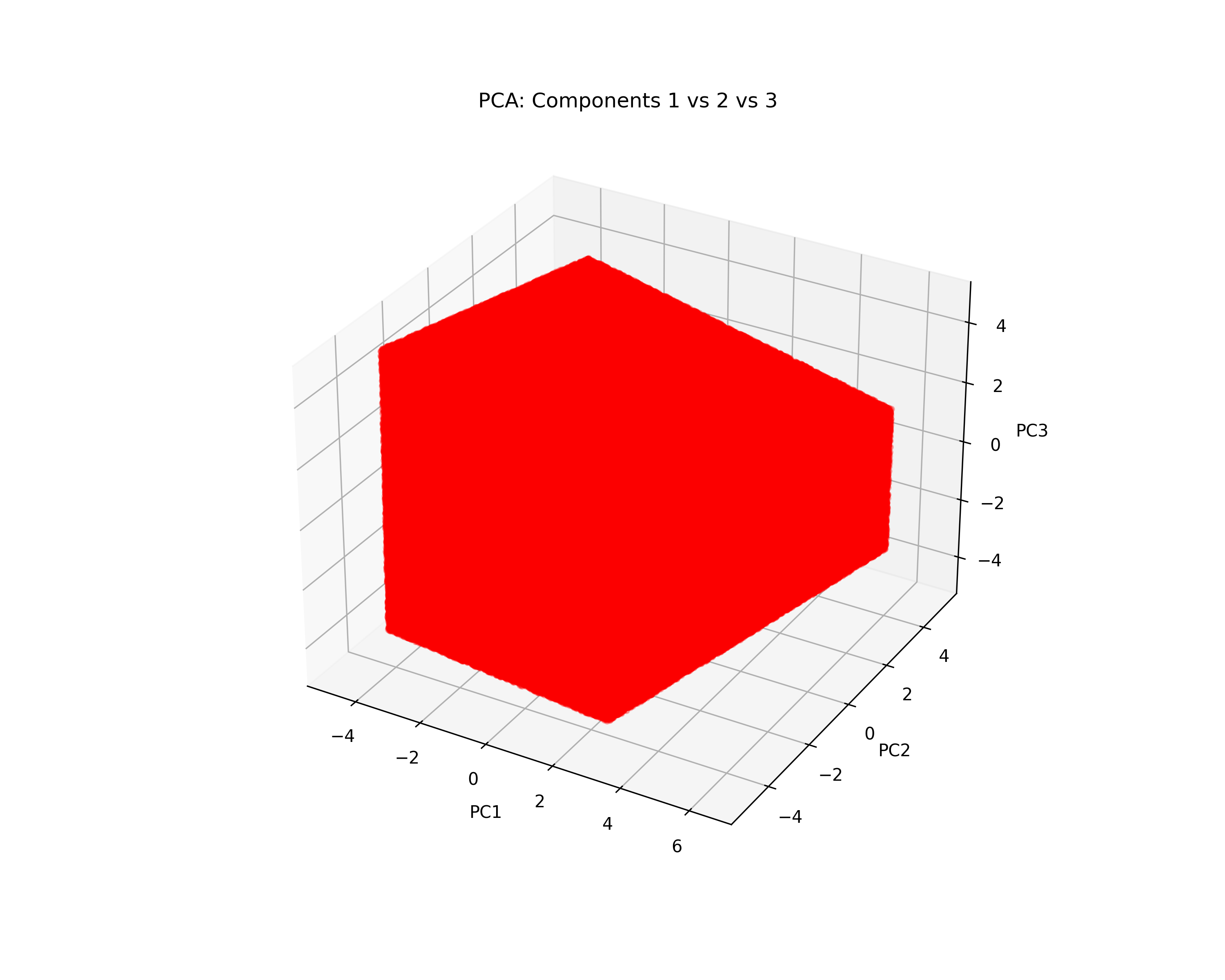}
    \caption{3D PCA projection of the first three principal components.}
    \label{fig:pca_3d}
\end{figure}

\subsection{Random Linear Projections}
To further explore the \textbf{geometry of the occupancy measure set}, we apply \textbf{random linear projections} of the high-dimensional data:
\begin{itemize}
    \item \textbf{2D Random Projections}: The data is projected onto \textbf{randomly chosen 2D subspaces}.
    \item \textbf{3D Random Projections}: The data is projected onto \textbf{randomly chosen 3D spaces}.
\end{itemize}

\begin{figure}[ht]
    \centering
    \includegraphics[width=0.7\textwidth]{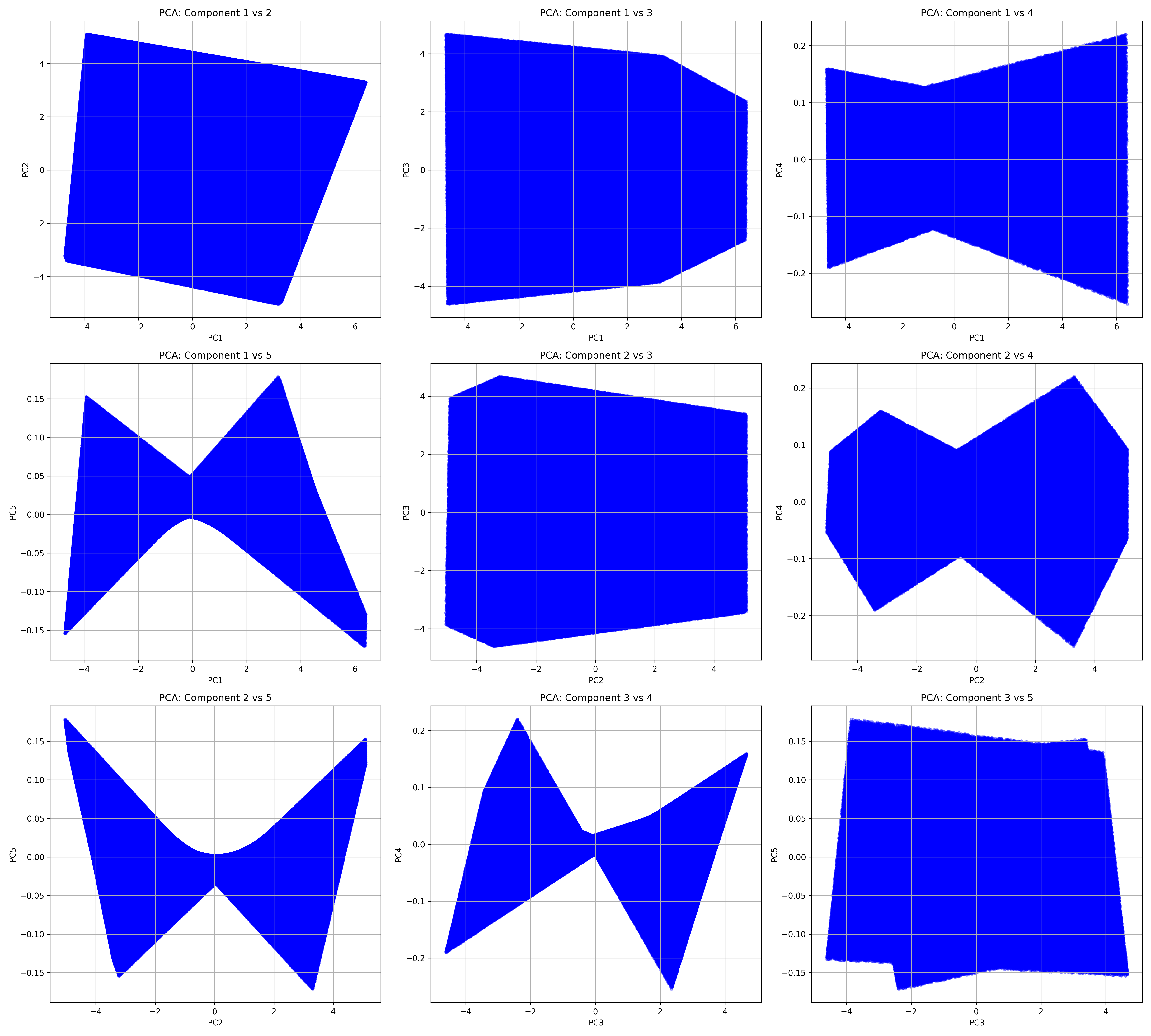}
    \caption{2D Random Projections of the Data.}
    \label{fig:random_2d}
\end{figure}

\begin{figure}[ht]
    \centering
    \includegraphics[width=0.6\textwidth]{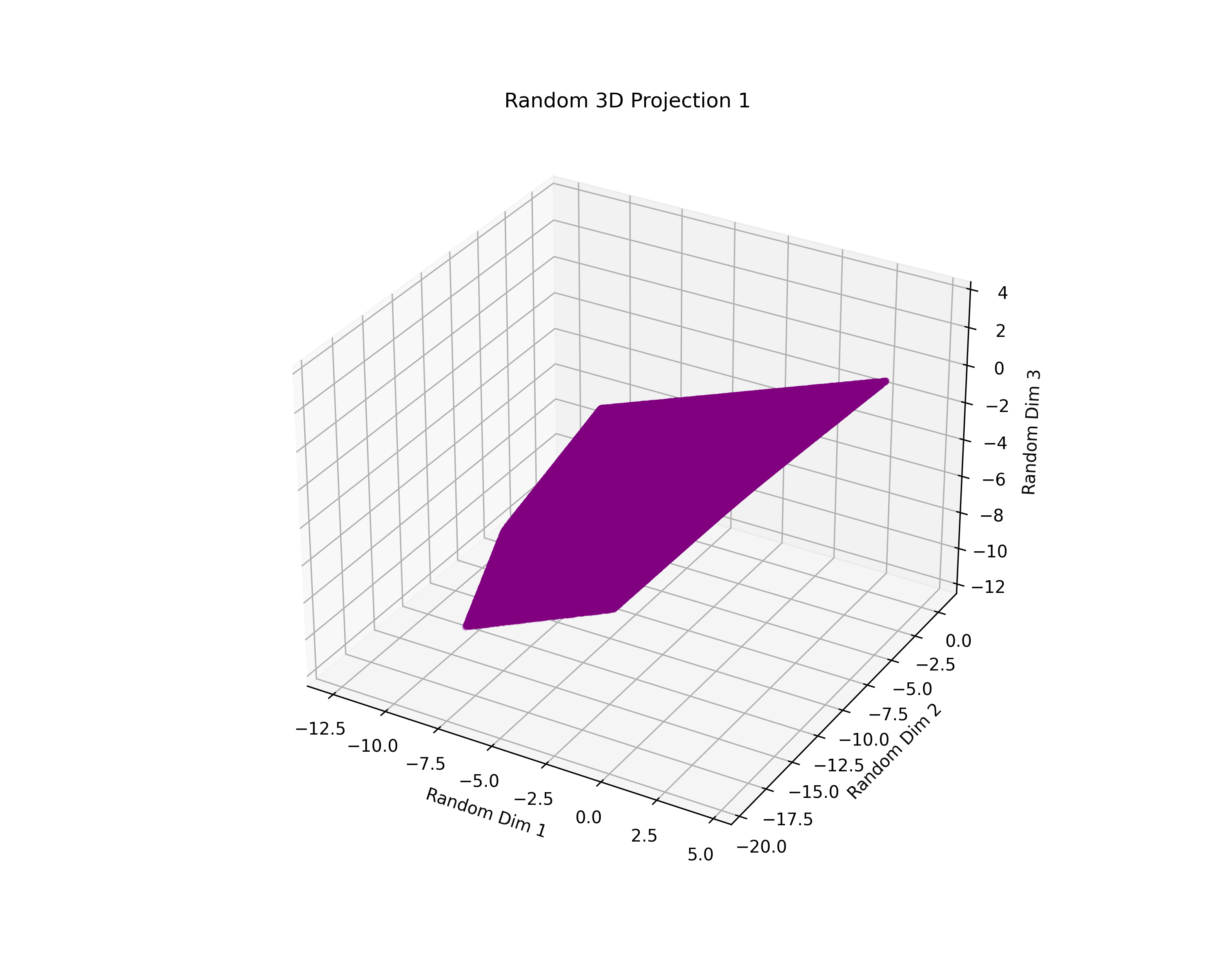}
    \caption{3D Random Projection Example.}
    \label{fig:random_3d}
\end{figure}

\section{Experiments: Robust Policy Evaluation}
Setting: Randomly generate $\hat{P}$ is nominal kernel, reward function $R$, and  $\pi$ is a policy. Discount factor $\gamma=0.9$, uncertainty radius $\beta = 0.01$, initial distribution $\mu = uniform distribution$.

Notations: Value function $v^\pi = (I-\gamma \hat{P}^\pi)^{-1}R^\pi$, Occupancy matrix $D^\pi = (I-\gamma \hat{P}^\pi)^{-1}R^\pi$, occupation measure $d^\pi = \mu^T(I-\gamma \hat{P}^\pi)^{-1}R^\pi$, return $J^\pi_{P} = (I-\gamma P^\pi)^{-1}R^\pi$,.

In this section, we evaluate the performance of Algorithm \ref{alg:RPE} for robust policy evaluation, focusing on the case \( p = 2 \). The algorithm requires computing \( F(\lambda) \) at each iteration, which involves solving the constrained matrix norm problem \( \max_{x \in \mathcal{B}} \|Ax\|_2 \). This can be efficiently handled using our spectral Algorithm \ref{alg:Ax:EH}. Figures \ref{main:fig:rve:value} and \ref{main:fig:rve:time} compare four methods for computing the robust return, specifically the penalty term \( J^\pi - J^\pi_{\Uc_2} \):  
\begin{itemize}
    \item \textbf{ SLSQP from scipy \cite{2020SciPy-NMeth} :} This is  A semi-brute force approach that uses Lemma \ref{main:rs:RPE:dual}. It computes the penalty term via scipy.minimize to directly optimize over ($b,k$). This is equivalent to optimize over only rank-one perturbation of nominal kernel as a\( (b, k) \) corresponds to selecting a rank-one perturbation of the nominal kernel \( P = \hat{P} - bk^\top \).  Note that this method is local, hence can sometime be stuck in very bad local solution. 
    
\item \textbf{Binary search Algorithm \ref{alg:RPE} :} Uses the binary search Algorithm \ref{alg:RPE}, and  spectral Algorithm \ref{alg:Ax:EH} for computing \( F(\lambda) \) in each iteration.

\item \textbf{Random Rank-One Kernel Sampling :} A semi-brute force approach that uses Lemma \ref{main:rs:RPE:dual} to sample random pairs \( (b, k) \in \mathcal{B}, \mathcal{K} \), empirically maximizing the penalty term. Since choosing \( (b, k) \) corresponds to selecting a rank-one perturbation of the nominal kernel \( P = \hat{P} - bk^\top \), the method is named accordingly.  
\item \textbf{Random Kernel Sampling :} A brute-force approach that samples random kernels directly from the uncertainty set \( \Uc_2 \), computing the empirical minimum as an estimate of the robust return.  
\end{itemize}

\begin{figure}
    \centering
    \includegraphics[width=0.7\linewidth]{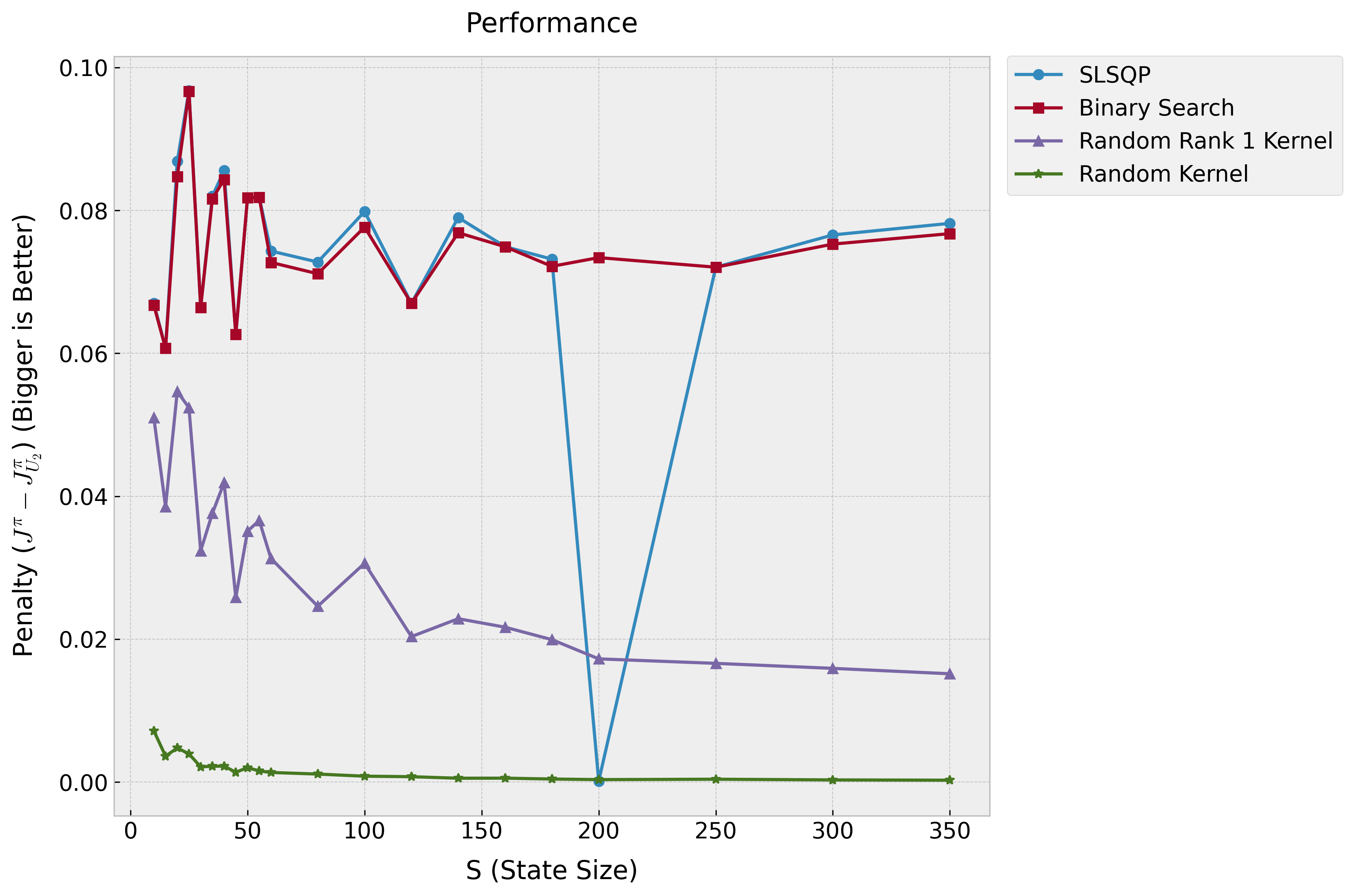}
    \caption{Performance of Robust Policy Evaluation methods with equal amount of time.}
    \label{fig:rve:value}
\end{figure}
Figure \ref{fig:rve:value} presents the penalty value (\(J^\pi - J^\pi_{\Uc_2}\)) computed by different methods across various state space sizes while keeping the action space fixed, with each method given the same computational time. Our binary search Algorithm \ref{alg:RPE} performs significantly better in both variations compared to brute-force (random kernel sampling) and semi-brute (random sampling of rank-one perturbations of the nominal kernel) approaches. Notably, the scipy SLSQP variant performs slightly better on average, but the spectral Algorithm \ref{alg:Ax:EH} is more reliable. This is expected, as the spectral method is global, while scipy SLSQP is a local optimizer and thus more prone to getting stuck in suboptimal solutions.

\begin{figure}
    \centering
    \includegraphics[width=0.7\linewidth]{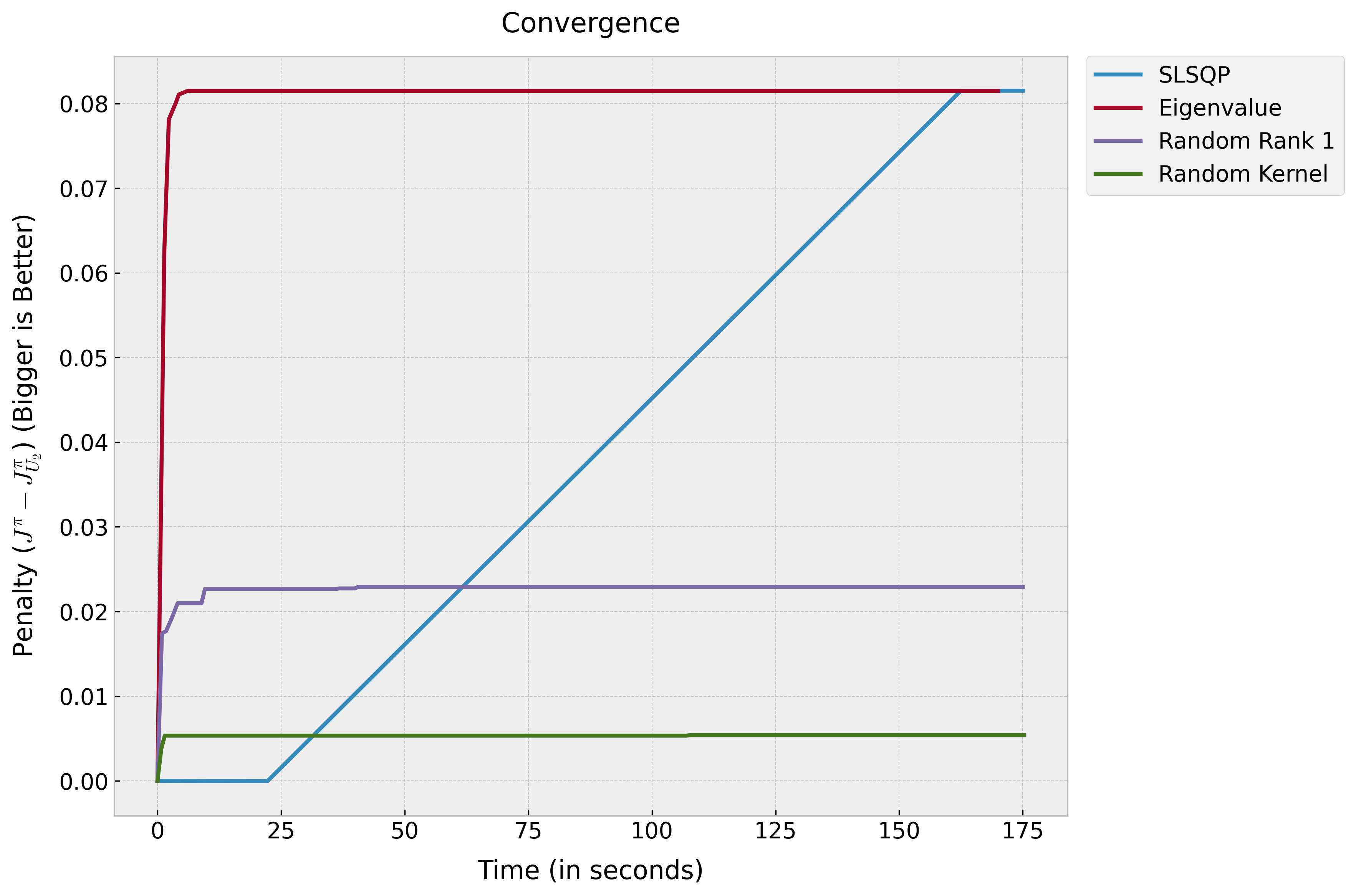}
    \caption{Convergence of Robust Policy Evaluation Methods}
    \label{fig:rve:time}
\end{figure}
\subsection{Penalty Values  (\texorpdfstring{$J^\pi - J^\pi_{\Uc_2}$}{Jpi - Jpi U2}) vs Sample Size for Different Algorithms}

Figure \ref{fig:beta_0.1}, figure \ref{fig:beta_0.02} and figure \ref{fig:beta_0.005} show Penalty Values  ($J^pi-J^\pi_{\Uc_2}$) for different values of $\beta$ calculated using different algorithms.

The experiment as shown in figure \ref{fig:beta_0.005} shows that the SLSQP algorithm is not robust for this problem and can end up in local minima multiple times.
\begin{figure}
    \centering
    \includegraphics[width=0.7\linewidth]{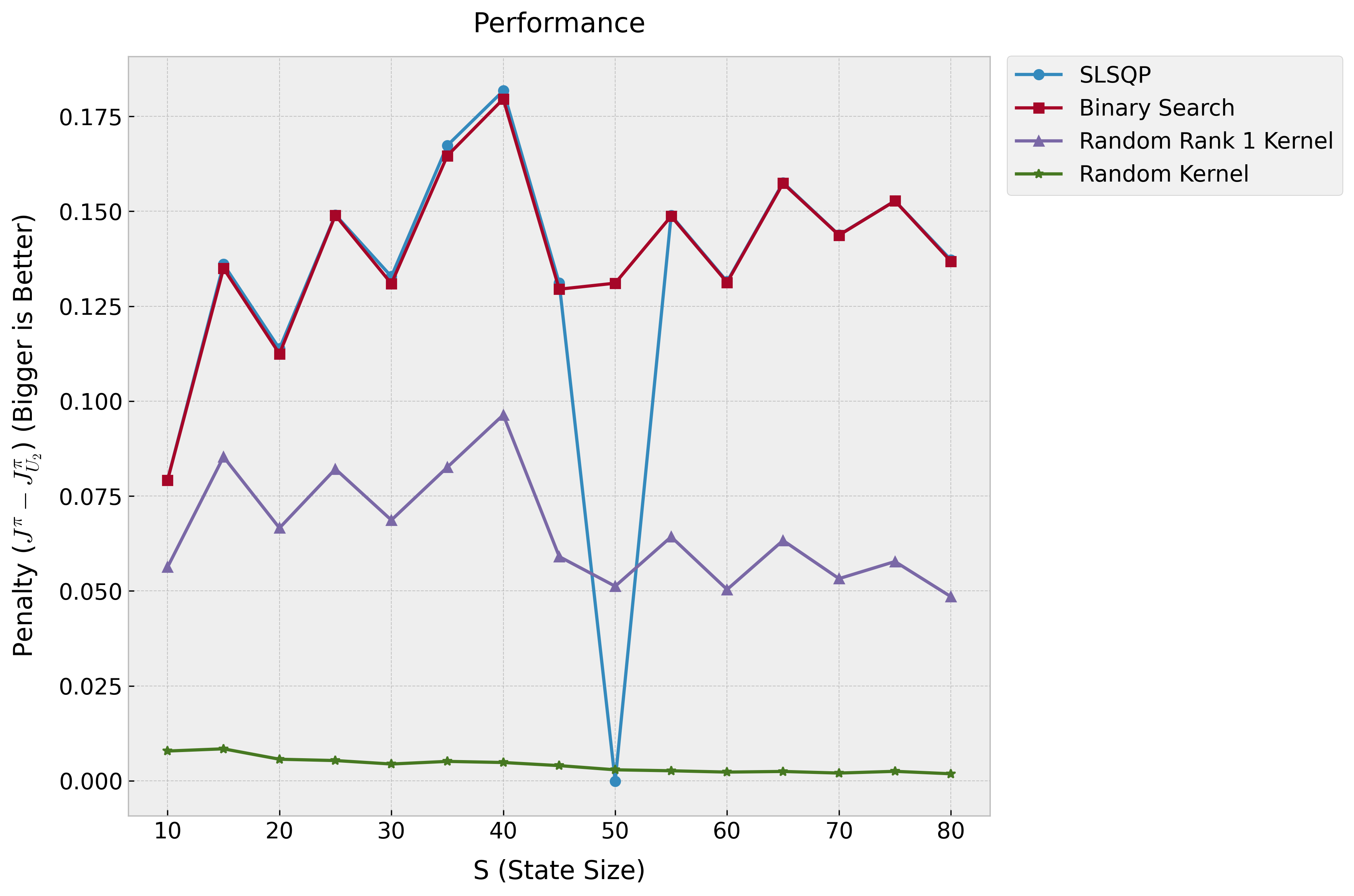}
    \caption{$\beta$=0.1, A=8}
    \label{fig:beta_0.1}
\end{figure}
\begin{figure}[ht]
    \centering
    \includegraphics[width=0.7\linewidth]{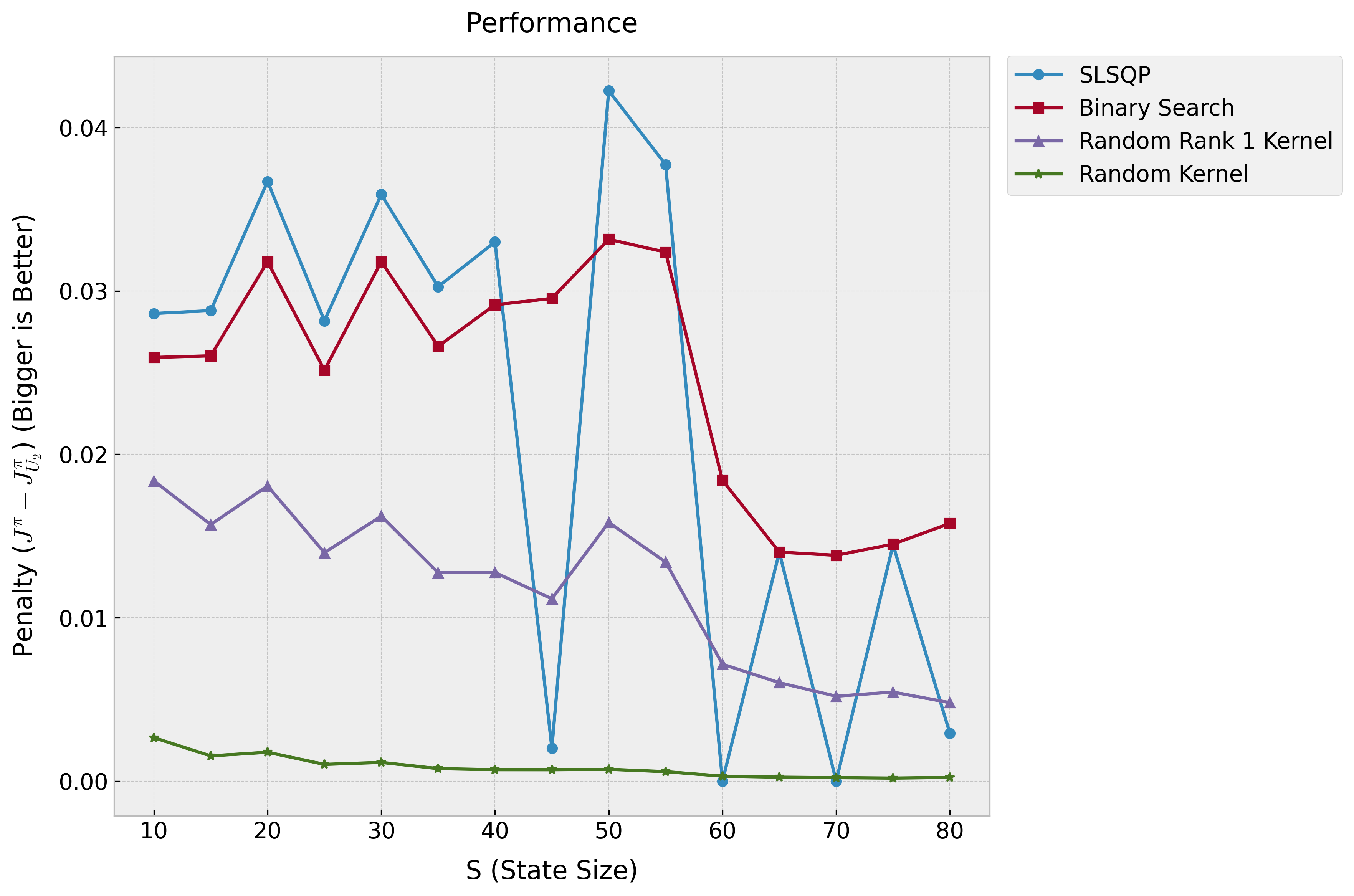}
    \caption{$\beta$=0.02, A=8}
    \label{fig:beta_0.02}
\end{figure}

\begin{figure}[ht]
    \centering
    \includegraphics[width=0.7\linewidth]{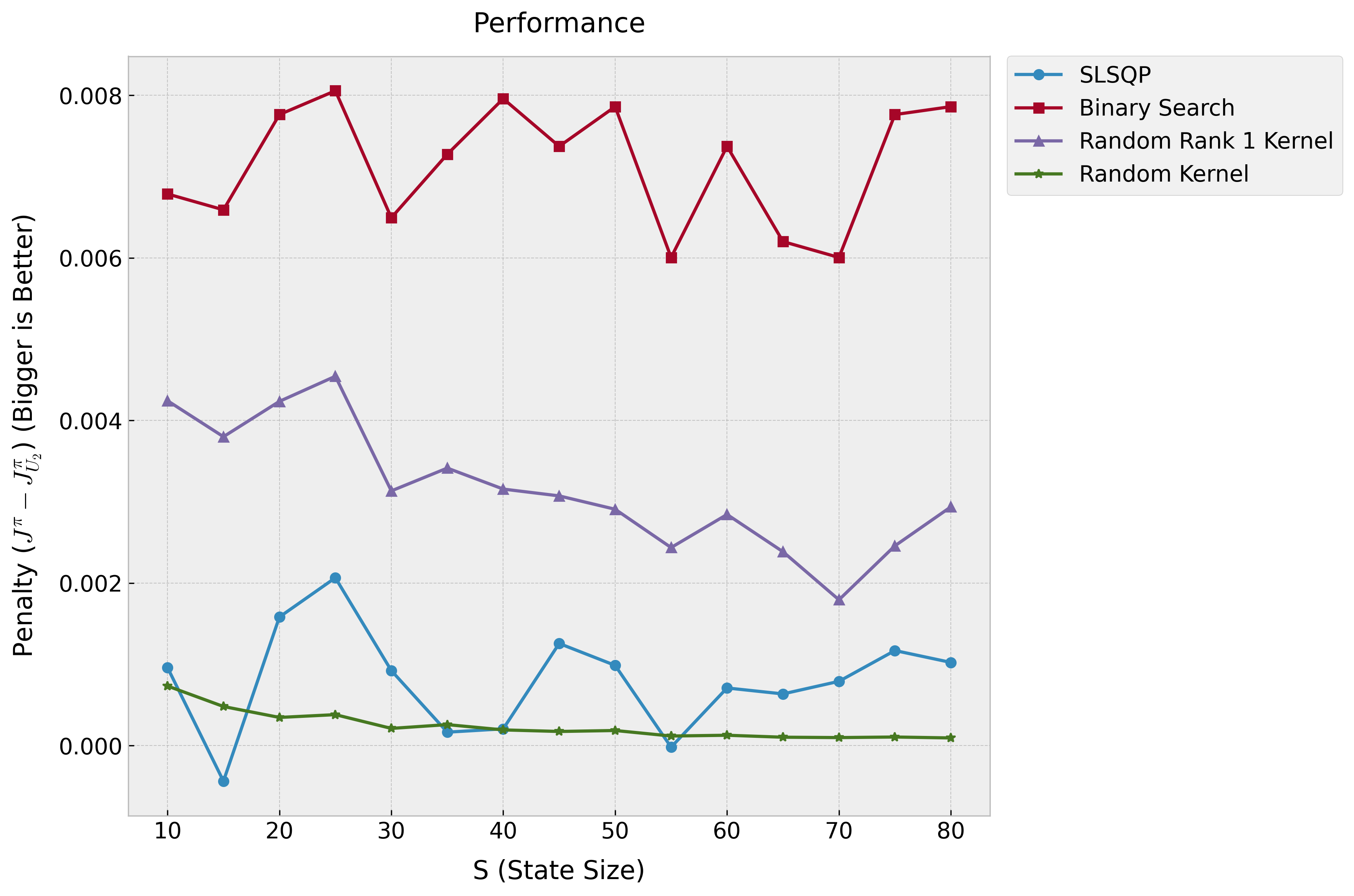}
    \caption{$\beta$=0.005, A=8}
    \label{fig:beta_0.005}
\end{figure}

Figure \ref{fig:rve:time} illustrates the convergence of different approaches over time (in seconds) for a fixed state space size (\(S=8\)). The spectral variation of Algorithm \ref{alg:RPE} converges rapidly, while the SLSQP variation gradually approaches the same performance. In contrast, the brute-force method shows slow, logarithmic improvement. This behavior arises because brute-force methods require an exponential number of samples (in the dimensionality of the problem) to adequately explore all directions. In comparison, our binary search Algorithm combined with the spectral Algorithm \ref{alg:Ax:EH} achieves significantly better efficiency, with a complexity of \(O(S^3 A^3 \log \epsilon^{-1})\).

More details of these experiments along with others can be found in the appendix, and codes are available at \url{https://anonymous.4open.science/r/non-rectangular-rmdp-77B8}. 

Our experiments confirm the efficiency of our binary search Algorithm \ref{alg:RPE} for robust policy evaluation, significantly outperforming brute-force approaches in both accuracy and convergence speed.

\begin{algorithm}

\subsection{Our Method}
\caption{ Binary Search for Robust Policy Evaluation for Uncertainty set $\Uc_p$}\label{app:alg:RPE}
\begin{algorithmic} [1]
\STATE \textbf{Input}: Tolerance $\epsilon =0.001, \beta=0.01$\\
\STATE\textbf{Initialize}: $\lambda=\frac{0.5}{1-\gamma}, \lambda_{max} = \frac{1}{1-\gamma}, \lambda_{min}=0 $  \\
\WHILE{ Tolerance is not met: $f(\lambda) > \epsilon$  }
\STATE Compute: 
\begin{align*}
f(\lambda)
 =& \max_{\norm{b}_p\leq 1, b\succeq 0}\beta\gamma\norm{ \Phi(v^\pi {d^\pi}^{\top} -\lambda D^\pi\bigm) H^\pi b}_q,
\end{align*}
where $\Phi = I-\frac{\mathbf{1}\mathbf{1}^\top}{S}$ is projection matrix and $(H^\pi b)(s) = \sum_{a}\pi(a|s)b(s,a)$ is policy averaging operator.
\STATE Bisection: If $f(\lambda) > \epsilon$, set $\lambda_{min} = \lambda$, if $f(\lambda) < \epsilon$, set $\lambda_{max} = \lambda$ 

\ENDWHILE

\STATE \textbf{Output}:  Robust return: $ J^\pi_{\Uc_p} =J^\pi - \lambda $.
\end{algorithmic}
\end{algorithm}

\subsection{Robust Penalty Function}

We have defined robust penalty function in \ref{subsection: robust_policy evaluation} as 

\[
F(\lambda) = \max_{b \in \mathcal{B}} \|E^\pi_\lambda b\|_q,
\]

Figure \ref{fig:beta_1_over_S} and Figure \ref{fig:beta_0_1} show graph of F($\lambda$) vs $\lambda$ for different value of  $\beta$  for fixed value of S=100 and A=10.

\begin{figure}[ht]
    \includegraphics[width=0.8\textwidth]{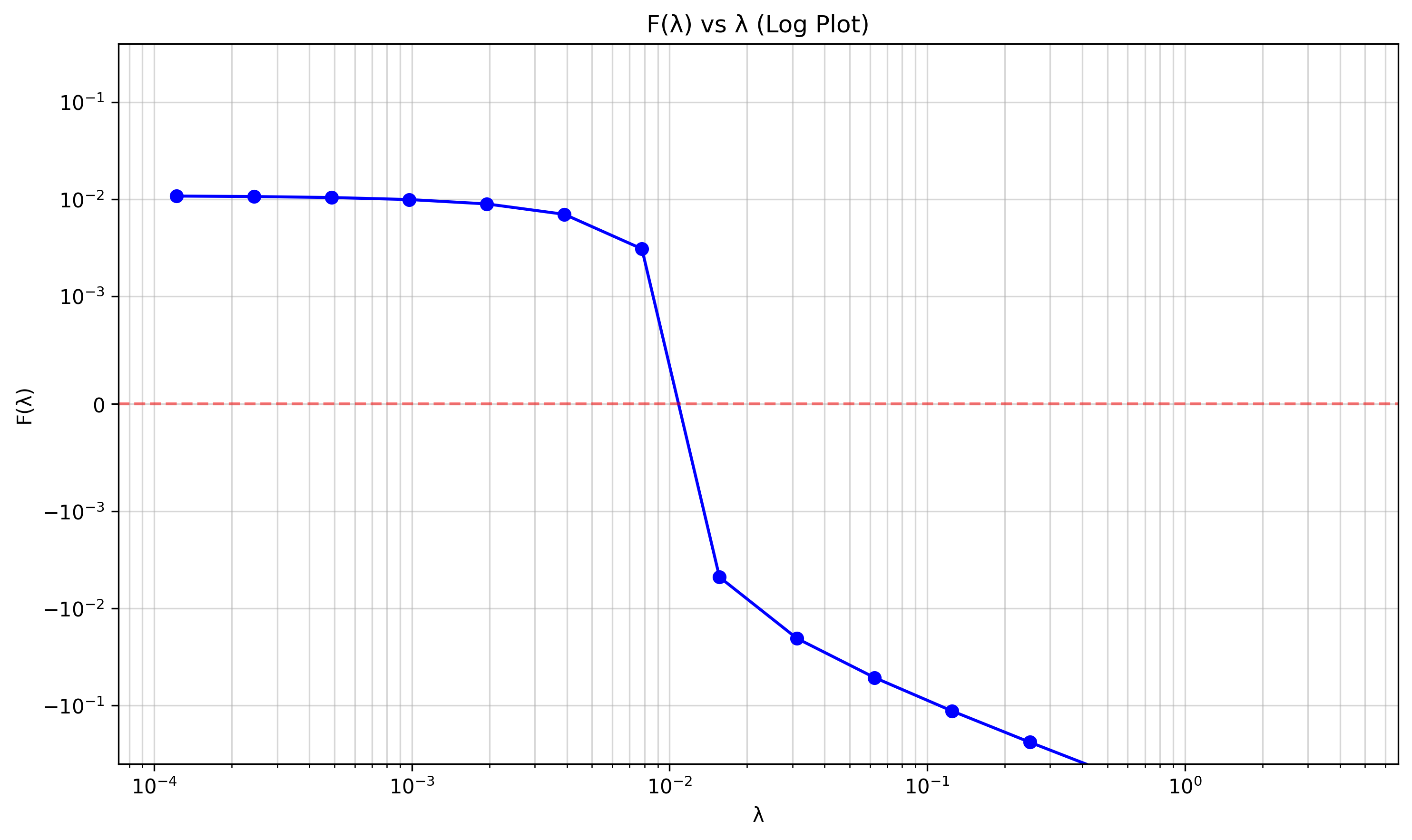}
    \caption{Robust Penalty Function F($\lambda$) vs $\lambda$ for S = 100, a = 10, $\beta = 1/S$}
    \label{fig:beta_1_over_S}
\end{figure}

\begin{figure}[ht]
    \includegraphics[width=0.8\textwidth]{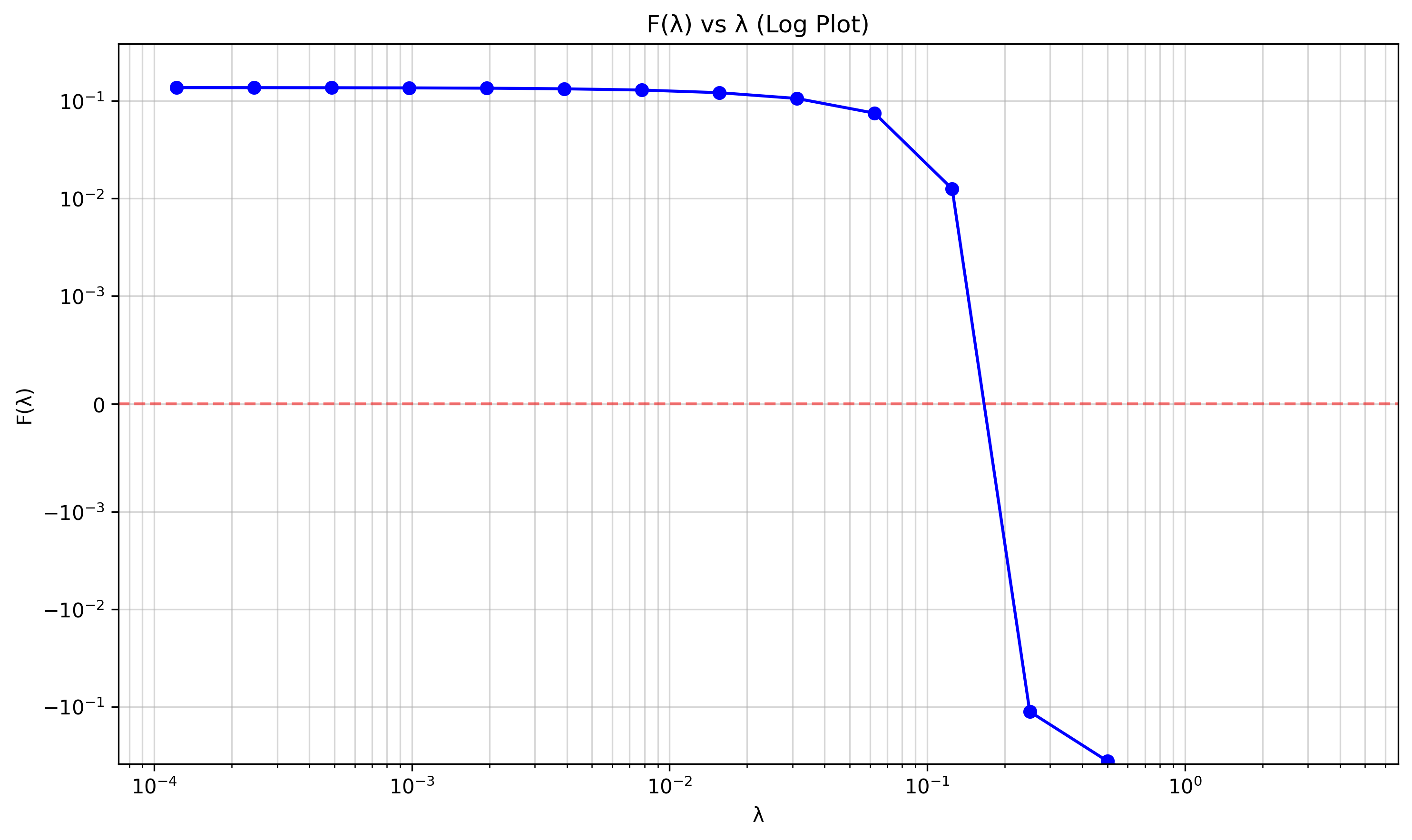}
    \caption{Robust Penalty Function F($\lambda$) vs $\lambda$ for S = 100, a = 10, $\beta = 0.1$}
    \label{fig:beta_0_1}
\end{figure}

\end{document}